\documentclass[nohyperref]{article}

\usepackage{microtype}
\usepackage{graphicx}
\usepackage{subfigure}
\usepackage{booktabs} 

\usepackage{hyperref}

\usepackage[accepted]{icml2022}

\usepackage{amsmath}
\usepackage{amssymb}
\usepackage{mathtools}
\usepackage{amsthm}

\usepackage[capitalize,noabbrev]{cleveref}

\usepackage{tikz}
\usetikzlibrary{automata,positioning}
\usetikzlibrary{arrows.meta}

\usepackage{nicefrac}

\DeclareMathOperator*{\argmin}{arg\,min}
\DeclareMathOperator*{\argmax}{arg\,max}

\newcommand{\bbE}{\mathbb{E}}
\newcommand{\bbR}{\mathbb{R}}
\newcommand{\bbG}{\mathbb{G}}
\newcommand{\bbN}{\mathbb{N}}
\newcommand{\calF}{\mathcal{F}}
\newcommand{\calM}{\mathcal{M}}
\newcommand{\calK}{\mathcal{K}}

\newcommand{\calS}{\mathcal{S}}
\newcommand{\calA}{\mathcal{A}}
\newcommand{\wt}{\widetilde}

\newcommand{\ind}{\mathbb{I}}
\newcommand{\indevent}[1]{\ind \{ #1 \}}
\newcommand{\VAR}{\mathrm{Var}}

\newcommand{\numagents}{m}
\newcommand{\sinit}{s_{\text{init}}}
\newcommand{\logterm}{\tau}
\newcommand{\sumhsak}{\sum_{h=1}^H\sum_{s\in\calS}\sum_{a\in\calA}\sum_{k=1}^K}
\newcommand{\sumhsa}{\sum_{h=1}^H\sum_{s\in\calS}\sum_{a\in\calA}}
\newcommand{\regret}{R_K}
\newcommand{\filt}{\calF}
\newcommand{\bbEk}{\bbE^k}

\newcommand{\ocset}{\Delta(\calM)}
\newcommand{\ocsetk}{\Delta(\calM,k)}
\newcommand{\KL}[2]{\text{KL}(#1 \;\|\; #2)}

\newcommand{\overl}[1]{\overline{ #1}}
\newcommand{\underl}[1]{\underline{ #1}}

\theoremstyle{plain}
\newtheorem{theorem}{Theorem}[section]

\newtheorem{lemma}[theorem]{Lemma}

\theoremstyle{definition}

\theoremstyle{remark}

\icmltitlerunning{Cooperative Online Learning in Stochastic and Adversarial MDPs}

\begin{document}

\twocolumn[
\icmltitle{Cooperative Online Learning in Stochastic and Adversarial MDPs}



\icmlsetsymbol{equal}{*}

\begin{icmlauthorlist}
\icmlauthor{Tal Lancewicki}{equal,tau}
\icmlauthor{Aviv Rosenberg}{equal,tau}
\icmlauthor{Yishay Mansour}{tau,goo}
\end{icmlauthorlist}

\icmlaffiliation{tau}{Tel Aviv University}
\icmlaffiliation{goo}{Google Research, Tel Aviv}

\icmlcorrespondingauthor{Tal Lancewicki}{lancewicki@mail.tau.ac.il}

\icmlkeywords{Machine Learning, ICML}

\vskip 0.3in
]



\printAffiliationsAndNotice{\icmlEqualContribution} 
\begin{abstract}
    We study cooperative online learning in stochastic and adversarial Markov decision process (MDP).
    That is, in each episode, $m$ agents interact with an MDP simultaneously and share information in order to minimize their individual regret.
    We consider environments with two types of randomness: \emph{fresh} -- where each agent's trajectory is sampled i.i.d, and \emph{non-fresh} -- where the realization is shared by all agents (but each  agent's trajectory is also affected by its own actions).
    More precisely, with non-fresh randomness the realization of every cost and transition is fixed at the start of each episode, and agents that take the same action in the same state at the same time observe the same cost and next state.
    We thoroughly analyze all relevant settings, highlight the challenges and differences between the models, and prove nearly-matching regret lower and upper bounds.
    To our knowledge, we are the first to consider cooperative reinforcement learning (RL) with either non-fresh randomness or in adversarial MDPs.
\end{abstract}

\section{Introduction}

\begin{table*}
    \caption{Summary of our regret upper and lower bounds for $\numagents$ agents facing $K$ episode interaction with an MDP that has $S$ states, $A$ actions and horizon $H$. The bounds ignore poly-logarithmic factors and lower order terms. (*) The algorithm requires $m = \sqrt{K}$ agents.}
    \begin{center}
        \begin{tabular}[c]{|c|c|c|c|c|c|c|c|c|}
            \hline
            Algorithm & Regret & Lower Bound & Randomness & Cost & Transition
            \\ 
            \hline \hline
            \texttt{coop-ULCVI} & $\sqrt{\frac{H^3 S A K}{\numagents}}$ & $\sqrt{\frac{H^3 S A K}{\numagents}}$ & fresh & stochastic & unknown
            \\
            \hline
            \texttt{coop-O-REPS} & $\sqrt{H^2 K} + \sqrt{\frac{H^2 S A K}{\numagents}}$ & $\sqrt{H^2 K} + \sqrt{\frac{H^2 S A K}{\numagents}}$ & fresh & adversarial & known
            \\
            \hline
            \texttt{coop-UOB-REPS} & $\sqrt{H^2 K} + \sqrt{\frac{H^4 S^2 A K}{\numagents}}$ & $\sqrt{H^2 K} + \sqrt{\frac{H^3 S A K}{\numagents}}$ & fresh & adversarial & unknown
            \\
            \hline
            \texttt{coop-ULCAE} & $\sqrt{H^5 S K} + \sqrt{\frac{H^7 S A K}{\sqrt{\numagents}}}$ & $\sqrt{H^2 S K} + \sqrt{\frac{H^3 S A K}{\numagents}}$ & non-fresh & stochastic & unknown
            \\
            \hline
            \texttt{coop-nf-O-REPS} & $\sqrt{H^2 S K} + \sqrt{\frac{H^2 S A K}{\numagents}}$ & $\sqrt{H^2 S K} + \sqrt{\frac{H^2 S A K}{\numagents}}$ & non-fresh & adversarial & known
            \\
            \hline
            \texttt{coop-nf-UOB-REPS} & $\sqrt{H^4 S^2 K}$ (*) & $\sqrt{H^2 S K} + \sqrt{\frac{H^3 S A K}{\numagents}}$ & non-fresh & adversarial & unknown
            \\
            \hline
        \end{tabular}
        \label{table: comparison}
    \end{center}
\end{table*}

Cooperative multi-agent reinforcement learning (MARL; see \citet{zhang2021multi}) achieved impressive empirical success in many applications such as cyber-physical systems \citep{adler2002cooperative, wang2016towards}, finance \citep{lee2002stock, lee2007multiagent} and sensor/communication networks \citep{cortes2004coverage,choi2009distributed}.
Many of the theoretical work on MARL focus on Markov Games (MGs) \citep{jin2021v}, where agents have a shared state, the transition and cost is affected by all agents' actions, and usually the goal is to converge to an equilibrium. On the other hand, in cooperative learning \citep{lidard2021provably}, agents do not affect each other and the notion of equilibrium becomes irrelevant. In this model, agents share information with each other and the goal here is to utilize the shared information in order to obtain a significant improvement upon single-agent performance. This model is a generalization of the extensively studied Multi-agent multi-armed bandit (MAB) model \citep{cesa2016delay} and reveals many new challenges that are unique to MDPs. 

In this paper we initiate the study of two topics not addressed before in the Cooperative MARL literature.
First, we differentiate between two types of randomness: \emph{fresh} -- where each agent's trajectory is sampled i.i.d, and \emph{non-fresh} -- where at any time the cost and transition kernel's randomness is shared by all agents. More precisely, if at the same time two different agents perform the same action in the same state, they observe the same cost and the same next state.
Second, we consider cooperation in the challenging adversarial MDP setting that generalizes stochastic MDPs and allows to model temporal changes in the environment through costs that change arbitrarily and are chosen by an adversary.

While previous works focus mostly on fresh randomness, non-fresh randomness models are just as well-motivated since different agents might experience the same dynamics and rewards when visiting the same state simultaneously. 
Conceptually, non-fresh randomness models cases where the randomness is more a function of the time than the agent.
For example, drones that fly together experience similar weather conditions and autonomous vehicles that drive on the same roads on the same time encounter the same traffic congestion. 
Moreover, the non-fresh randomness model is theoretically highly challenging, as we show in this paper.
We indicate a gap in the lower bounds between fresh and non-fresh randomness and identify the weaknesses of current optimistic approaches in handling non-fresh randomness, thus requiring us to develop new algorithmic techniques.

Our main contributions can be summarized as follows.
First, we derive multi-agent versions of known regret minimization algorithms in stochastic and adversarial MDPs, and thoroughly analyze their regret in the fresh randomness model.
To complement our bounds, we formally prove matching lower bounds (the adversarial MDP with unknown transitions lower bound nearly-matches, as optimal regret has not been achieved even for a single agent).
Second, we point to the failure of optimistic methods under non-fresh randomness and prove lower bounds that reveal a significant gap from the fresh randomness case.
Our novel constructions for these lower bounds carefully take advantage of the agents' shared random seed to make sure that they cannot explore different areas of the environment simultaneously.
Third, we develop a novel multi-agent action-elimination based algorithm for stochastic MDP with non-fresh randomness that forces the agents to scatter at a carefully chosen time so that exploration is maximized.
Through novel analysis of the relations between the policies of different agents and the error propagation, we prove near-optimal regret for the algorithm.
Finally, for adversarial MDP (where action-elimination is not possible) with non-fresh randomness, we design a novel exploration mechanism to replace optimism and show that it can achieve near-optimal regret for a large number of agents.
\cref{table: comparison} summarizes all our regret lower and upper bounds. 

\subsection{Related Work}

\textbf{Multi-agent multi-armed bandit.}
Cooperation was previously studied in both stochastic MAB \citep{dubey2020cooperative,wang2020optimal,madhushani2021one,landgren2021distributed} and adversarial MAB \citep{cesa2016delay,cesa2019delay,bar2019individual,ito2020delay}.
While we extend some ideas from the MAB literature to RL, many of the challenges that this paper faces do not arise in MAB.
Notably, fresh vs. non-fresh randomness, which is the main focus of the paper, is unique to MDPs as it involves dynamics (i.e., state transitions) that do not exist in MAB.

A different line of work studies \textit{collisions} in cooperative MAB \cite{liu2010distributed,bubeck2021cooperative}, where agents are penalized by choosing the same action. In these works, the goal of the cooperation is to minimize the number of collisions rather than improve the agents' performance. Thus, this line of work which has different motivation is only marginally related to our work. See  \citet{bubeck2021cooperative} for a more thorough literature review on collisions in MAB.

\textbf{MARL.}
There is a long line of research on the theoretical aspects of MARL, mainly focusing on MGs \citep{littman1994markov,bai2020provable,bai2020near,xie2020learning,zhang2020model,liu2021sharp,jin2021v}.
As mentioned before, this literature is only partially related to our model since it aims to converge to an equilibrium rather than minimize individual regret. 
More related is the literature on decentralized MARL that considers stochastic MDPs with fresh randomness.
However, the theoretical guarantees provided by these works are either asymptotic or less tight than our bounds \citep{zhang2018networked,zhang2018fully,zhang2021finite,lidard2021provably}.

\textbf{Single-agent RL.}
There is a rich literature on regret minimization in both stochastic \citep{jaksch2010near,azar2017minimax,jin2018q,jin2020learning,jin2020provably,yang2019sample,zanette2020frequentist,zanette2020learning} and adversarial \citep{zimin2013online,rosenberg2019online,rosenberg2019onlineb,rosenberg2021stochastic,jin2020simultaneously,cai2020provably,luo2021policy} MDPs.
Note that for a single agent, fresh and non-fresh randomness are identical.

\section{Preliminaries}

A finite-horizon episodic MDP $\calM$ is defined by the tuple $(\calS , \calA , H , p, \{ c^{k} \}_{k=1}^K)$. 
$\calS$ (of size $S$) and $\calA$ (of size $A$) are finite state and action spaces, $H$ is the horizon and $K$ is the number of episodes. 
$p$ is a transition function such that the probability to move to state $s'$ when taking action $a$ in state $s$ at time $h$ is $p_h(s' | s,a)$.
$c^k \in [0,1]^{H S A}$ is the cost function for episode $k$.
In the \textit{adversarial} setting the sequence of cost functions $\{ c^k \}_{k=1}^K$ is chosen by an oblivious adversary before the interaction starts, while in the \textit{stochastic} setting the costs are sampled i.i.d from a stationary distribution (that does not depend on $k$) with mean $c^k_h(s,a) = c_h(s,a)$.
The adversarial MDP model generalizes stochastic MDPs.

A policy $\pi$ is a function such that $\pi_h(a | s)$ gives the probability to take action $a$ in state $s$ at time $h$.  
If $\pi$ is deterministic we often abuse notation and use $\pi_h(s)$ for the action chosen by the policy.
Given a cost function $c$, the value $V^\pi_h(s)$ of $\pi$ is the expected cost when starting from state $s$ at time $h$, i.e., $V_{h}^{\pi}(s) = \bbE^{p,\pi} [ \sum_{h'=h}^{H}c_{h'}(s_{h'},a_{h'}) | s_{h}=s]$ where the notation $\bbE^{p,\pi} [\cdot]$ means that actions are chosen by $\pi$ and transitions are determined by $p$. 
We also define the $Q$-function $Q_{h}^{\pi}(s,a) = \bbE^{p,\pi} [ \sum_{h'=h}^{H}c_{h'}(s_{h'},a_{h'}) \mid s_h = s,a_h = a]$ that satisfies the Bellman equations \citep{sutton2018reinforcement}:
\begin{align}
    \nonumber
    Q_{h}^{\pi}(s,a) 
    & = 
    c_{h}(s,a)+\bbE_{p_h(\cdot\mid s,a)} 
    \big[V_{h+1}^{\pi}\big]
    \\
    V_{h}^{\pi}(s)  
    & =
    \langle \pi_{h}(\cdot\mid s),Q_h^\pi(s,\cdot)\rangle,
    \nonumber
\end{align}
where $\bbE_{r(\cdot)} [f]$ denotes the expectation of $f(x)$ where $x$ is sampled from the distribution $r$, and $\langle \cdot, \cdot \rangle$ is the dot product.

\textbf{Multi-agent interaction.}
A team of $\numagents$ agents interacts with the MDP $\calM$.
At the beginning of episode $k$, every agent $v \in [\numagents]$ picks a policy $\pi^{k,v}$ and starts in the initial state $s_1^{k,v} = s_{\text{init}}$. 
At time $h=1,\dots,H$, each agent observes its current state $s^{k,v}_h$ and samples an action $a^{k,v}_h \sim \pi^{k,v}_{h} (\cdot | s^{k,v}_h)$.
In the \textit{fresh randomness} model, the next state is sampled independently for each agent, i.e., $s^{k,v}_{h+1} \sim p_{h}(\cdot | s^{k,v}_h,a^{k,v}_h)$. 
For \textit{non-fresh randomness}, the next state is sampled once for each state-action pair $S^k_h(s,a) \sim p_{h}(\cdot | s, a)$ ahead of the episode, and then every agent $v$ that takes action $a$ in $s$ at time $h$ transitions to the same state $S^k_h(s,a)$, i.e., $s^{k,v}_{h+1} = S^k_h(s^{k,v}_h,a^{k,v}_h)$.
Similarly, the cost $C^{k,v}_h$ suffered by the agent is either  sampled independently when randomness is fresh, or sampled once for each state-action pair $(s,a)$ ahead of the episode when randomness is non-fresh.
Note that for adversarial MDPs the costs are not stochastic so it is always the case that $C^{k,v}_h = c^k_h(s^{k,v}_h,a^{k,v}_h)$.
At the end of the episode, the team observes the trajectories and costs of all agents $\{ s^{k,v}_h,a^{k,v}_h,C^{k,v}_h \}_{h=1,v=1}^{H \quad ,m}$ (i.e., bandit feedback).

\textbf{Regret.}
Let $V^{k,\pi}$ be the value function of $\pi$ with respect to $c^k$.
The pseudo-regret of an agent $v$ is the cumulative difference between the values of its policies and the values of the best fixed policy in hindsight. 
The performance of the team is measured by the maximal individual pseudo-regret (note that this criterion is stronger than the average regret):
\[
    \regret
    =
    \max_{v\in[m]}
    \sum_{k=1}^{K} V_{1}^{k,\pi^{k,v}}(s_\text{init}) - \min_{\pi} \sum_{k=1}^{K}  V_{1}^{k,\pi}(s_\text{init}).
\]
For stochastic MDP, we use the more common definition of the regret in which, for $k \in [K]$, $V^{k,\pi}_h(s) = V^{\pi}_h(s)$ for the cost function $c$ (the mean of the costs distribution).

\textbf{Occupancy measures.}
For policy $\pi$, let $q^\pi$ be its occupancy measure such that $q^\pi_h(s)$ is the probability to visit $s$ at time $h$ playing $\pi$ and $q^\pi_h(s,a) = q^\pi_h(s) \pi_h(a \mid s)$.
By definition, $V^{k,\pi}_1(s_\text{init}) = \langle q^\pi , c^k \rangle$, so we can write the regret in terms of occupancy measures as follows:
\[
    \regret
    =
    \max_{v \in [m]} \sum_{k=1}^K \langle q^{\pi^{k,v}} , c^k \rangle - \min_{q \in \ocset} \sum_{k=1}^K \langle q , c^k \rangle,
\]
where $\ocset$ is the set of valid occupancy measures which corresponds to the set of stochastic policies and is a convex polytope in $\bbR^{HSA}$ defined by $O(HSA)$ linear inequalities.

\textbf{Additional notations.}
The notation $\wt O(\cdot)$ hides constants, lower order terms and poly-logarithmic factors, including $\log (K/\delta)$ for some confidence parameter $\delta$. 
For $n \in \bbN$ we denote $[n] = \{1,2,\dots,n\}$, the indicator of event $E$ is denoted by $\indevent{E}$, and $x \vee y = \max \{x,y\}$.
$\pi^\star$ denotes the optimal policy (best in hindsight for the adversarial case).

\section{Fresh Randomness}

The main principle that guides us in the design of algorithms for fresh randomness is the following:
even if all agents play the same policy, the team still gathers ``$\numagents$ times more data''.
Thus, we take a single-agent regret minimization algorithm $\texttt{ALG}$ and let all the agents play the policy that it outputs.
$\texttt{ALG}$ is then updated based on the observations of all agents.

For the stochastic setting we propose an optimistic algorithm we call \texttt{coop-ULCVI} based on the single-agent algorithms of \citet{azar2017minimax,dann2019policy}. 
The algorithm maintains empirical estimates of the transition probabilities and costs, based on samples from all agents. 
At the beginning of episode $k$ it constructs an optimistic estimate $\underl{Q}^k$ of the optimal $Q$-function $Q^*$ so that with high probability (w.h.p) $\underl{Q}^k_h(s,a) \leq Q^*_h(s,a)$. 
The agents all play the same deterministic policy which is greedy with respect to $\underl{Q}^k$: $\pi^{k,v}_h(s) = \argmax_a \underl{Q}_h^k(s,a)$.
Even if agents arrive together at the same state and take the same action, we still get multiple i.i.d samples. 
Hence, the empirical estimates are based on $m$ times more samples compared to the non-cooperative (single-agent) setting.
This is the key property that allows us to prove the following improved regret bound.
Detailed description of the \texttt{coop-ULCVI} algorithm and the proof of \cref{thm-paper:reg-coop-ulcvi} appear in  \Cref{appendix:fresh-stochastic}.

\begin{theorem}
    \label{thm-paper:reg-coop-ulcvi}
    For stochastic MDP with fresh randomness, \verb|coop-ULCVI| ensures w.h.p,
    $
        \regret
        =
        \wt O \big( \sqrt{\frac{H^3 S A K}{\numagents}} \big).
    $
\end{theorem}

This bound improves upon \citet{lidard2021provably} by a factor of $\sqrt{H}$, and is in fact optimal up to logarithmic factors as shown by our lower bound in \cref{appendix: lower bounds}. 
The lower bound is built on a simple observation \citep{ito2020delay}: minimizing the sum of regrets of $m$ agents in $K$ episodes is harder than minimizing the regret of a single agent in $mK$ episodes. 
By single-agent lower bound \citep{domingues2021episodic}, the sum of regrets is $\Omega(\sqrt{H^3 S A K \numagents})$, so that the lower bound on the average regret matches our regret bound in \cref{thm-paper:reg-coop-ulcvi}.

For the adversarial setting we propose \texttt{coop-O-REPS} which is based on the single-agent O-REPS algorithm \cite{zimin2013online}.
Essentially, this is Online Mirror Descent \citep{shalev2011online} on the set of occupancy measures with entropy regularization.
More specifically, in episode $k$ all agents play policy $\pi^k$ computed as follows:
\begin{align}
    \label{eq:o-reps-update-rule}
    q^{\pi^k} = \argmin_{q \in \ocset} \eta \langle q , \hat c^{k-1} \rangle + \KL{q}{q^{\pi^{k-1}}},
\end{align}
where $\KL{\cdot}{\cdot}$ is the KL-divergence, $\eta$ is a learning rate and $\hat c^k$ is an importance sampling estimator.
The main difference in our algorithm is the new estimator that incorporates the observations from all the different agents as follows:
\begin{align}
    \hat c^k_h(s,a) = \frac{c^k_h(s,a) \indevent{\exists v: \  s^{k,v}_h = s,a^{k,v}_h = a}}{W^k_h(s,a) + \gamma},
    \label{eq-paper:importance-sampling estimator}
\end{align}
where $\gamma$ is a bias added for high-probability regret \citep{neu2015explore}, and $W_h^k(s,a)$ is the probability that \textit{some} agent visits state $s$ and takes action $a$ at time $h$ -- this quantity will play a major role in the analysis of all our algorithms.
Conveniently, ${W_h^k(s,a) = 1 - (1 - q_h^{\pi^k}(s,a))^m}$ as it is the complement of the event that all agents do not visit $(s,a)$ at time $h$.
Thus, the algorithm can be implemented efficiently similarly to the single-agent algorithm \citep{zimin2013online}.

For unknown transitions we propose \verb|coop-UOB-REPS| based on single-agent UOB-REPS \citep{jin2020learning}, which is similar to \verb|coop-O-REPS| but uses an estimate $\Delta(\calM,k)$ of the set of occupancy measures which contains the true set $\Delta(\calM)$ with high probability. 
Note that without knowing $p$, we cannot compute $W_h^k$. Instead, we use an optimistic estimate $U_h^k$ which bounds $W_h^k$ from above with high probability.
The full algorithms and analysis for the adversarial setting with fresh randomness appear in \cref{appendix:adversarial-fresh-known,appendix:adversarial-fresh-unknown}.

\begin{theorem}
    \label{thm-paper:reg-coop-o-reps}
    For adversarial MDP with fresh randomness, \verb|coop-O-REPS| ensures w.h.p,
    $$
        \regret
        =
        \wt O \left( H \sqrt{K} + \sqrt{\frac{H^2 S A K}{\numagents}} \right),
    $$
    for known dynamics, and \verb|coop-UOB-REPS| ensures w.h.p,
    $$
        \regret
        =
        \wt O \left( H \sqrt{K} + \sqrt{\frac{H^4 S^2 A K}{\numagents}} \right),
    $$
    for unknown dynamics.
\end{theorem}

In \cref{appendix: lower bounds} we show these bounds are optimal, except for an extra $\sqrt{H S}$ factor in the second term of our unknown dynamics bound which we cannot hope to remove here since it is still an open problem even for single-agent.
Notice the additional $H \sqrt{K}$ term that does not appear in the stochastic setting.
It follows from the lower bound for single-agent adversarial MDP with full-information feedback (not bandit feedback), which is equivalent to our setting in the best case scenario where the agents manage to visit all state-actions.

\begin{proof}[Proof sketch for \cref{thm-paper:reg-coop-o-reps}]
    By standard analysis, the regret scales with two terms: \emph{penalty} of order $H/\eta$, and \emph{stability} of order $\eta \sum_{k,h,s,a} q^{\pi^k}_h(s,a) \hat c^k_h(s,a)^2$.
    We show that the stability (which accounts for the estimator's variance) decreases as the number of agents increases.
    In particular we prove that, 
    \begin{align}
        \label{eq:q/W leq 1/m + q}
        q_h^{\pi^k}(s,a) 
        \leq 
        \Bigl( \frac{1}{m} + q_h^{\pi^k}(s,a) \Bigr) W_h^k(s,a). 
    \end{align}
    This implies that either the probability to observe cost $c_h^k(s,a)$ is $m$ times the probability of a single agent to observe it , or that this probability is at least a constant. 
    Hence, $q^{\pi^k}_h(s,a) \hat c^k_h(s,a)^2 \leq (1/m + q^{\pi^k}_h(s,a)) \hat c^k_h(s,a)$. 
    Then, by concentration, the stability amounts to $\eta H K (1 + S A / m)$, and optimizing over $\eta$ gives the desired bound.
    
    With unknown dynamics, there is an additional error term that comes from the estimation of the occupancy measures set and the bias of the estimator (in particular $U^k_h$).
    This error is handled similarly to the stochastic case.
\end{proof}

\section{The Challenges of Non-fresh Randomness}
\label{sec:challenges}

Unlike fresh randomness, in the non-fresh randomness setting the total amount of feedback is not necessarily $m$ times the feedback of a single agent. 
In fact, any algorithm that uses deterministic policies (e.g., optimistic algorithms) simply fails in this setting.
The reason is that all agents follow the exact same trajectory since the policy does not introduce any randomness and the transitions are fixed ahead of the episode.
Thus, the total amount of feedback is exactly the same as a single-agent would have gathered, which means $\Omega(\sqrt{H^3 S A K})$ regret with no benefit from multiple agents.

The next theorem shows that the \textit{non-fresh randomness} setting is significantly harder than fresh randomness even in terms of the statistical lower bound.
While the regret for stochastic MDP with \textit{fresh randomness} scales only logarithmically with $K$ for large enough $m$, $H\sqrt{S K}$ regret is unavoidable under \textit{non-fresh randomness} even if $m \to \infty$.

\begin{figure}[t]
    \centering
    \begin{tikzpicture}[shorten >=0pt,node distance=0.75cm,on grid,auto,scale=0.7,every node/.style={scale=0.7}] 
    	\node[state] (s_0) at (0,0)  {$s_0$}; 
    	\node[state] (s_b) at (4.5,0) {$s_{b}$}; 
    	\node[state] (s_1) at (-4,1) {$s_1$}; 
    	\node[state] (s_2) at (-2,2.5) {$s_2$}; 
    	\node[state] (s_S) at (1,2.5) {$s_S$}; 
    	\node at (-0.5,2.5) {...};
    	\node [above of = s_2]{MAB};
    	\node [above of = s_1]{MAB};
    	\node [above of = s_S]{MAB};
    	\path[-{>[scale=1.5,width=2.5]}]
    	(s_0) edge  node {$a\ne a_1$} (s_b)
    	(s_b) edge [loop above] node {$c_h^k(s_{b},\cdot) = 1$} ()
    	(s_1) edge [bend left=10]    node {$ $} (s_0)
    	(s_0) edge		[bend left=10] 		node [below,sloped,inner sep=3pt] {w.p $1/S$} (s_1)
    	(s_2) edge [bend left=10]   node [pos = 0.4,sloped] {w.p 1} (s_0)
    	(s_0) edge [bend left=10] 		node  {$ $} (s_2)
    	(s_S) edge [bend left=10]   node {$ $} (s_0)
    	(s_0) edge [bend left=10] 		node  {$ $} (s_S);
    \end{tikzpicture}
    \caption{Lower bound construction for non-fresh randomness.}
    \label{fig:lower bound}
\end{figure}
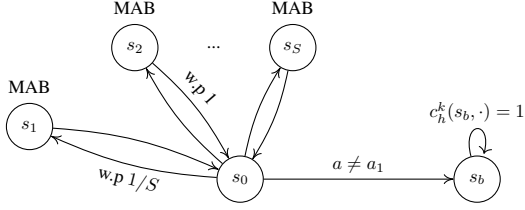

\begin{theorem}
    \label{thm-paper:lower-bound-stochastic-non-fresh}
    For any $S, A, H, \numagents \in \bbN$ and $K \ge S A H$, and for any algorithm $\texttt{ALG}$,
    there exists a stochastic MDP with non-fresh randomness such that $\texttt{ALG}$ suffers expected average regret of at least $\Omega \big(H \sqrt{S K} + \sqrt{\frac{H^3 S A K}{\numagents}}\big)$.
\end{theorem}

\begin{proof}[Proof sketch]
    We construct the following MDP illustrated in \cref{fig:lower bound}. 
    All agents start in $s_0$. 
    Taking action $a_1$ transitions to one of the MAB states $s_1,\dots,s_S$ with probability $1/S$ to each. 
    Taking any other action $a\neq a_1$ transitions to a bad state $s_b$ which is a sink with maximal cost $1$.
    Each MAB state encodes a hard MAB instance: 
    one action gives cost $0$ with probability $1/2 + \epsilon$ and cost $1$ otherwise, while the rest of the actions give cost $0$ or $1$ with probability $1/2$.
    From the MAB states all actions transition back to $s_0$ with probability $1$.
    
    Since the bad state has higher cost than every MAB state and does not contribute to exploration at all, we can assume that all agents choose action $a_1$ every time they arrive to $s_0$. 
    Recall that transitions are non-fresh, so all agents visit exactly the same states. 
    This is the critical point in our construction, as it means that exploration is limited to the $A$ actions in the states that all agents visit (and cannot remove $S$ from the regret).
    
    Choosing $\epsilon$ as in standard MAB lower bounds, we get that the regret from each of MAB state is $\Omega(\sqrt{(1+A/m)X})$, where $X$ is the total number of visits to that state. For last, $X\approx K/S$ with high probability which implies that the regret from each MAB state is $\Omega(\sqrt{(1+A/m) K / S})$. 
    Summing over all states and time steps, we get the desired lower bound.
    We note that this is a simplified version, missing a factor of $\sqrt{H}$ in the second term.
    For the full construction and other lower bounds, see \cref{appendix: lower bounds}.
\end{proof}

In \cref{sec:non-fresh-stochastic} we face non-fresh randomness in stochastic MDP and present the \texttt{coop-ULCAE} algorithm based jointly on optimism and action-elimination. 
It is important to note that, much like optimistic algorithms, existing RL action elimination algorithms (e.g., \citet{xu2021fine}) are deterministic and thus fail in the cooperative non-fresh randomness setting, even though they achieve optimal regret for single-agent.
Moreover, naive ways to make these algorithms use stochastic policies that succeed in cooperative MAB, such as uniform exploration of non-eliminated arms, lead to sub-optimal regret in RL because exploration must be controlled more carefully to ensure important states are reached with large enough probability.
Hence, we develop a novel exploration method for our algorithm which guarantees that agents can deviate from the optimistic policy and explore potentially optimal actions with minimal effect on the regret.

For adversarial MDP, non-fresh randomness introduces an additional challenge.
Due to correlations between the trajectories of different agents, there is no clear and simple relation between $W^k_h(s,a)$ and $q^{\pi^k}_h(s,a)$ as in the fresh randomness setting. 
In fact, \cref{eq:q/W leq 1/m + q} does not hold anymore. 
In \cref{sec:non-fresh-adversarial-known} we present a sophisticated technique for bounding the ratio $q_h^{\pi^k}(s,a)/W_h^k(s,a)$ through a Linear Programming formulation. 
This allows us to prove optimal regret bounds for adversarial cost and known dynamics. 

Existing algorithms for adversarial cost and \textit{unknown} dynamics are optimistic in essence, and as mentioned before, such algorithms fail to utilize cooperation under non-fresh randomness.
In \cref{sec:non-fresh-adversarial-unknown} we overcome this challenge with a novel exploration mechanism and prove regret that does not depend on $A$ (up to logarithmic factors) if there are at least $\sqrt{K}$ agents.
However, finding the optimal regret for general $m$ still remains an important open question.

\section{Non-fresh Randomness - Stochastic MDP}
\label{sec:non-fresh-stochastic}

As outlined in \cref{sec:challenges}, under non-fresh randomness, we cannot let all agents play the optimistic policy as in the fresh randomness case. 
Hence, we want agents to occasionally deviate from the optimistic policy for the purpose of exploration. 
A naive approach would be to let agents explore a random action with probability $\epsilon$ and to follow the optimistic policy with probability $1-\epsilon$. 
That way, we get $m \epsilon $ more feedback for $m\geq 1/\epsilon$ and the regret of playing the optimistic policy would scale as $\sqrt{SAK/(m\epsilon)}$ (ignoring dependency in $H$). On the other hand, deviating with an arbitrary action can lead to cost of order of $H$, which happens for approximately $\epsilon K$ episodes, so one would have to set $\epsilon \leq \sqrt{SA/K}$ in order to obtain improvement over single agent regret. Thus, the number of agents must be at least $\numagents \geq 1/\epsilon \geq \sqrt{K/(SA)}$ for an improvement. 
In this section we significantly reduce the number of agents required for a gain in the regret, and show that it can depend on $A$ alone.

Another natural approach, which leads to optimal regret in cooperative MAB, is \emph{action-elimination}, i.e., eliminate all actions that are clearly sub-optimal and explore uniformly at random over non-eliminated actions.
However, this approach would also fail in RL because it does not explore efficiently enough.
More precisely, agents deviate too much from the optimistic policy so we cannot guarantee that they visit important states.
A closer look at action-elimination algorithms for RL \citep{xu2021fine} reveals that they use deterministic policies for this very reason.

\begin{algorithm}[t]
    \caption{\textsc{coop-ULCAE}} 
    \label{alg:paper-coop-ulcae}
    \begin{algorithmic}[1]

        \STATE {\bf initialize:} $\calA^0_h(s) = \calA$ for every $s \in \calS$ and $h \in [H]$.

        \FOR{$k=1,\dots,K$}
            
            \STATE \textbf{Compute} $ \underl{Q}^k,\overl{Q}^k $ based on empirical estimates.
            
            \STATE \textbf{Set} optimistic policy $\underl{\pi}^k_h(s) \in \argmin_{a \in \calA} \underl{Q}^k_h(s,a)$.
            
            \STATE \textbf{Eliminate} sub-optimal actions:  remove $a$ from $\calA^k_h(s)$ if $\exists a' \in \calA^k_h(s)$ s.t. $\underl{Q}_{h}^{k}(s,a)>\overl{Q}_{h}^{k}(s,a')$.

            \STATE \textbf{Set policies for agents}: for every $v \in [m]$ sample $h_v \in [H]$ uniformly at random
            and set:
            $$
                \pi^{k,v}
                =
                \begin{cases}
                        \underl{\pi}^{k}  & \text{with probability } 1 - \epsilon
                        \\
                        \pi^{k,h_v}        & \text{with probability } \epsilon,
                \end{cases}
            $$
                where $\pi^{k,h'}_h = \underl{\pi}_{h}^{k}$ for any $h\ne h'$ and uniform over $\calA^k_h(s)$ at $(h,s)$.
                
            \STATE Play episode $k$, observe feedback and update empirical estimates.
        \ENDFOR
    \end{algorithmic}
\end{algorithm}

Our algorithm, cooperative upper lower confidence action-elimination (\verb|coop-ULCAE|), is presented in \cref{alg:paper-coop-ulcae} and in its full version (together with the full analysis) in \cref{appendix:non-fresh stochastic}. 
It takes inspiration from the two previous approaches but utilizes multi-agent exploration in a nearly optimal way.
It explores only over non-eliminated actions, but also makes sure that deviation from the optimistic policy is minimal, thus avoiding ``non-important'' states.
This is achieved by playing a random non-eliminated action only at one step during the episode, selected uniformly at random.

Formally, the algorithm maintains a set of \textit{active actions} in each state $\calA^k_h(s)$, consisting of only potentially-optimal actions. 
In episode $k$, it computes optimistic and pessimistic estimates of $Q^*$, $\underl{Q}^k $ and $\overl{Q}^k$, respectively, such that w.h.p $\underl{Q}^k_h(s,a) \leq Q^*_h(s,a) \leq \overl{Q}^k_h(s,a)$.
Hence, if for actions $a$ and $a'$, $\underl{Q}_{h}^{k}(s,a)>\overl{Q}_{h}^{k}(s,a')$, then $a$ is clearly sub-optimal and we can eliminate it. 
The policies of the agents are determined as follows: agent $v$ plays the optimistic policy (greedy with respect to $\underl{Q}^k$) with probability $1-\epsilon$, and with probability $\epsilon$ she plays the optimistic policy except for one random time step $h_v$ where she takes a uniformly random active action. 
The key idea is that deviating on a single time step  with an active arm would have only minor affect on the regret, so we can set $\epsilon$ much larger compared to the naive $\epsilon$-exploration approach described in the beginning of this section.

\begin{theorem}
    \label{thm-paper:reg-coop-ulcae}
    For stochastic MDP with non-fresh randomness, \verb|coop-ULCAE| ensures with high probability,
    $$
        {\regret 
        = 
        \wt O \biggl(\sqrt{H^{5}SK}  + \sqrt{\frac{H^{7}SAK}{\sqrt{m}}} \biggr)}.
    $$
\end{theorem}

If $\numagents \geq H^4 A^2$, the first term is dominant, in which case the regret is nearly optimal and matches our lower bound (\cref{thm-paper:lower-bound-stochastic-non-fresh}) up to $H^{3/2}$. 
Otherwise, we have optimal dependence in $S,A,K$ but there is still a gap of $H^2$ and more importantly $1/\sqrt[4]{m}$. 
Determining the optimal dependency in $\numagents$ for this setting is an important open question.

\begin{proof}[Proof sketch for \cref{thm-paper:reg-coop-ulcae}]
    To simplify presentation, we ignore $poly(H)$ factors in the proof sketch and use the notation $V^\pi = V^{\pi}_1(\sinit)$. 
    For agent $v$, we first break the regret into episodes in which she plays the optimistic policy $\calK_{\texttt{OP}}^v$ and episodes in which she plays an exploration policy $\calK_{\texttt{EXP}}^v$:
    \[
        \underbrace{
        \sum_{k\in \calK_{\texttt{OP}}^v} V^{\underl{\pi}^k} - V^{\pi^\star}
        }
        _{R_{\texttt{OP}}^v}
        +
        \underbrace{
        \sum_{k\in \calK_{\texttt{EXP}}^v} V^{\pi^{k,h_v}} - V^{\pi^\star}}
        _{R_{\texttt{EXP}}^v}.
    \]
    Then, we show that the regret of playing $\underl{\pi}^k$ is bounded by the difference between the optimistic and pessimistic estimates of $Q^\star$ over the trajectory of $\underl{\pi}^k$. 
    This difference shrinks with the confidence radius and mainly scales as,
    \begin{align}
    \label{eq-paper:stochastic non-fresh optimistic radius}
        R_{\texttt{OP}}^v
        \lesssim
        \sum_{k=1}^K \sum_{h=1}^H \sum_{s \in \calS} \sum_{a \in \calA}\frac{q_{h}^{\underl{\pi}^{k}}(s,a) }{\sqrt{n_{h}^{k}(s,a)\vee1}},
    \end{align}
    where $n^k_h(s,a) = \sum_{j=1}^{k-1} \indevent{\exists v: \  s^{j,v}_h = s,a^{j,v}_h=a}$ is the number of times some agent visited $(s,a)$ at time $h$ before episode $k$.
    Now, one can show that $n_h^k(s,a)$ is approximately the sum of probabilities that \textit{some} agent visits $(h,s,a)$, i.e., $n_h^k(s,a) \approx \sum_{j=1}^{k-1} W_h^j(s,a)$. 
    Trivially $W_h^j(s,a) \geq q_{h}^{\underl{\pi}^{j}}(s,a)$, but we can further utilize the exploration of all the agents to bound $W_h^j(s,a)$ in terms of $q_{h}^{\underl{\pi}^{j}}(s)$ and not $q_{h}^{\underl{\pi}^{j}}(s,a)$, as follows:
    With probability ${1-(1-\nicefrac{\epsilon}{HA})^m}$ some agent plays the policy ${\pi}^{j,h}$ and takes action $a$ at time $h$. 
    In that case, she would arrive to $s$ in time $h$ with probability $q_{h}^{{\pi}^{j,h}}(s)$. Recall that ${\pi}^{j,h}$ and $\underl{\pi}^j$ are identical up to time $h$, and so $q_{h}^{{\pi}^{j,h}}(s) = q_{h}^{\underl{\pi}^{j}}(s)$. Also, it is possible to show that $1-(1-\epsilon/A)^m \approx m\epsilon/A$ whenever $\epsilon \leq A/m$. 
    Thus, we get the better bound:
    \[
        n_h^k(s,a) 
        \approx \sum_{j=1}^{k-1} W_h^j(s,a) 
        \gtrsim
        \frac{m\epsilon}{A} \sum_{j=1}^{k-1} q_{h}^{\underl{\pi}^{j}}(s).
    \]
    Combining this with \cref{eq-paper:stochastic non-fresh optimistic radius}, we obtain:
    \begin{align}
        \label{eq-paper:OP.R bound}
        R_{\texttt{OP}}^v
        \lesssim
        \sqrt{\frac{A}{m\epsilon}} \sum_{k,h,s}\frac{\sum_a q_{h}^{\underl{\pi}^{k}}(s,a) }{\sqrt{\sum_{j=1}^{k-1} q_{h}^{\underl{\pi}^{j}}(s)}} 
        \lesssim 
        \sqrt{\frac{S A K}{m\epsilon}},
    \end{align}
    where the last relation uses $\sum_{a \in \calA} q_{h}^{\underl{\pi}^{k}}(s,a) = q_{h}^{\underl{\pi}^{k}}(s)$, the Cauchy–Schwarz inequality, and standard arguments \citep[Lemma B.18]{rosenberg2020near} to bound ${\sum_k\nicefrac{q_{h}^{\underl{\pi}^{k}}(s) }{\sum_{j=1}^{k-1} q_{h}^{\underl{\pi}^{j}}(s)} \lesssim \log K}$.
    
    For $R_{\texttt{EXP}}^v$ we utilize the fact that when the agent plays an exploration policy, she deviates from the optimistic policy using an active action. 
    Particularly, we show that similar to the regret of the optimistic policy, the regret of the exploration episodes scales with the difference between the optimistic and pessimistic estimates of $Q^\star$ over the trajectory of ${\pi}^{k,h_v}$, but with additional penalty due to the deviation which is overall bounded recursively by $R_{\texttt{OP}}^v$, i.e., 
    \begin{align}
    \label{eq-paper:stochastic non-fresh explore radius}
        R_{\texttt{EXP}}^v
        \lesssim
        R_{\texttt{OP}}^v
        +
        \sum_{k\in \calK_\texttt{EXP}^v}\sum_{h,s,a}\frac{q_{h}^{{\pi}^{k,h_v}}(s,a) }{\sqrt{n_{h}^{k}(s,a)\vee1}}.
    \end{align}
    While we can bound $n_{h}^{k}(s,a)$ as before in terms of $q_{h}^{\underl{\pi}^{j}}(s)$, it can be very different than $q_{h}^{{\pi}^{k,h_v}}(s)$.
    Intuitively, once the agent deviated from the optimal policy, we have a much weaker guarantee on the quality of our confidence sets in the states that she reaches since the cooperative exploration is done over the trajectory of the optimistic policy.
    Thus, we cannot use similar arguments to the ones we used to bound $R_{\texttt{OP}}^v$, and in particular \cref{eq-paper:stochastic non-fresh optimistic radius}.
    Instead, we only utilize samples gathered by $v$ in $\calK_\texttt{EXP}^v$, and bound $n_h^k(s,a) \gtrsim \sum_{j \in \calK_\texttt{EXP}^v,j<k} q_{h}^{{\pi}^{j,h_v}}(s,a)$. 
    Using the fact that the number of exploration episodes for $v$ is approximately $|\calK_\texttt{EXP}^v| \approx \epsilon K$ and standard arguments, the second term in \cref{eq-paper:stochastic non-fresh explore radius} is bounded by $\sqrt{SAK\epsilon}$. 
    To finish, combine the bounds and set $\epsilon = \min \{ \frac{ A}{\numagents},\frac{1}{\sqrt{m}}\}$.
\end{proof}

\section{Non-fresh Randomness - Adversarial MDP}
\label{sec:non-fresh-adversarial}

\subsection{Known Transitions}
\label{sec:non-fresh-adversarial-known}

Before tackling the most challenging model -- adversarial MDP with non-fresh randomness and unknown transitions, we first study the case of known transitions.
While some of the challenges that we tackled in \cref{sec:non-fresh-stochastic} are alleviated when transitions are known, in this section we face additional challenges that stem from the fact that now an adversary is choosing the sequence of cost functions instead of them being sampled from a fixed distribution.

More precisely, under non-fresh randomness there are strong correlations between the trajectories of different agents.
This is in stark contrast to the fresh randomness setting where, by playing the same policy for all agents, we obtained different i.i.d samples which enabled us to prove that our estimator has reduced variance and therefore get an improved regret bound (\cref{thm-paper:reg-coop-o-reps}).
The correlations between the agents' trajectories introduce two main challenges: a statistical challenge and a computational challenge.

The statistical challenge resembles the ones we faced in the stochastic case (\cref{sec:non-fresh-stochastic}) -- whenever the policy is close to being deterministic, the trajectories of the agents are almost identical.
However, since costs are adversarial, here we need different techniques that are compatible with adversarial online learning.
More formally, for a near-deterministic policy $\pi^k$, we have $W_h^k(s,a) \approx q_h^{\pi^k}(s,a)$ which means that there is almost no benefit from the cooperation between the agents.
Even though algorithms for adversarial environments inherently choose stochastic policies, it is still unclear a-priori how to bound the ratio $q_h^{\pi^k}(s,a)/W_h^k(s,a)$, except for the trivial bound of $1$ which leads to single-agent regret guarantees.
Recall that for fresh randomness we bounded this ratio in \cref{eq:q/W leq 1/m + q} which relies on the independence of the agents' trajectories given the policy $\pi^k$.
However, this bound no longer holds and we develop a new bound that is suitable to the 
non-fresh randomness setting and builds on a novel LP formulation.

The computational challenge follows because we no longer have a closed-form expression to compute $W^k_h(s,a)$ like we had under fresh randomness.
We propose to solve this challenge by Monte Carlo estimation of $W^k_h(s,a)$ which we show does not damage the final regret guarantees.

For this setting we propose \texttt{coop-nf-O-REPS}, presented (together with its analysis) in \cref{appendix:adversarial non-fresh adversarial}.
It follows the same update rule (\cref{eq:o-reps-update-rule}) of \texttt{coop-O-REPS}, but instead of $W_h^k(s,a)$ (which is now hard to compute) in the definition of the importance sampling estimator (\cref{eq-paper:importance-sampling estimator}), it uses a Monte Carlo estimate $\wt W_h^k(s,a)$.
The estimate is computed by simulating the run of multiple agents playing policy $\pi^k$ over the MDP for $\wt O(K)$ times and taking the fraction of times where some agent visited the state-action pair $(s,a)$ at time $h$. 
The approximation error adds a small bias of order $O(1/\sqrt{K})$ which affects the total regret by only a constant factor.
We note that the computational complexity of this algorithm is similar to standard O-REPS-based algorithms which are known to have $poly(K,H,S,A)$ per-episode computational complexity \citep{dick2014online}.

\begin{theorem}
    \label{thm-paper:reg-coop-nf-o-reps}
    For adversarial MDP with non-fresh randomness and known dynamics, \verb|coop-nf-O-REPS| ensures w.h.p,
    $$
        \regret
        =
        \wt O \left( H \sqrt{S K} + \sqrt{\frac{H^2 S A K}{\numagents}}  \right).
    $$
\end{theorem}

The above regret bound is optimal up to logarithmic factors.
The lower bound can be found in \cref{appendix: lower bounds}, and features a similar construction to the one in the proof of \cref{thm-paper:lower-bound-stochastic-non-fresh}. Note that in \cref{thm-paper:lower-bound-stochastic-non-fresh} there is an extra $\sqrt{H}$ factor in the second term, which only appears for unknown dynamics.

\begin{proof}[Proof sketch for \cref{thm-paper:reg-coop-nf-o-reps}]
    Similarly to the proof of \cref{thm-paper:reg-coop-o-reps}, the regret scale with the penalty term $H / \eta$, and the stability term $\eta  \sum_{k,h,s,a} q^{\pi^k}_h(s,a) \hat c^k_h(s,a)^2$.
    Bounding the approximation error of $\wt W_h^k(s,a)$ by $\gamma/2$
    gives us:
    \begin{align}
    (stability)
        \lesssim
        \eta \sum_{k,h,s,a} \frac{q^{\pi^k}_h(s,a) }{W_h^k(s,a)} \hat c^k_h(s,a) 
        \label{eq-paper:nf basic reg bound}
    \end{align}
    To further bound the right-hand-side we use a standard concentration bound of $\hat c^k_h(s,a)$ around $c^k_h(s,a) \leq 1$. It remains to bound the ratio $q^{\pi^k}_h(s,a)/{W_h^k(s,a)}$ which is our main technical novelty in this proof. 
    Let $M^k_{h}(s)$ be the random variable that represents the number of agents that arrive at state $s$ in time $h$ and denote $p_i = \Pr[M^k_h(s) = i]$. 
    Let $\bbEk[ \cdot]$ denote an expectation conditioned on everything that occurred before the start of episode $k$.
    By definition,
    \begin{align}
    \nonumber
        W_h^k(s,a) 
        = 
        \bbE^k \Big[  1-(1-\pi^k_{h}(a\mid s))^{M^k_{h}(s)}\Big] 
        \quad\qquad 
        \\
        \geq \bbE^k \left[\frac{M^k_{h}(s)\pi^k_{h}(a\mid s)}{1+M^k_{h}(s)\pi^k_{h}(a\mid s)} \right] 
        = \sum_{i=0}^{m} \frac{p_{i} i \pi^k_{h}(a\mid s)}{1+i\pi^k_{h}(a\mid s)},
        \label{eq-paper:W bound}
    \end{align}
    where the inequality holds deterministically for every realization of $M^k_{h}(s)$ (see \cref{lemma:linear-approx}). 
    Note that the expected value of $M^k_{h}(s)$ is  $m q^{\pi^k}_h(s)$, so the right-hand-side of Eq. \eqref{eq-paper:W bound} is bounded by the value of the following linear program:
        \begin{align*}
    	   \min_{p_{0},...,p_{m}} & \sum_{i=0}^{m} p_{i}\frac{i\pi^k_{h}(a\mid s)}{1+i\pi^k_{h}(a\mid s)}
    	  \\
    	  s.t. \quad 
    	  & \sum_{i=0}^{m} p_{i}i = m q_{h}^{k}(s) \quad;\quad \sum_{i=0}^{m} p_{i} = 1.
    \end{align*}
    Now, we can solve the LP by considering the dual problem (see \cref{lemma:linear-approx-non-fresh}), and get that $W_h^k(s,a) \geq \frac{m q_{h}^{\pi^k}(s) \pi^k_{h}(a \mid  s)} {1+m\pi^k_{h}(a\mid s)} $. 
    Hence, Eq. \eqref{eq-paper:nf basic reg bound} is bounded by $\eta H S K (1 + A / m)$, and we obtain the claim by optimizing over $\eta$.
\end{proof}

\subsection{Unknown Transitions}
\label{sec:non-fresh-adversarial-unknown}

When facing adversarial MDPs with unknown transitions and non-fresh randomness, we encounter all the challenges presented in previous sections.
The combination of these challenges makes this model especially hard from an algorithmic perspective.
Specifically, the only way (currently) to obtain regret bounds in adversarial MDPs with unknown dynamics is via optimism. Unfortunately, as discussed in \cref{sec:challenges,sec:non-fresh-stochastic}, optimistic methods fail under non-fresh randomness. 
Moreover, our solution for the stochastic case is based on action-elimination so it cannot be extended to adversarial costs. 
Instead, in this section we present a novel exploration mechanism which guarantees near-optimal regret for large enough number of agents.
Importantly, if we used optimism, regret would not improve even for $m \to \infty$.

We propose the \texttt{coop-nf-UOB-REPS} algorithm, presented in \cref{alg:paper-coop-nf-uob-reps} and in full version (together with analysis) in \cref{appendix:adversarial-non-fresh-unknown}. 
Similarly to \texttt{coop-O-REPS} and \texttt{coop-nf-O-REPS}, it maintains a policy $\pi^k$ through the O-REPS update rule (\cref{eq:coop-nf-uob-reps-update-rule}),
however unlike the previous algorithms, some agents play a different policy than $\pi^k$ for the purpose of exploration.
We now present the two key features that allow our algorithm to perform efficient exploration in this challenging setting.

First, we equip the algorithm with a novel exploration mechanism: for every $(h,a,k)$ we assign an agent $\sigma(h,a,k)$ to follow  $\pi^k$ up to time $h$, and then take action $a$. 
The rest of the agents follow the policy $\pi^k$. 
This exploration mechanism is motivated by \texttt{coop-ULCAE}, but since costs are adversarial, we cannot eliminate actions and thus have much weaker guarantees on the regret of the exploration policies. 
As a result, we require many agents so that each agent would explore less often. 
In particular, there are $HAK$ targets to explore and whenever $\numagents \geq \sqrt{K}$ we can choose $\sigma$ so that each agent performs exploration for at most $HA\sqrt{K}$ episodes.
Second, to avoid the complex dependencies between the agents' trajectories, we use a new importance sampling estimator that ignores all agents except for $\sigma(h,a,k)$, i.e.,
$$
    \hat c^k_h(s,a) = \frac{c^k_h(s,a) \indevent{s^{k,\sigma(h,a,k)}_h = s}}{u^k_h(s) + \gamma},
$$
where $u^k_h(s) \approx q^{\pi^k}_h(s)$.
Notice that this estimator is approximately unbiased (up to $\gamma$ and approximation errors) since $\indevent{a^{k,\sigma(h,a,k)}_h = a} = 1$.
Finally, since we do not know the set occupancy measures under unknown dynamics, we use an approximation $\Delta(\calM,k)$ based on empirical estimates.

\begin{algorithm}[t]
    \caption{\textsc{coop-nf-UOB-REPS}} 
    \label{alg:paper-coop-nf-uob-reps}
    \begin{algorithmic}[1]
        \STATE {\bf initialize:} define a mapping $\sigma: [H] \times \calA \times [K] \to [\numagents]$.
        \FOR{$k=1,\dots,K$}
            
            \STATE \textbf{Compute} $\pi^{k}_h(a\mid s) = q^{k}_h(s,a) / q^{k}_h(s)$ for:
            \begin{align}
                \label{eq:coop-nf-uob-reps-update-rule}
                q^{k} = \argmin_{q \in \Delta(\calM,k)} \eta \langle q , \hat c^{k-1} \rangle + \KL{q}{q^{k-1}}.
            \end{align}
    
            \STATE \textbf{Set policies for agents}: For every $(h,a,k)$ set the policy of agent $v = \sigma(h,a,k)$ to be:
                $$
                \pi^{k,v}_{h'}(a'\mid s)
                =
                \begin{cases}
                        \indevent{a' = a}      &  h' = h
                        \\
                        {\pi}^{k}_{h'}(a' \mid s)  & h' \ne h,
                \end{cases}
                $$
                and for the rest of the agents set $\pi^{k,v} = \pi^k$.
                
            \STATE \textbf{Play} episode $k$, observe feedback, update transition empirical estimates, and compute cost estimator $\hat c^k$.
        \ENDFOR
    \end{algorithmic}
\end{algorithm}

\begin{theorem}
    \label{thm-paper:reg-coop-nf-uob-reps}
    Assume $S \geq A$ and $m \ge \sqrt{K}$. 
    For adversarial MDP with non-fresh randomness and unknown dynamics, \verb|coop-nf-O-REPS| ensures w.h.p,
    $
        \regret
        =
        \wt O ( H^2 S \sqrt{K} ).
    $
\end{theorem}

The above result shows that optimal regret is attainable up to factor of $\sqrt{HS}$. Recall that the extra factor $\sqrt{HS}$ compared to the lower bound of \cref{thm-paper:lower-bound-stochastic-non-fresh} also appear in the state-of-the-art upper bound for single-agent. There is still a significant gap on the number of agents required for optimal regret, and finding the minimal number of agents to ensure such regret still remains an important open problem.

\begin{proof}[Proof sketch for \cref{thm-paper:reg-coop-nf-uob-reps}]
    The proof focuses on bounding the regret of the O-REPS policies $\{ \pi^k \}_{k=1}^K$, since the number of episodes that each agent does not play these policies is at most $H A \sqrt{K}$, resulting in extra regret of at most $H^2 S \sqrt{K}$.
    Similarly to the proof of \cref{thm-paper:reg-coop-nf-o-reps} we need to bound the stability term, but here the unknown dynamics also introduce additional approximation errors.
    The analysis of the stability term resembles the proof of \cref{thm-paper:reg-coop-nf-o-reps} but utilizes the fact that $q^{\pi^k}_h(s,a)/u_h^k(s) \lesssim \pi^k_h(a\mid s)$.
    The analysis of the approximation errors takes inspiration from the proof of \cref{thm-paper:reg-coop-ulcae}, and shows that it scales with the sum of confidence radius over the trajectory of $\pi^k$:
    \[
       (*) =  \sum_{h,s,a,k}\frac{ H\sqrt{S} q^{\pi^k}_h(s,a)}{\sqrt{n_h^k(s,a)}} 
        \approx 
        \sum_{h,s,a,k}\frac{H\sqrt{S} q^{\pi^k}_h(s,a)}{\sqrt{ \sum_{j=1}^{k-1} W_h^j(s,a)}}.
    \]
    We then utilize agents' exploration to lower bound $W_h^k(s,a)$. 
    Recall that agent $\sigma(h,s,k)$ follows $\pi^k$ until time $h$, so she arrives at $s$ with probability $q^{\pi^k}_h(s)$ and then takes action $a$ deterministically. Hence, $ W_h^k(s,a) \geq q^{\pi^k}_h(s)$, yielding:
    \begin{align*}
        (*) & \lesssim  \sum_{h,s,k}\frac{H\sqrt{S} \sum_a q^{\pi^k}_h(s,a)}{\sqrt{\sum_{j=1}^{k-1}q^{\pi^j}_h(s)}}
        \lesssim
        H^2 S \sqrt{K}. \qedhere
    \end{align*}
\end{proof}

\section{Conclusions and Future Work}

In this paper we studied cooperation in multi-agent RL.
We introduced the non-fresh randomness model and characterized its challenges compared to standard fresh randomness.
We provided nearly-matching regret lower and upper bounds in all relevant settings, and developed novel techniques for handling different types of randomness in various models.

Our work leaves two important directions for future work.
First, our regret bounds for non-fresh randomness with unknown transitions are not tight for both stochastic and adversarial MDPs.
Second, we assume agents communicate through a fully-connected graph.
Extending our results to general communication graphs (as in MAB) is an interesting future direction that would also require analyzing delayed feedback \citep{lancewicki2020learning,howson2021delayed}.

\section*{Acknowledgements}

This project has received funding from the European Research Council (ERC) under the European Union’s Horizon 2020 research and innovation program (grant agreement No. 882396), by the Israel Science Foundation (grant number 993/17), Tel Aviv University Center for AI and Data Science (TAD), and the Yandex Initiative for Machine Learning at Tel Aviv University.

\newpage
\bibliography{example_paper}
\bibliographystyle{icml2022}

\onecolumn
\newpage
\clearpage

\appendix

\paragraph{Additional notations.}
While in the main paper the notation $\lesssim$ hides lower order terms and logarithmic factors, in the appendix it only hides constant factors, i.e., $x \lesssim y$ if and only if $x = O(y)$.
We use the notation $\bbEk[ \cdot]$ to denote an expectation conditioned on everything that occurred before the beginning of episode $k$.
Furthermore, $\bbEk[ \cdot \mid \pi]$ denotes an expectation conditioned on everything that occurred before the beginning of episode $k$, and when playing episode $k$ using the policy $\pi$.
$n^k_h(s,a)$ denotes the number of samples we have from $(s,a,h)$ in the beginning of episode $k$.
More precisely, for fresh randomness $n^k_h(s,a) = \sum_{j=1}^{k-1} \sum_{v=1}^{\numagents} \indevent{s^{j,v}_h=s,a^{j,v}_h=a}$, and for non-fresh randomness $n^k_h(s,a) = \sum_{j=1}^{k-1} \indevent{\exists v: \  s^{j,v}_h=s,a^{j,v}_h=a}$.
Finally, $\VAR_{r(\cdot)} [f]$ denotes the variance of $f(x)$ where $x$ is sampled from the distribution $r$.
For transition function $p'$ and policy $\pi$, $q^{p',\pi}_h(s,a,s') = \Pr [s_h=s,a_h=a,s_{h+1}=s' \mid p',\pi]$ denotes its occupancy measure, $q^{p',\pi}_h(s,a) = \sum_{s'} q^{p',\pi}_h(s,a,s')$ and $q^{p',\pi}_h(s) = \sum_{a} q^{p',\pi}_h(s,a)$.
When the transition function is $p$, we often use the shorter notation $q^{\pi} = q^{p,\pi}$.
$\pi^\star$ denotes the optimal policy (or best in hindsight), $V^\star$ denotes its values function and $q^\star$ its occupancy measure.

\section{Lower bounds}
\label{appendix: lower bounds}

In this section we provide proofs for the lower bounds that appear in \cref{table: comparison}.

\subsection{Fresh randomness}

\begin{theorem}[Lower bound for stochastic MDP with fresh randomness]
    \label{thm:lower-bound-stochastic-fresh}
    Let $S, A, H, \numagents \in \bbN$ and $K \ge S A H$.
    For any algorithm $\texttt{ALG}$ there exists a stochastic MDP $\calM$ with fresh randomness such that: (i) $\calM$ has $\Theta (S)$ states, $\Theta (A)$ actions and horizon $\Theta (H)$; (ii) Running $\texttt{ALG}$ with $\numagents$ agents for $K$ episodes suffers expected average regret of at least $\Omega (\sqrt{\frac{H^3 S A K}{\numagents}})$.
\end{theorem}

\begin{proof}
    The proof is similar to the proof of \citet[Theorem 4]{ito2020delay}.
    Notice that cooperative regret minimization with $\numagents$ agents for $K$ episodes is harder than single agent regret minimization for $K \numagents$ episodes, because we can solve the second problem using an algorithm for the first problem.
    Simply let the agents play the first $\numagents$ episodes one by one, feed the feedback to the algorithm and then again let the agents play sequentially.
    This implies that the cumulative expected regret of the agents is at least $\Omega (\sqrt{H^3 S A K \numagents})$ by standard lower bounds for MDPs \citep{domingues2021episodic}.
    Thus, the average expected regret is at least $\Omega (\sqrt{\frac{H^3 S A K}{\numagents}})$.
\end{proof}

\begin{theorem}[Lower bound for adversarial MDP with fresh randomness and known transition]
    \label{thm:lower-bound-adversarial-fresh-known-p}
    Let $S, A, H, \numagents \in \bbN$ and $K \ge S A H$.
    For any algorithm $\texttt{ALG}$ there exists an adversarial MDP $\calM$ with fresh randomness such that: (i) $\calM$ has $\Theta (S)$ states, $\Theta (A)$ actions and horizon $\Theta (H)$; (ii) Running $\texttt{ALG}$ with $\numagents$ agents for $K$ episodes, when the transition function is known, suffers expected average regret of at least $\Omega (\sqrt{H^2 K} + \sqrt{\frac{H^2 S A K}{\numagents}})$.
\end{theorem}

\begin{proof}
    The lower bound is obtained by a combination of the following two constructions, i.e., with probability $1/2$ the MDP has the structure of the first construction and with probability $1/2$ of the other one:
    \begin{enumerate}
        \item Similarly to the proof of \cref{thm:lower-bound-stochastic-fresh}, cooperative regret minimization with $\numagents$ for $K$ episodes is harder than single agent regret minimization for $K \numagents$ episodes.
        Invoking the $\Omega (\sqrt{H^2 S A K \numagents})$ lower bound of \citet{zimin2013online} for adversarial MDP with bandit feedback and known transition, this gives us the $\Omega ( \sqrt{\frac{H^2 S A K}{\numagents}})$ lower bound.
        
        \item Cooperative regret minimization with bandit feedback is harder than single agent regret minimization with full-information feedback because the transition function is known so the agents share information only about the cost function (which is fully revealed under full-information feedback).
        Thus, for the second construction we can simply use the construction of \citet{zimin2013online} for the $\Omega(\sqrt{H^2 K})$ lower bound of single agent adversarial MDP with known transition and full-information feedback. \qedhere
    \end{enumerate}
\end{proof}

\begin{theorem}[Lower bound for adversarial MDP with fresh randomness and unknown transition]
    \label{thm:lower-bound-adversarial-fresh-unknown-p}
    Let $S, A, H, \numagents \in \bbN$ and $K \ge S A H$.
    For any algorithm $\texttt{ALG}$ there exists an adversarial MDP $\calM$ with fresh randomness such that: (i) $\calM$ has $\Theta (S)$ states, $\Theta (A)$ actions and horizon $\Theta (H)$; (ii) Running $\texttt{ALG}$ with $\numagents$ agents for $K$ episodes, when the transition function is unknown, suffers expected average regret of at least $\Omega (\sqrt{H^2 K} + \sqrt{\frac{H^3 S A K}{\numagents}})$.
\end{theorem}

\begin{proof}
    Similarly to the proof of \cref{thm:lower-bound-adversarial-fresh-known-p} we use a combination of two constructions.
    The first one is the construction from \cref{thm:lower-bound-stochastic-fresh} which gives the $\Omega(\sqrt{\frac{H^3 S A K}{\numagents}})$ lower bound, and the second one is the construction from \cref{thm:lower-bound-adversarial-fresh-known-p} which gives the $\Omega(\sqrt{H^2 K})$ lower bound.
\end{proof}

\subsection{Non-fresh randomness}

\begin{theorem}[Lower bound for adversarial MDP with non-fresh randomness and known transition]
    \label{thm:lower-bound-adversarial-non-fresh-known-p}
    Let $S, A, H, \numagents \in \bbN$ and $K \ge S A H$.
    For any algorithm $\texttt{ALG}$ there exists an adversarial MDP $\calM$ with non-fresh randomness such that: (i) $\calM$ has $\Theta (S)$ states, $\Theta (A)$ actions and horizon $\Theta (H)$; (ii) Running $\texttt{ALG}$ with $\numagents$ agents for $K$ episodes, when the transition function is known, suffers expected average regret of at least $\Omega (\sqrt{H^2 S K} + \sqrt{\frac{H^2 S A K}{\numagents}})$.
\end{theorem}

\begin{proof}
    Consider the following MDP with horizon $2H$.
    There are $A$ actions $a_1,a_2,\dots,a_A$ and $S+2$ states: the initial state $s_0$, a bad state $s_b$ and the MAB states $s_1,s_2,\dots,s_S$.
    The agent starts in the initial state $s_0$ where action $a_1$ transitions to each of the MAB states $s_1,\dots,s_S$ with probability $1/S$, and all the other actions $a_2,\dots,a_A$ transition to the bad state $s_b$.
    In the bad state the cost is always $1$ and all the actions just stay in it with probability $1$.
    Each MAB state $s_i$ encodes a hard multi-arm bandit problem for each horizon step $h$.
    That is, all the actions transition back to the initial state $s_0$, but one action (sampled uniformly at random) suffers cost $0$ with probability $1/2 + \epsilon$ (and otherwise $1$) while the other actions suffer cost $0$ with probability $1/2$, where $\epsilon \approx \sqrt{S A / K}$ which is standard for MAB$\slash$RL lower bounds.
    
    Without loss of generality we can assume that all of the agents always choose action $a_1$ in the initial state because otherwise they transition to the bad state and suffer maximal cost.
    Critically, this means that all of the agents visit the same state in every time step (because of non-fresh randomness).
    
    Denote by $T_{i,h}$ the number of visits to MAB state $s_i$ in step $h$.
    We utilize the lower bound for cooperation in multi-arm bandit \citep{seldin2014prediction,ito2020delay} in order to bound the average expected regret from below by
    \begin{align*}
        \Omega \left( \bbE \left[ \sum_{i=1}^S \sum_{h=1}^H \sqrt{\left( 1 + \frac{A }{\numagents} \right) T_{i,2h}}\right] \right) = \Omega \left( SH \sqrt{1 + \frac{A}{\numagents}} \bbE [ \sqrt{X} ] \right),
    \end{align*}
    for $X \sim Bin(n = K , p = 1/S)$ because in each even step size $2h$ one of the MAB states is sampled uniformly at random.
    By \cref{lem:bin-rv-expected-square-root}, we have $\bbE [ \sqrt{X} ] \ge \Omega (\sqrt{n p})$ for $n \ge 1/p^2$ which proves the lower bound $\Omega ( \sqrt{H^2 S K} + \sqrt{\frac{H^2 S A K}{\numagents}})$ for $K \ge S^2$.
    We note that a more involved proof of the lower bound in each state reveals that the standard assumption $K \ge H S A$ is sufficient.
    For more details see the proof of \citet[Theorem 2.7]{rosenberg2020near}.
\end{proof}

\begin{theorem}[Lower bound for stochastic MDP with non-fresh randomness]
    \label{thm:lower-bound-stochastic-non-fresh}
    Let $S, A, H, \numagents \in \bbN$ and $K \ge S A H$.
    For any algorithm $\texttt{ALG}$ there exists a stochastic MDP $\calM$ with non-fresh randomness such that: (i) $\calM$ has $\Theta (S)$ states, $\Theta (A)$ actions and horizon $\Theta (H)$; (ii) Running $\texttt{ALG}$ with $\numagents$ agents for $K$ episodes suffers expected average regret of at least $\Omega (\sqrt{H^2 S K} + \sqrt{\frac{H^3 S A K}{\numagents}})$.
\end{theorem}

\begin{proof}
    Similarly to the proof of \cref{thm:lower-bound-adversarial-fresh-known-p} we use a combination of two constructions.
    The first one is presented in the rest of the proof and gives the $\Omega(\sqrt{\frac{H^3 S A K}{\numagents}})$ lower bound, and the second one is the construction from \cref{thm:lower-bound-adversarial-non-fresh-known-p} which gives the $\Omega(\sqrt{H^2 S K})$ lower bound.
    
    Consider the following MDP with horizon $2H + 1$.
    There are $A$ actions $a_1,a_2,\dots,a_A$ and $2S+3$ states: the initial state $s_0$, a bad state $s_b$, a good state $s_g$, the MAB states $s_1,s_2,\dots,s_S$, and the wait states $s^w_1,\dots,s^w_S$.
    The agent starts in the initial state $s_0$ where action $a_1$ transitions to each of the wait states $s^w_1,\dots,s^w_S$ with probability $1/S$, and all the other actions $a_2,\dots,a_A$ transition to the bad state $s_b$.
    In the bad state the cost is always $1$ and all the actions just stay in it with probability $1$, while in the good state the cost is always $0$ and all the actions just stay in it with probability $1$.
    
    For each $i \in [S]$, the pair of states $(s_i,s^w_i)$ encodes a hard multi-arm bandit problem with $H A$ actions and costs either $0$ or $\Omega(H)$.
    In the next paragraph we describe how the MAB problem is encoded, but first notice that this achieves the desired lower bound.
    In each episode all of the agents visit the same MAB problem and do not obtain any information about the other ones.
    Thus, similarly to the proof of \cref{thm:lower-bound-adversarial-non-fresh-known-p}, we can utilize the lower bound for cooperation in MAB \citep{seldin2014prediction} in order to prove the lower bound:
    \[
        \Omega \left( \sum_{i=1}^S H \sqrt{\frac{HA}{\numagents} \cdot \frac{K}{S}} \right) = \Omega \left(\sqrt{\frac{H^3 S A K}{\numagents}} \right).
    \]
    
    Finally, we describe how to encode a hard MAB instance through the pair of state $(s_i,s^w_i)$.
    In the wait state $s^w_i$ the action $a_1$ transitions to $s_i$ with probability $1$, and the action $a_2$ stays in state $s^w_i$ if the step $h$ is at most $H+2$, otherwise it transitions to the bad state $s_b$.
    All the other actions $a_3,\dots,a_A$ always transition to the bad state.
    In state $s_i$ all the actions transition to the good state $s_g$ with probability $1/2$ and to the bad state $s_b$ with probability $1/2$, except for one action in a specific time step (both sampled uniformly at random) that transition to the good state with probability $1/2 + \epsilon$ (and to the bad state with probability $1/2 - \epsilon$) for some $\epsilon \approx \sqrt{S A / K}$ which is standard for MAB$\slash$RL lower bounds.
    
    Notice that this is in fact MAB with $H A$ actions since the learner needs to pick both the right action and the right horizon step.
    Moreover, the cost is either $0$ if the agents is successful in transitioning to the good state, or $\Theta (H)$ if the learner transitions to the bad state (in any case that the good state is not reached, the bad state will be reached before time step $H+2$).
\end{proof}

\begin{theorem}[Lower bound for adversarial MDP with non-fresh randomness and unknown transition]
    \label{thm:lower-bound-adversarial-non-fresh-unknown-p}
    Let $S, A, H, \numagents \in \bbN$ and $K \ge S A H$.
    For any algorithm $\texttt{ALG}$ there exists an adversarial MDP $\calM$ with non-fresh randomness such that: (i) $\calM$ has $\Theta (S)$ states, $\Theta (A)$ actions and horizon $\Theta (H)$; (ii) Running $\texttt{ALG}$ with $\numagents$ agents for $K$ episodes, when the transition function is unknown, suffers expected average regret of at least $\Omega (\sqrt{H^2 S K} + \sqrt{\frac{H^3 S A K}{\numagents}})$.
\end{theorem}

\begin{proof}
    Follows immediately from \cref{thm:lower-bound-stochastic-non-fresh} since adversarial MDPs generalize stochastic MDPs.
\end{proof}

\subsection{Auxiliary lemmas}

\begin{lemma}
    \label{lem:bin-rv-expected-square-root}
    Let $X \sim Bin(n,p)$ and assume that $n \ge 1/p^2$.
    Then, $\bbE [\sqrt{X}] \ge 0.01 \sqrt{np}$.
\end{lemma}

\begin{proof}
    By Markov inequality we have:
    \[
        \bbE[\sqrt{X}]
        \ge
        \frac{\sqrt{np}}{10} \Pr \left[ \sqrt{X} \ge \frac{\sqrt{np}}{10} \right]
        =
        \frac{\sqrt{np}}{10} \Pr \left[ X \ge \frac{np}{100} \right]
        =
        \frac{\sqrt{np}}{10} \left( 1 - \Pr \left[ X < \frac{np}{100} \right] \right).
    \]
    Thus, it suffices to show that $\Pr \left[ X < \frac{np}{100} \right] \le 9/10$ which follows immediately from Hoeffding inequality and the assumption that $n \ge 1/p^2$.
\end{proof}

\newpage

\section{The \texttt{coop-ULCVI} algorithm for stochastic MDPs with fresh randomness}
\label{appendix:fresh-stochastic}

\begin{algorithm}[t]
    \caption{\textsc{Cooperative Upper Lower Confidence Value Iteration (coop-ULCVI)}} 
    \label{alg:coop-ulcvi}
    \begin{algorithmic}[1]
        \STATE {\bf input:} state space $\calS$, action space $\calA$, horizon $H$, confidence parameter $\delta$, number of episodes $K$, number of agents $\numagents$.

        \STATE {\bf initialize:} $n^1_h(s,a)=0,n^1_h(s,a,s')=0,C^1_h(s,a)=0 \  \forall (s,a,s',h) \in \calS \times \calA \times \calS \times [H]$.

        \FOR{$k=1,\dots,K$}
                
               \STATE set $\hat p^{k}_h(s' \mid s,a) \gets \frac{n^{k}_h(s,a,s')}{n^{k}_h(s,a)\vee 1} , \hat c^{k}_h(s,a) \gets \frac{C^{k}_h(s,a)}{n^{k}_h(s,a)\vee 1} \ \forall (s,a,s',h) \in \calS \times \calA \times \calS \times [H]$.
            
            \STATE compute $\{ \pi^k_h(s) \}_{s,h}$ via \textsc{Optimistic-Pessimistic Value Iteration} (\cref{alg: optimistic pessimistic value iteration}).
    
            \STATE set $n^{k+1}_h(s,a) \gets n^{k}_h(s,a),n^{k+1}_h(s,a,s') \gets n^{k}_h(s,a,s'),C^{k+1}_h(s,a) \gets C^{k}_h(s,a) \ \forall (s,a,s',h) \in \calS \times \calA \times \calS \times [H]$.
    
            \FOR{$v=1,\dots,\numagents$}
            
                \STATE observe initial state $s^{k,v}_1$.
            
                \FOR{$h=1,\dots,H$}
        
                    \STATE pick action $a^{k,v}_h = \pi^k_h(s^{k,v}_h)$, suffer cost $C^{k,v}_h \sim c_h(s^{k,v}_h,a^{k,v}_h)$ and observe next state $s^{k,v}_{h+1} \sim p_h(\cdot \mid s^{k,v}_h,a^{k,v}_h)$.
                    
                    \STATE update $n^{k+1}_h(s^{k,v}_h,a^{k,v}_h) \gets n^{k+1}_h(s^{k,v}_h,a^{k,v}_h) + 1 , n^{k+1}_h(s^{k,v}_h,a^{k,v}_h,s^{k,v}_{h+1})\gets n^{k+1}_h(s^{k,v}_h,a^{k,v}_h,s^{k,v}_{h+1}) + 1$.
                    
                    \STATE update $C^{k+1}_h(s^{k,v}_h,a^{k,v}_h) \gets C^{k+1}_h(s^{k,v}_h,a^{k,v}_h) + C^{k,v}_h$.
        
                \ENDFOR
            \ENDFOR
        \ENDFOR
    \end{algorithmic}
\end{algorithm}

For the setting of stochastic MDPs with fresh randomness we propose the Cooperative Upper Lower Confidence Value Iteration algorithm (\verb|coop-ULCVI|; see \cref{alg:coop-ulcvi}).
The idea is simple: all the agents run the same optimistic policy, but the estimated costs and transition models are updated based on the trajectories of all of them.
Since the randomness is fresh in this setting, we expect the agents to observe $\numagents$ times more information.
Next, we prove the following optimal regret bound for \verb|coop-ULCVI|.

\begin{theorem}
    \label{thm:reg-coop-ulcvi}
    With probability $1 - \delta$, the individual regret of each agent of \verb|coop-ULCVI| is
    \[
        \regret
        =
        O \left( \sqrt{\frac{H^3 S A K}{\numagents}} \log \frac{\numagents K H S A}{\delta} + H^3 S^2 A \log^2 \frac{\numagents K H S A}{\delta} \right).
    \]
\end{theorem}

\begin{algorithm}[t]
    \caption{\textsc{Optimistic-Pessimistic Value Iteration}} 
    \label{alg: optimistic pessimistic value iteration}
    \begin{algorithmic}[1]
        \STATE {\bf input:} state space $\calS$, action space $\calA$, horizon $H$, confidence parameter $\delta$, number of episodes $K$, number of agents $\numagents$, visit counters $n^{k}$, empirical transition function $\hat p^{k}$, empirical cost function $\hat c^{k}$.

        \STATE {\bf initialize:} $\underl{V}^k_{H+1}(s)=\overl{V}^k_{H+1}(s)= 0$ for all $s \in \calS$.

        \FOR{$h=H,H-1,\ldots,1$}
            \FOR{$s \in \calS$}
                \FOR{$a \in \calA$}
                    \STATE set the bonus $b^k_h(s,a) = b^k_h(s,a ; c) + b^k_h(s,a ; p)$ defined as follows (for $\logterm = 3 \log \frac{6 S A H K \numagents}{\delta}$),
                    \begin{alignat*}{2}
                        &b^k_h(s,a ; c)
                        &&=
                        \sqrt{ \frac{2 \logterm }{n^{k}_h(s,a)\vee 1}}
                        \\
                        &b^k_h(s,a ; p) 
                        &&= \sqrt{\frac{2\VAR_{\hat p^{k}_h(\cdot \mid s,a)}(\underl{V}^k_{h+1}) \logterm}{n^{k}_h(s,a)\vee 1}} + \frac{44 H^2 S \logterm}{n^{k}_h(s,a)\vee 1} + \frac{1}{16 H}\bbE_{\hat p^{k}_h(\cdot \mid s,a)}\left[ \overl{V}^k_{h+1} - \underl{V}^k_{h+1} \right].
                    \end{alignat*}
                    
                    \STATE compute optimistic and pessimistic Q-functions:
                    \begin{alignat*}{2}
                        &\underl{Q}^k_h(s,a)
                        &&=
                        \hat c^{k}_h(s,a) - b^k_h(s,a) + \bbE_{\hat p^{k}_h(\cdot \mid s,a)}[\underl{V}^k_{h+1}]
                        \\ 
                        &\overl{Q}^k_h(s,a) 
                        &&=
                        \hat c^{k}_h(s,a) + b^k_h(s,a) +\bbE_{\hat p^{k}_h(\cdot \mid s,a)}[\overl{V}^k_{h+1}].
                    \end{alignat*}
                \ENDFOR
    
                \STATE set $\pi^k_h(s) \in \argmin_{a \in \calA} \underl{Q}^k_h(s,a)$.
        
                \STATE set $\underl{V}^k_h(s) = \max \{ \underl{Q}^k_h(s,\pi^k_h(s)) , 0 \}$,  $\overl{V}^k_h(s) = \min \{ \overl{Q}^k_h(s,\pi^k_h(s)) , H \}$.
            \ENDFOR
        \ENDFOR
    \end{algorithmic}
\end{algorithm}

\subsection{The good event, optimism and pessimism}

Define the following events (for $\logterm = 3 \log \frac{6 S A H K \numagents}{\delta}$): 
\begin{align*}
    E^c(k) 
    & = 
    \left\{ \forall (s,a,h):\ |\hat{c}^{k}_h(s,a) -c_h(s,a)|  \leq \sqrt{ \frac{2 \logterm }{n^{k}_h(s,a)\vee 1}} \right\}
    \\
    E^p(k) 
    & = 
    \left\{ \forall (s,a,s',h):\ |p_h (s'|s,a) - \hat{p}^{k}_h (s'|s,a)| \le \sqrt{\frac{2 p_h(s'|s,a)\logterm}{n^{k}_{h}(s,a)\vee 1}} + \frac{2 \logterm}{n^{k}_h(s,a) \vee 1} \right\} 
    \\
    E^{pv1}(k)
    & =
    \left\{ \forall (s,a,h):\ | \left((\hat{p}^{k}_h(\cdot | s,a)-p_h(\cdot | s,a) \right) \cdot V_{h+1}^\star | \leq \sqrt{\frac{2\VAR_{p_h(\cdot \mid  s,a)}(V^\star_{h+1}) \logterm }{n^{k}_{h}(s,a)\vee 1}} + \frac{5 H \logterm }{n^{k}_h(s,a)\vee 1} \right\}
    \\
    E^{pv2}(k)
    & =
    \left\{ \forall (s,a,h):\ | \sqrt{\VAR_{p_h(\cdot \mid  s,a )}(V_{h+1}^\star)} -  \sqrt{\VAR_{\hat{p}_h^{k}(\cdot \mid  s,a )}(V_{h+1}^\star)} | \leq \sqrt{\frac{12 H^2 \logterm}{n^{k}_h(s,a)\vee 1}} \right\}
\end{align*}

The basic good event, which is the intersection of the above events, is the one used in \citet{efroni2021confidence}. 
The following lemma establishes that the good event holds with high probability. 
The proof is supplied in \citet[Lemma 13]{efroni2021confidence} by applying standard concentration results.

\begin{lemma}[The First Good Event]
    \label{lemma: the first good event RL UL}
    Let $\bbG_1 =\cap_{k = 1}^K E^c(k) \cap_{k = 1}^K E^p(k) \cap_{k = 1}^K E^{pv1}(k) \cap_{k = 1}^K E^{pv2}(k)$ be the basic good event. 
    It holds that $\Pr(\bbG_1)\geq 1-\delta/2$.
\end{lemma}

Under the first good event, we can prove that the value is optimistic using standard techniques (similar to \citet[Lemma 14]{efroni2021confidence}).

\begin{lemma}[Upper Value Function is Pessimistic, Lower Value Function is Optimistic] 
    \label{lemma: optimism ucbvi-UL}
    Conditioned on the first good event $\bbG_1$, it holds that $\underl{V}^k_h(s) \leq  V^\star_h(s) \leq V^{\pi^k}_h(s) \leq \overl{V}^k_h(s)$ for every $k=1,\dots,K$, $s \in \calS$ and $h=1,\dots,H$.
\end{lemma}

Finally, using similar techniques to \citet[Lemma 21]{efroni2021confidence}, we can prove an additional high probability event which hold alongside the basic good event $\bbG_1$.
To that end, we define the filtration $\{ \filt^k \}_{k \ge 1}$ as the $\sigma$-algebra that contains the information on all observed data until the beginning of episode $k$ (including the initial state of episode $k$).
In addition, we define the filtration $\{ \filt^k_h \}_{k \ge 1,h \ge 1}$ as the $\sigma$-algebra that contains the information on all observed data until step $h$ of episode $k$ (including the $h$-th state of episode $k$).

\begin{lemma}[The Good Event]
    \label{lem:ULCVI-good-event}
    Let $\bbG_1$ be the event defined in \cref{lemma: the first good event RL UL}.
    The second good event is the intersection of two events $\bbG_2 = E^{OP} \cap  E^{\VAR}$ defined as follows:
    \begin{align*}
        & E^{OP}
        =
        \left\{ \forall h \in [H], v \in [\numagents]:\ \sum_{k=1}^K \bbE[\overl{V}^k_h(s^{k,v}_h) - \underl{V}^k_h(s^{k,v}_h) \mid \filt^k_h] \leq 18 H^2 \logterm + \left( 1+\frac{1}{2H} \right) \sum_{k=1}^K \overl{V}^k_h(s^{k,v}_h) - \underl{V}^k_h(s^{k,v}_h) \right\}
        \\
        & E^{\VAR}
        = 
        \left\{ \forall v \in [\numagents]:\  \sum_{k=1}^K \sum_{h=1}^H \VAR_{p_h(\cdot \mid s^{k,v}_h,a^{k,v}_h)}( V^{\pi^k}_{h+1}) \leq  4 H^3 \logterm + 2 \sum_{k=1}^K \sum_{h=1}^H \bbE [\VAR_{p_h(\cdot \mid s^{k,v}_h,a^{k,v}_h)}( V^{\pi^k}_{h+1}) \mid \filt^k ] \right\}.
    \end{align*}
    Then, the good event $\bbG = \bbG_1 \cap \bbG_2$ holds with probability at least $1-\delta$.
\end{lemma}

\subsection{Proof of Theorem~\ref{thm:reg-coop-ulcvi}}

\begin{lemma}[Key Recursion Bound]
    \label{lemma: key recursion bound UL}
    Conditioning on the good event $\bbG$, the following bound holds for all $h\in [H]$ and $v \in [\numagents]$. 
    \begin{align*}
        \sum_{k=1}^K \overl{V}^k_h(s^{k,v}_h)-\underl{V}^k_h(s^{k,v}_h)
        & \leq  
        18 H^2 \logterm +\sum_{k=1}^K \frac{226 H^2 S \logterm}{ n^{k}_h(s^{k,v}_h,a^{k,v}_h)\vee 1} 
        + \sum_{k=1}^K \frac{2 \sqrt{2 \logterm}}{\sqrt{n^{k}_h(s^{k,v}_h,a^{k,v}_h)\vee 1}} 
        \\
        & \quad + 
        \sum_{k=1}^K 2\sqrt{2 \logterm}\frac{\sqrt{\VAR_{p_h(\cdot \mid s^{k,v}_h,a^{k,v}_h)}(V^{\pi^k}_{h+1})}}{\sqrt{n^{k}_h(s^{k,v}_h,a^{k,v}_h)\vee 1}} + \left( 1+ \frac{1}{2H} \right)^2 \sum_{k=1}^K \overl{V}^k_{h+1}(s^{k,v}_{h+1}) -\underl{V}^k_{h+1}(s^{k,v}_{h+1}).
    \end{align*}
\end{lemma}

\begin{proof}
    We bound each of the terms in the sum as follows:
    \begin{align}
        \nonumber
        \overl{V}^k_h(s^{k,v}_h)-\underl{V}^k_h(s^{k,v}_h)
        & \le   
        2b^k_h(s^{k,v}_h,a^{k,v}_h ; c) + 2 b^k_h(s^{k,v}_h,a^{k,v}_h ; p) + \bbE_{\hat{p}^{k}_h(\cdot \mid  s^{k,v}_h,a^{k,v}_h)}[  \overl{V}^k_{h+1} - \underl{V}^k_{h+1}]
        \\
        \nonumber
        & = 
        2b^k_h(s^{k,v}_h,a^{k,v}_h ; c) + 2 b^k_h(s^{k,v}_h,a^{k,v}_h ; p)
        \\
        \nonumber
        & \qquad + 
        \bbE_{p_h(\cdot \mid  s^{k,v}_h,a^{k,v}_h)} [ \overl{V}^k_{h+1} - \underl{V}^k_{h+1}] + (\hat{p}^{k}_h-p_h)(\cdot |s^{k,v}_h,a^{k,v}_h) \cdot ( \overl{V}^k_{h+1} - \underl{V}^k_{h+1})
        \\
        \nonumber
        & \leq  
       2b^k_h(s^{k,v}_h,a^{k,v}_h ; c) + 2 b^k_h(s^{k,v}_h,a^{k,v}_h ; p)
       \\
       \label{eq: central theorem UL RL relation 1}
       & \qquad +
       \frac{8H^2 S \logterm}{n^{k}_h(s^{k,v}_h,a^{k,v}_h)\vee 1}+ \left( 1+\frac{1}{4H} \right) \bbE_{p_h(\cdot \mid  s^{k,v}_h,a^{k,v}_h)} [ \overl{V}^k_{h+1} - \underl{V}^k_{h+1}],
    \end{align}
    where the last relation holds by \citet[Lemma B.13]{cohen2021minimax} which upper bounds
    \[
        (\hat{p}^{k}_h-p_h)(\cdot |s^{k,v}_h,a^{k,v}_h) \cdot \left( \overl{V}^k_{h+1} - \underl{V}^k_{h+1} \right) 
        \leq 
        \frac{8H^2 S \logterm}{n^{k}_h(s^{k,v}_h,a^{k,v}_h)\vee 1}+ \frac{1}{4H} \bbE_{p_h(\cdot \mid  s^{k,v}_h,a^{k,v}_h)} [ \overl{V}^k_{h+1} - \underl{V}^k_{h+1}]
    \]
    by setting $\alpha=4H,C_1=C_2=2$ and bounding $H \logterm ( 2C_2+ \alpha S C_1/2)\leq 8H^2 S \logterm$ (the assumption of the lemma holds since the event $\cap_k E^p(k)$ holds). 
    Taking the sum over $k \in [K]$ we get that
    \begin{align}
        \nonumber
        \sum_{k=1}^K \overl{V}^k_h(s^{k,v}_h)-\underl{V}^k_h(s^{k,v}_h)
        & \leq 
        \sum_{k=1}^K 2b^k_h(s^{k,v}_h,a^{k,v}_h ; c) + \sum_{k=1}^K 2 b^k_h(s^{k,v}_h,a^{k,v}_h ; p) + \sum_{k=1}^K \frac{8H^2 S \logterm}{n^{k}_h(s^{k,v}_h,a^{k,v}_h)\vee 1}
        \\
        \label{eq: central theorem UL RL relation 12}
        & \qquad + 
        \left( 1+\frac{1}{4H} \right) \sum_{k=1}^K \bbE_{p_h(\cdot \mid  s^{k,v}_h,a^{k,v}_h)} [ \overl{V}^k_{h+1} - \underl{V}^k_{h+1}].
    \end{align}
    The first sum is bounded by definition by
    \[
        \sum_{k=1}^K b^k_h(s^{k,v}_h,a^{k,v}_h ; c)
        \leq 
        \sum_{k=1}^K \sqrt{ \frac{2 \logterm }{n^{k}_h(s^{k,v}_h,a^{k,v}_h)\vee 1}},
    \]
    and the second sum is bounded in \citet[Lemma 24]{efroni2021confidence} by
    \begin{align*}
        \sum_{k=1}^K b^k_h(s^{k,v}_h,a^{k,v}_h)
        & \leq 
        \sum_{k=1}^K \frac{109 H^2 S \logterm}{ n^{k}_h(s^{k,v}_h,a^{k,v}_h)\vee 1} + \sqrt{2\logterm} \sum_{k=1}^K \frac{\sqrt{\VAR_{p_h(\cdot \mid s^{k,v}_h,a^{k,v}_h)}(V^{\pi^k}_{h+1})}}{\sqrt{n^{k}_h(s^{k,v}_h,a^{k,v}_h)\vee 1}}
        \\
        & \qquad + 
        \frac{1}{8H}\sum_{k=1}^K \bbE_{p_h(\cdot \mid  s^{k,v}_h,a^{k,v}_h)} [ \overl{V}^k_{h+1} - \underl{V}^k_{h+1}].
    \end{align*}
    Plugging this into~\eqref{eq: central theorem UL RL relation 12} and rearranging the terms we get
    \begin{align*}
        \sum_{k=1}^K \overl{V}^k_h(s^{k,v}_h)-\underl{V}^k_h(s^{k,v}_h)
        & \leq 
        \sum_{k=1}^K \frac{2 \sqrt{2 \logterm}}{\sqrt{n^{k}_h(s^{k,v}_h,a^{k,v}_h)\vee 1}} + 2\sqrt{2 \logterm} \sum_{k=1}^K \frac{\sqrt{\VAR_{p_h(\cdot \mid s^{k,v}_h,a^{k,v}_h)}(V^{\pi^k}_{h+1})}}{\sqrt{n^{k}_h(s^{k,v}_h,a^{k,v}_h)\vee 1}}
        \\
        & \qquad + 
        \sum_{k=1}^k \frac{226 H^2 S \logterm}{ n^{k}_h(s^{k,v}_h,a^{k,v}_h)\vee 1} + \left( 1+ \frac{1}{2H} \right) \sum_{k=1}^K \bbE_{p_h(\cdot \mid  s^{k,v}_h,a^{k,v}_h)} [ \overl{V}^k_{h+1} - \underl{V}^k_{h+1}]
        \\
        &\leq  
        18H^2 \logterm + \sum_{k=1}^K \frac{2 \sqrt{2 \logterm}}{\sqrt{n^{k}_h(s^{k,v}_h,a^{k,v}_h)\vee 1}} + \sum_{k=1}^k \frac{226 H^2 S \logterm}{ n^{k}_h(s^{k,v}_h,a^{k,v}_h)\vee 1}
        \\
        & \qquad 
        + \sum_{k=1}^K 2\sqrt{2 \logterm}\frac{\sqrt{\VAR_{p_h(\cdot \mid s^{k,v}_h,a^{k,v}_h)}(V^{\pi^k}_{h+1})}}{\sqrt{n^{k}_h(s^{k,v}_h,a^{k,v}_h)\vee 1}} +
        \left( 1+ \frac{1}{2H} \right)^2\sum_{k=1}^K \overl{V}^k_{h+1}(s^{k,v}_{h+1}) -  \underl{V}^k_{h+1}(s^{k,v}_{h+1}),
    \end{align*}
    where the last inequality follows since the second good event holds.
\end{proof}

\begin{proof}[Proof of \cref{thm:reg-coop-ulcvi}]
    Start by conditioning on the good event which holds with probability greater than $1-\delta$. 
    Applying the optimism-pessimism of the upper and lower value function we get
    \begin{align}
        \label{eq: central thm UL RL 1 relation}
        \sum_{k=1}^K V_{1}^{\pi^k}(s^{k,v}_1) - V_{1}^\star(s^{k,v}_1)   
        \leq 
        \frac{1}{\numagents} \sum_{v=1}^\numagents \sum_{k=1}^K \overl{V}^k_1(s^{k,v}_1) - \underl{V}^k_1(s^{k,v}_1).
    \end{align}
    Iteratively applying \cref{lemma: key recursion bound UL} and bounding the exponential growth by $(1+\frac{1}{2H})^{2H}\leq e\leq 3$, the following upper bound on the cumulative regret is obtained.
    \begin{align}
        \nonumber
        \eqref{eq: central thm UL RL 1 relation} 
        & \leq 
        54 H^2 \logterm + \frac{1}{\numagents} \sum_{v=1}^\numagents \sum_{k=1}^K \sum_{h=1}^H \frac{ 678 H^2 S \logterm}{n^{k}_h( s^{k,v}_h,a^{k,v}_h)\vee 1} 
        \\
        \label{eq: rl final bound relation 2 UL}
        & \quad + 
        \frac{1}{\numagents} \sum_{v=1}^\numagents \sum_{k=1}^K \sum_{h=1}^H \frac{6 \sqrt{2 \logterm}}{\sqrt{n^{k}_h(s^{k,v}_h,a^{k,v}_h)\vee 1}} + \frac{1}{\numagents} \sum_{v=1}^\numagents \sum_{k=1}^K \sum_{h=1}^H \frac{6 \sqrt{2 \logterm \VAR_{p_h(\cdot \mid s^{k,v}_h,a^{k,v}_h)}(V^{\pi^k}_{h+1}) }}{\sqrt{n^{k}_h(s^{k,v}_h,a^{k,v}_h)}}.
    \end{align}
    
    We now bound each of the three sums in \cref{eq: rl final bound relation 2 UL}. 
    We bound the first sum in \cref{eq: rl final bound relation 2 UL} via standard analysis as follows:
    \begin{align}
        \nonumber
        \sum_{v=1}^\numagents \sum_{k=1}^K \sum_{h=1}^H \frac{1}{n^{k}_h( s^{k,v}_h,a^{k,v}_h)\vee 1} 
        & =
        \sum_{h=1}^H \sum_{s \in \calS} \sum_{a \in \calA} \sum_{k=1}^K \frac{\sum_{v=1}^\numagents \indevent{s^{k,v}_h=s,a^{k,v}_h=a}}{n^{k}_h(s,a)\vee 1}
        \\
        \nonumber
        & = 
        \sum_{h=1}^H \sum_{s \in \calS} \sum_{a \in \calA} \sum_{k=1}^K \indevent{n^{k}_h(s,a) \ge \numagents} \frac{\sum_{v=1}^\numagents \indevent{s^{k,v}_h=s,a^{k,v}_h=a}}{n^{k}_h(s,a)\vee 1}
        \\
        \nonumber
        & \qquad + 
        \sum_{h=1}^H \sum_{s \in \calS} \sum_{a \in \calA} \sum_{k=1}^K \indevent{n^{k}_h(s,a) < \numagents} \frac{\sum_{v=1}^\numagents \indevent{s^{k,v}_h=s,a^{k,v}_h=a}}{n^{k}_h(s,a)\vee 1}
        \\
        \nonumber
        & \le
        2 H S A \logterm +  \sum_{h=1}^H \sum_{s \in \calS} \sum_{a \in \calA} \sum_{k=1}^K \indevent{n^{k}_h(s,a) < \numagents} \frac{\sum_{v=1}^\numagents \indevent{s^{k,v}_h=s,a^{k,v}_h=a}}{n^{k}_h(s,a)\vee 1}
        \\
        \label{eq:bound-harmonic-sum}
        & \le
        2 H S A \logterm + 2 \numagents H S A
        \le
        4 \numagents H S A \logterm,
    \end{align}
    where the first inequality is by \cref{lem:sum-1/n}.
    
    The second sum in~\cref{eq: rl final bound relation 2 UL} is bounded as follows,
    \begin{align*}
        \sum_{v=1}^\numagents \sum_{k=1}^K \sum_{h=1}^H \frac{1}{\sqrt{n^{k}_h(s^{k,v}_h,a^{k,v}_h)\vee 1}}
        & \le
        \sum_{v=1}^\numagents \sum_{k=1}^K \sum_{h=1}^H \frac{\indevent{n^{k}_h(s^{k,v}_h,a^{k,v}_h) \ge \numagents}}{\sqrt{n^{k}_h(s^{k,v}_h,a^{k,v}_h)\vee 1}} + 2 \numagents H S A
        \\
        & \le
        \sqrt{\sum_{v=1}^\numagents \sum_{k=1}^K \sum_{h=1}^H 1} \sqrt{\sum_{v=1}^\numagents \sum_{k=1}^K \sum_{h=1}^H \frac{\indevent{n^{k}_h(s^{k,v}_h,a^{k,v}_h) \ge \numagents}}{n^{k}_h( s^{k,v}_h,a^{k,v}_h)\vee 1}} + 2 \numagents H S A
        \\
        & \le
        \sqrt{K H \numagents} \sqrt{2 H S A \logterm} + 24 H S A \logterm
        =
        \sqrt{2 \numagents H^2 S A K \logterm} + 2 \numagents H S A,
    \end{align*}
    where the first inequality is similar to \cref{eq:bound-harmonic-sum}, the second is by Cauchy–Schwarz, and the third is by \cref{lem:sum-1/n}.
    
    The third sum in \cref{eq: rl final bound relation 2 UL} is bounded by applying the Cauchy–Schwarz inequality as follows,
    \begin{align*}
        \sum_{v=1}^\numagents \sum_{k=1}^K \sum_{h=1}^H & \frac{\sqrt{\VAR_{p_h(\cdot \mid s^{k,v}_h,a^{k,v}_h)}(V^{\pi^k}_{h+1}) }}{\sqrt{n^{k}_h(s^{k,v}_h,a^{k,v}_h)}}
        \le
        \sum_{v=1}^\numagents \sum_{k=1}^K \sum_{h=1}^H \indevent{n^{k}_h(s^{k,v}_h,a^{k,v}_h) \ge \numagents} \frac{\sqrt{ \VAR_{p_h(\cdot \mid s^{k,v}_h,a^{k,v}_h)}(V^{\pi^k}_{h+1}) }}{\sqrt{n^{k}_h(s^{k,v}_h,a^{k,v}_h)}} + 2 \numagents H^2 S A
        \\
        & \le
        \sqrt{\sum_{v=1}^\numagents \sum_{k=1}^K \sum_{h=1}^H \VAR_{p_h(\cdot \mid s^{k,v}_h,a^{k,v}_h)}(V^{\pi^k}_{h+1}) } \sqrt{\sum_{v=1}^\numagents \sum_{k=1}^K \sum_{h=1}^H \frac{\indevent{n^{k}_h(s^{k,v}_h,a^{k,v}_h) \ge \numagents}}{n^{k}_h(s^{k,v}_h,a^{k,v}_h)} } + 2 \numagents H^2 S A
        \\
        & \le
        \sqrt{\sum_{v=1}^\numagents \sum_{k=1}^K \sum_{h=1}^H \VAR_{p_h(\cdot \mid s^{k,v}_h,a^{k,v}_h)}(V^{\pi^k}_{h+1}) } \sqrt{2 H S A \logterm} + 2 \numagents H^2 S A
        \\
        & \le
        \sqrt{2 H S A \logterm} \sqrt{\sum_{v=1}^\numagents \sum_{k=1}^K \sum_{h=1}^H \bbE \left[ \VAR_{p_h(\cdot \mid s^{k,v}_h,a^{k,v}_h)}(V^{\pi^k}_{h+1}) \mid \filt^k \right] + 4 H^3 \logterm } + 2 \numagents H^2 S A
        \\
        & \le
       \sqrt{2 H S A \logterm} \sqrt{\sum_{v=1}^\numagents \sum_{k=1}^K \sum_{h=1}^H \bbE \left[ \VAR_{p_h(\cdot \mid s^{k,v}_h,a^{k,v}_h)}(V^{\pi^k}_{h+1}) \mid \filt^k \right] } + 5 \numagents H^2 S A \logterm
        \\
        & \overset{(*)}{=}
        \sqrt{2 H S A \logterm} \sqrt{\sum_{v=1}^\numagents \sum_{k=1}^K \bbE \left[ \left( V^{\pi^k}_1(s^{k,v}_1) - \sum_{h=1}^H c_h(s^{k,v}_h,a^{k,v}_h) \right)^2  \mid \filt^k  \right] } + 5 \numagents H^2 S A \logterm
        \\
        & \le
        \sqrt{2 H S A \logterm} \sqrt{\sum_{v=1}^\numagents \sum_{k=1}^K H^2} + 5 \numagents H^2 S A \logterm
       \le
        \sqrt{2  \numagents H^3 S A K \logterm} + 5 \numagents H^2 S A \logterm.
    \end{align*}
    where the first inequality is similar to \cref{eq:bound-harmonic-sum}, the third inequality is by \cref{lem:sum-1/n}, the forth is by event $E^{\VAR}$, and $(*)$ is by the law of total variance \citep[Lemma B.14]{cohen2021minimax}.
\end{proof}

\subsection{Auxiliary lemmas}

\begin{lemma}
    \label{lem:sum-1/n}
    It holds that
    \[
        \sum_{h=1}^H \sum_{s \in \calS} \sum_{a \in \calA} \sum_{k=1}^K \indevent{n^{k}_h(s,a) \ge \numagents} \frac{\sum_{v=1}^\numagents \indevent{s^{k,v}_h=s,a^{k,v}_h=a}}{n^{k}_h(s,a)\vee 1}
        \le
        2 H S A \log (K \numagents).
    \]
\end{lemma}

\begin{proof}
    By \citet[Lemma B.18]{rosenberg2020near}, we have that
    \begin{align*}
        \sum_{h=1}^H \sum_{s \in \calS} \sum_{a \in \calA} \sum_{k=1}^K \indevent{n^{k}_h(s,a) \ge \numagents} \frac{\sum_{v=1}^\numagents \indevent{s^{k,v}_h=s,a^{k,v}_h=a}}{n^{k}_h(s,a)\vee 1}
        & \le
        \sum_{h=1}^H \sum_{s \in \calS} \sum_{a \in \calA} 2 \log \left( \sum_{k=1}^K \sum_{v=1}^\numagents \indevent{s^{k,v}_h=s,a^{k,v}_h=a} \right)
        \\
        & \le
        2 H S A \log (K \numagents).
        \qedhere
    \end{align*}
\end{proof}

\newpage

\section{The \texttt{coop-ULCAE} algorithm for stochastic MDPs with non-fresh randomness}
\label{appendix:non-fresh stochastic}

\begin{algorithm}[t]
    \caption{\textsc{Cooperative Upper Lower Confidence action elimination (coop-ULCAE)}} 
    \label{alg:coop-ulcae}
    \begin{algorithmic}[1]
        \STATE {\bf input:} state space $\calS$, action space $\calA$, horizon $H$, confidence parameter $\delta$, number of episodes $K$, number of agents $\numagents$, exploration parameter $\epsilon > 0$.

        \STATE {\bf initialize:} $n^1_h(s,a)=0,n^1_h(s,a,s')=0,C^1_h(s,a)=0,\calA^0_h(s) = \calA \  \forall (s,a,s',h) \in \calS \times \calA \times \calS \times [H]$.

        \FOR{$k=1,\dots,K$}
                
            \STATE set $\hat p^{k}_h(s' \mid s,a) \gets \frac{n^{k}_h(s,a,s')}{n^{k}_h(s,a)\vee 1} , \hat c^{k}_h(s,a) \gets \frac{C^{k}_h(s,a)}{n^{k}_h(s,a)\vee 1} \ \forall (s,a,s',h) \in \calS \times \calA \times \calS \times [H]$.
            
            \STATE compute $\{ \underl{\pi}^k_h(s) \}_{s,h}$ via \textsc{Optimistic-Pessimistic Value Iteration} (\cref{alg: optimistic pessimistic value iteration}).
            
            \STATE set $\calA^k_h(s) \gets \calA^{k-1}_h(s)$ for every $s,h$.
            
            \STATE remove sub-optimal actions for every $s,h$: if $\exists a,a' \in \calA^k_h(s)$ s.t. $\underl{Q}_{h}^{k}(s,a)>\overl{Q}_{h}^{k}(s,a')$, then $\calA^k_h(s) \gets \calA^k_h(s)\backslash\{a\}$.
            
            \STATE set $I^{k}_h(s,a,s')=0,I^{k}_h(s,a)=0,IC^{k}_h(s,a)=0 \  \forall (s,a,s',h) \in \calS \times \calA \times \calS \times [H]$.
    
            \FOR{$v=1,\dots,\numagents$}
            
                \STATE sample $h_v \in [H]$ uniformly at random.
                
                \STATE set $
                    \pi^{k,v}
                    =
                    \begin{cases}
                        \underl{\pi}^{k}  & \text{with probability } 1 - \epsilon
                        \\
                        \pi^{k,h_v}        & \text{with probability } \epsilon
                \end{cases}
                $, where:
                $
                    \pi^{k,h'}_h(a \mid s)
                    =
                    \begin{cases}
                        \underl{\pi}_{h}^{k}(a \mid s)  & h \ne h'
                        \\
                        \frac{1}{\calA^k_h(s)}        & h = h'.
                \end{cases}
                $
            
                \STATE observe initial state $s^{k,v}_1$.
            
                \FOR{$h=1,\dots,H$}
        
                    \STATE pick action $a^{k,v}_h \sim \pi^{k,v}_h(\cdot \mid s^{k,v}_h)$, suffer cost $C^{k,v}_h$ and observe next state $s^{k,v}_{h+1}$.
                    
                    \STATE update $I^{k}_h(s^{k,v}_h,a^{k,v}_h)\gets 1 , I^{k}_h(s^{k,v}_h,a^{k,v}_h,s^{k,v}_{h+1})\gets 1 , IC^{k}_h(s^{k,v}_h,a^{k,v}_h) \gets C^{k,v}_h$.
        
                \ENDFOR
            \ENDFOR
            
            \STATE set $n^{k+1}_h(s,a) \gets n^{k}_h(s,a) + I^{k}_h(s,a),n^{k+1}_h(s,a,s') \gets n^{k}_h(s,a,s') + I^{k}_h(s,a,s') \ \forall (s,a,s',h)$.
            
            \STATE set $C^{k+1}_h(s,a) \gets C^{k}_h(s,a) + IC^{k}_h(s,a) \ \forall (s,a,h)$.
        \ENDFOR
    \end{algorithmic}
\end{algorithm}

For the setting of stochastic MDPs with non-fresh randomness we propose the Cooperative Upper Lower Confidence Action Elimination algorithm (\verb|coop-ULCAE|; see \cref{alg:coop-ulcae}).
Recall that if all the agents play the optimistic policy (like \verb|coop-ULCVI|), the regret will not improve since the randomness is non-fresh.
Thus, we want each agent to diverge from the trajectory of the optimistic policy at some point.
To that end, at some step each agent takes a random action.
At the other steps it follows the optimistic policy to make sure that its regret does not increase.
Finally, since all actions have probability to be explored, we eliminate sub-optimal actions to avoid unnecessary over exploration.

\begin{theorem}
    \label{thm:reg-coop-ulcae}
    With probability $1 - \delta$, setting $\epsilon = \min \{ \frac{H A}{\numagents} , \frac{1}{\sqrt{\numagents}} \}$, the individual regret of each agent of \verb|coop-ULCAE| is
    \begin{align*}
    \regret = 
        O\Biggl(
        	\sqrt{H^{5}SK} \log\frac{\numagents HSAK}{\delta} 
        	 + & \sqrt{\frac{H^{7}SAK}{\sqrt{m}}} \log\frac{\numagents HSAK}{\delta} 
        	+ \sqrt{\frac{H^{8}SAK}{\numagents}}\log\frac{\numagents HSAK}{\delta} 
        	\\
        	& + H^{5}S^{2}A \log^2\frac{\numagents HSAK}{\delta} 
        	+ \frac{H^{6}S^{2}A^{2}}{\sqrt{\numagents}} \log^2 \frac{\numagents HSAK}{\delta}
    	\Biggr).
    \end{align*}
\end{theorem}

\subsection{The good event, optimism and pessimism}

Define the following events (for $\logterm = 3 \log \frac{6 S A H K \numagents}{\delta}$): 
\begin{align*}
    E^c(k) 
    & = 
    \left\{ \forall (s,a,h):\ |\hat{c}^{k}_h(s,a) -c_h(s,a)|  \leq \sqrt{ \frac{2 \logterm }{n^{k}_h(s,a)\vee 1}} \right\}
    \\
    E^p(k) 
    & = 
    \left\{ \forall (s,a,s',h):\ |p_h (s'|s,a) - \hat{p}^{k}_h (s'|s,a)| \le \sqrt{\frac{2 p_h(s'|s,a)\logterm}{n^{k}_h(s,a)\vee 1}} + \frac{2 \logterm}{n^{k}_h(s,a) \vee 1} \right\} 
    \\
    E^{pv1}(k)
    & =
    \left\{ \forall (s,a,h):\ | \left((\hat{p}^{k}_h(\cdot | s,a)-p_h(\cdot | s,a) \right) \cdot V_{h+1}^\star | \leq \sqrt{\frac{2\VAR_{p_h(\cdot \mid  s,a)}(V^\star_{h+1}) \logterm }{n^{k}_h(s,a)\vee 1}} + \frac{5 H \logterm }{n^{k}_h(s,a)\vee 1} \right\}
    \\
    E^{pv2}(k)
    & =
    \left\{ \forall (s,a,h):\ | \sqrt{\VAR_{p_h(\cdot \mid  s,a )}(V_{h+1}^\star)} -  \sqrt{\VAR_{\hat{p}_h^{k}(\cdot \mid  s,a )}(V_{h+1}^\star)} | \leq \sqrt{\frac{12 H^2 \logterm}{n^{k}_h(s,a)\vee 1}} \right\}
\end{align*}

The basic good event, which is the intersection of the above events, is the one used in \citet{efroni2021confidence}. 
The following lemma establishes that the good event holds with high probability. 
The proof is supplied in \citet[Lemma 13]{efroni2021confidence} by applying standard concentration results.

\begin{lemma}[The First Good Event]
    \label{lemma: the first good event RL UL-non-fresh}
    Let $\bbG_1 =\cap_{k = 1}^K E^c(k) \cap_{k = 1}^K E^p(k) \cap_{k = 1}^K E^{pv1}(k) \cap_{k = 1}^K E^{pv2}(k)$ be the basic good event. 
    It holds that $\Pr(\bbG_1)\geq 1-\delta/2$.
\end{lemma}

Under the first good event, we can prove that the value is optimistic using standard techniques (similar to \citet[Lemma 14]{efroni2021confidence}).

\begin{lemma}[Upper Value Function is Pessimistic, Lower Value Function is Optimistic] 
    \label{lemma: optimism ucbvi-non-fresh}
    Conditioned on the first good event $\bbG_1$, it holds that $\underl{V}^k_h(s) \leq  V^\star_h(s) \leq V^{\underl{\pi}^k}_h(s) \leq \overl{V}^k_h(s)$ and that $\underl{Q}^k_h(s,a) \leq  Q^\star_h(s,a) \leq Q^{\underl{\pi}^k}_h(s,a) \leq \overl{Q}^k_h(s,a)$ for every $k=1,\dots,K$, $s \in \calS$, $a \in \calA$ and $h=1,\dots,H$.
    Moreover, $\pi^\star_h(s) \in \calA^k_h(s)$ for every $k=1,\dots,K$, $s \in \calS$ and $h=1,\dots,H$.
\end{lemma}

Finally, we define the following events that are more specific to our algorithmic action elimination framework:
\begin{align*}
    E^{n1} &= 
     \left\{\forall (k,h,s)\in[K]\times[H]\times\mathcal{S} \ \forall a\in \mathcal{A}_h^k(s):
            n_{h}^{k}(s,a) \ge \frac{m\epsilon}{4HA}\sum_{j=1}^{k-1}q_{h}^{\underl{\pi}^{j}}(s) - \log\frac{6HSA}{\delta} 
        \right\}
    \\
    E^{n2} &= 
     \left\{ \forall (k,h,s,a,v)\in[K]\times[H]\times\mathcal{S}\times\mathcal{A}\times [m]:
            n_{h}^{k}(s,a) \ge \frac{1}{2}\sum_{j=1}^{k-1}q_{h}^{\pi^{j,v}}(s,a) - \log\frac{6mHSA}{\delta} 
        \right\}
    \\
    E^{\epsilon} & = 
     \left\{ \forall (h',v) \in  [H] \times [\numagents]:
            \sum_{k=1}^{K}\mathbb{I} \{\pi^{k,v}=\pi^{k,h'}\} 
            \le
            \frac{\epsilon}{H}K + \sqrt{K\log\frac{6\numagents H}{\delta}}
        \right\}
\end{align*}

\begin{lemma}[The Second Good Event]
    \label{lemma:good-event-AE}
    Let $\bbG_2 =E^{n1} \cap E^{n2}\cap E^{\epsilon}$ be the second good event. 
    It holds that $\Pr(\bbG_2)\geq 1-\delta/2$.
\end{lemma}

As a direct consequence, we get that the good event $\bbG$ which is the intersection of $\bbG_1$ and $\bbG_2$ holds with probability $1 - \delta$.

\begin{lemma}[The Good Event]
    \label{lem:good-event-stochastic-non-fresh}
    Let $\bbG_1$ be the first good event defined in \cref{lemma: the first good event RL UL-non-fresh}, and $\bbG_2$ be the second good event defined in \cref{lemma:good-event-AE}.
    Then, the good event $\bbG = \bbG_1 \cap \bbG_2$ holds with probability $1 - \delta$.
\end{lemma}

\begin{proof}[Proof of \cref{lemma:good-event-AE}]
    We show that each of the events $\neg E^{n1}, \neg E^{n2}, \neg E^{\epsilon}$ occur with probability at most $\delta / 6$. Then, by a union bound we obtain the statement.
    
    \textbf{$\Pr[\neg E^{n1}] \le \delta/6$:}
    Without loss of generality, assume that in each episode, each agent uniformly randomizes a permutation over all actions, $\sigma^{k,v}$, and in case of exploration takes the first active arm in the permutation $\sigma^{k,v}$: $\arg\min_{a\in\mathcal{A}_{h}^{k}(s)}\sigma^{k,v}(a)$. 
    For any $a\in \mathcal{A}_h^k(s)$,
    \begin{align}
        \nonumber
        n_{h}^{k}(s,a) 
        & =
        \sum_{j=1}^{k-1} \mathbb{I}\{\exists v : s_{h}^{j,v}=s,a_{h}^{j,v}=a\}
        \\
        \nonumber
        & \geq 
        \sum_{j=1}^{k-1} \mathbb{I}\{\exists v : s_{h}^{j,v}=s,a_{h}^{j,v}=a,h_{v}=h,\pi^{j,v}=\pi^{j,h_{v}},\sigma^{j,v}(a) = 1\}
        \\
        \nonumber
        & \underset{(*)}{\geq}
        \sum_{j=1}^{k-1} \mathbb{I}\{\exists v : s_{h}^{j,v}=s,h_{v}=h,\pi^{j,v}=\pi^{j,h_{v}},\sigma_{s}^{j,v}(a)=1\}
        \\
        \label{eq:n_geq_optimmistic_visit}
        & \underset{(**)}{=}
        \sum_{j=1}^{k-1} \mathbb{I}\{s_{h}^{j,\underl{\pi}^{j}}=s\} \mathbb{I}\{\exists v : h_{v}=h,\pi^{j,v}=\pi^{j,h_{v}},\sigma_{s}^{j,v}(a)=1\},
    \end{align}
    For $(*)$, recall that $a\in \mathcal{A}_h^k(s)$, which implies that $a\in \mathcal{A}_h^j(s)$. Therefore, if $h_{v}=h,\pi^{j,v}=\pi^{j,h_{v}},\sigma^{j,v}(a) = 1$ (that is, the agent explores $h$, and $a$ is the first action in the permutation), then $a_h^k = a$. $(**)$ is because each agent that randomize $h_v = h$ follows the (deterministic) optimistic until horizon $h$.
    Since $h_v$, $\sigma^{j,v}$ and the event $\{\pi^{j,v}=\pi^{j,h_{v}}\}$ are randomized independently,
    \begin{align*}
        \mathbb{E} \biggl[ \mathbb{I}\{s_{h}^{j,\underl{\pi}^{j}}=s\} & \mathbb{I}\{\exists v : h_{v}=h,\pi^{j,v}=\pi^{j,h_{v}},\sigma_{s}^{j,v}(a)=1\} \mid \filt^{j} \biggr] =
        \\
        & = 
        q_h^{\underl{\pi}^j}(s) \Pr \left[ \exists v : h_{v}=h,\pi^{j,v}=\pi^{j,h_{v}},\sigma_{s}^{j,v}(a)=1 \right]
        \\
        & = 
        q_h^{\underl{\pi}^j}(s) \left[ 1-\Pr \left[ \forall v : h_{v}\ne h\vee\pi^{j,v}\ne\pi^{j,h_{v}}\vee\sigma_{s}^{j,v}(a)\ne1 \right] \right]
        \\
        & =
        q_h^{\underl{\pi}^j}(s) \left[ 1- \left( \Pr[h_{1}\ne h\vee\pi^{j,1}\ne\pi^{j,h_{1}}\vee\sigma_{s}^{j,1}(a)\ne1] \right)^{m} \right]
        \\
        & =
        q_h^{\underl{\pi}^j}(s) \left[ 1- \left( 1-\Pr[h_{1}=h,\pi^{j,1}=\pi^{j,h_{1}},\sigma_{s}^{j,1}(a)=1] \right)^{m} \right]
        \\
        & =
        q_h^{\underl{\pi}^j}(s) \left[ 1- \left( 1-\frac{\epsilon}{HA} \right)^{m} \right]
        \\
        & =
        q_h^{\underl{\pi}^j}(s) \left[ 1- \left(  \left( 1-\frac{\epsilon}{HA} \right)^{\frac{HA}{\epsilon}} \right)^{m\frac{\epsilon}{HA}}  \right]
        \\
        \tag{$( 1-x^{-1})^x \leq e$}
        & \geq q_h^{\underl{\pi}^j}(s) \left[ 1-e^{-\frac{m\epsilon}{HA}} \right]
        \\
        \tag{\ensuremath{e^{-x}\leq1-x+\frac{x^{2}}{2}}}
        & \geq  
        q_h^{\underl{\pi}^j}(s) \left[ \frac{m\epsilon}{HA}-\frac{1}{2} \left( \frac{m\epsilon}{HA} \right)^{2} \right]
        \\
        \tag{\ensuremath{\epsilon\leq\frac{HA}{m}}}      
        & \geq  
        q_h^{\underl{\pi}^j}(s) \frac{m\epsilon}{2HA}.
    \end{align*}
    By \cite{dann2017unifying}[Lemma F.4] and \cref{eq:n_geq_optimmistic_visit} we have,
    $
        \Pr\left[\exists k: n_{h}^{k}(s,a) < \frac{m\epsilon}{4HA}\sum_{j=1}^{k-1}q_{h}^{\underl{\pi}^{j}}(s) - \log\frac{6HSA}{\delta} \right]
        \leq 
        \frac{\delta}{6HSA}.
    $
    By taking the union bound over all $h\in[H],s\in\mathcal{S}$ and $a\in\mathcal{A}$, we get $\Pr(\neg E^{n1})\leq \delta/6$.
    
    \textbf{$\Pr[\neg E^{n2}] \le \delta/6$:} 
    For any $v\in[m]$,
    \begin{align*}
        n_{h}^{k}(s,a) 
        =
        \sum_{j=1}^{k-1} \mathbb{I}\{\exists v'\in[m]:s_{h}^{j,v'}=s, a_{h}^{j,v'}=a\}
        \geq
        \sum_{j=1}^{k-1} \mathbb{I}\{s_{h}^{j,v}=s, a_{h}^{j,v}=a\}.
    \end{align*}
    Again, by \cite{dann2017unifying}[Lemma F.4], we get 
    $
        \Pr\left[\exists k\in[K]:n_{h}^{k}(s,a)<\frac{1}{2}\sum_{j=1}^{k-1}q_{h}^{\pi^{j,v}}(s,a)-\log\frac{6mHSA}{\delta}\right] 
        \leq
        \frac{\delta}{6mHSA}.
    $
    Taking the union bound we get $\Pr[\neg E^{n2}] \leq \delta/6$.
    
    \textbf{$\Pr[\neg E^{\epsilon}] \le \delta/6$:} 
    Directly from Hoeffding's inequality and a union bound.
\end{proof}

\subsection{Proof of Theorem~\ref{thm:reg-coop-ulcae}}

\begin{proof}[Proof of Theorem~\ref{thm:reg-coop-ulcae}]
    By \cref{lem:good-event-stochastic-non-fresh}, the good event holds with probability $1 - \delta$.
    We now analyze the regret under the assumption that the good event holds.
    We start by decomposing the regret according to the policy played by agent $v$:
    \begin{align*}
        \regret
        & =
        \sum_{k=1}^{K} V^{\pi^{k,v}}_1(s_{1}^{k,v}) - V^{\star}_1(s_{1}^{k,v})
        \\
        & =
        \sum_{k=1}^K \indevent{\pi^{k,v} = \underl{\pi}^k} \left( V^{\underl{\pi}^k}_1(s_{1}^{k,v}) - V^{\star}_1(s_{1}^{k,v}) \right)
        +
        \sum_{k=1}^K \sum_{h'=1}^H \indevent{\pi^{k,v} = \pi^{k,h'}} \left( V^{\pi^{k,h'}}_1(s_{1}^{k,v}) - V^{\star}_1(s_{1}^{k,v}) \right)
        \\
        & \le
        \sum_{k=1}^K V^{\underl{\pi}^k}_1(s_{1}^{k,v}) - V^{\star}_1(s_{1}^{k,v})
        +
        \sum_{k=1}^K \sum_{h'=1}^H \indevent{\pi^{k,v} = \pi^{k,h'}} \left( V^{\pi^{k,h'}}_1(s_{1}^{k,v}) - V^{\star}_1(s_{1}^{k,v}) \right).
    \end{align*}
    
    For the first term we use \cref{lem:reg-optimistic-pi}, then \cref{lem:sum-1-n-optimistic-policy} and then \cref{lem:sum-variance-under-opitmistic-policy}:
    \begin{align*}
        \sum_{k=1}^K V^{\underl{\pi}^k}_1(s_{1}^{k,v}) - V^{\star}_1(s_{1}^{k,v})
        & \lesssim
        H \sum_{k=1}^K \sum_{h=1}^H \bbEk \left[ \frac{\sqrt{\logterm \VAR_{p_h(\cdot \mid s^{k}_h,a^{k}_h)}(V^{\underl{\pi}^k}_{h+1})}}{\sqrt{n^{k}_h(s^{k}_h,a^{k}_h)\vee 1}} + \frac{H^2 S \logterm}{n^{k}_h(s^{k}_h,a^{k}_h)\vee 1} \mid \underl{\pi}^k \right]
        \\
        & \lesssim
        H \sqrt{\logterm} \sum_{k=1}^K \sum_{h=1}^H \bbEk \left[ \frac{\sqrt{ \VAR_{p_h(\cdot \mid s^{k}_h,a^{k}_h)}(V^{\underl{\pi}^k}_{h+1})}}{\sqrt{n^{k}_h(s^{k}_h,a^{k}_h)\vee 1}} \mid \underl{\pi}^k \right] + \frac{H^5 S^2 A^2 \logterm^2}{\numagents \epsilon}
        \\
        & \lesssim
        \sqrt{\frac{H^6 S A K \logterm^2}{\numagents \epsilon}} + \frac{H^5 S^2 A^2 \logterm^2}{\numagents \epsilon}.
    \end{align*}
    
    For the second term we use \cref{lem:reg-pi-k-h}, then \cref{lem:sum-1-n-pi-k-h,lem:sum-1-n-optimistic-policy} and then \cref{lem:sum-1-sqrt-n-pi-k-h,lem:sum-1-sqrt-n-pi-optim}:
    \begin{align*}
        \sum_{k=1}^K \sum_{h'=1}^H & \indevent{\pi^{k,v} = \pi^{k,h'}} \left( V^{\pi^{k,h'}}_1(s_{1}^{k,v}) - V^{\star}_1(s_{1}^{k,v}) \right)
        \lesssim
        \\
        & \lesssim
        H \sum_{k=1}^K \sum_{h'=1}^H \sum_{h=1}^H \indevent{\pi^{k,v} = \pi^{k,h'}} \bbEk \left[ \frac{\sqrt{\logterm \VAR_{p_h(\cdot \mid s^{k}_h,a^{k}_h)}(V^{\underl{\pi}^k}_{h+1})}}{\sqrt{n^{k}_h(s^{k}_h,a^{k}_h)\vee 1}} + \frac{H^2 S \logterm}{n^{k}_h(s^{k}_h,a^{k}_h)\vee 1} \mid \pi^{k,h'} \right]
        \\
        & \qquad +
        H^2 \sum_{k=1}^K \sum_{h'=1}^H \sum_{h=1}^H \indevent{\pi^{k,v} = \pi^{k,h'}} \bbEk \left[ \frac{\sqrt{\logterm \VAR_{p_h(\cdot \mid s^{k}_h,a^{k}_h)}(V^{\underl{\pi}^k}_{h+1})}}{\sqrt{n^{k}_h(s^{k}_h,a^{k}_h)\vee 1}} + \frac{H^2 S \logterm}{n^{k}_h(s^{k}_h,a^{k}_h)\vee 1} \mid \underl{\pi}^k \right]
        \\
        & \lesssim
        H \sqrt{\logterm} \sum_{k=1}^K \sum_{h'=1}^H \sum_{h=1}^H \indevent{\pi^{k,v} = \pi^{k,h'}} \bbEk \left[ \frac{\sqrt{ \VAR_{p_h(\cdot \mid s^{k}_h,a^{k}_h)}(V^{\underl{\pi}^k}_{h+1})}}{\sqrt{n^{k}_h(s^{k}_h,a^{k}_h)\vee 1}} \mid \pi^{k,h'} \right] + H^5 S^2 A \logterm^2 
        \\
        & \qquad +
        H^2 \sqrt{\logterm} \sum_{k=1}^K \sum_{h'=1}^H \sum_{h=1}^H \indevent{\pi^{k,v} = \pi^{k,h'}} \bbEk \left[ \frac{\sqrt{ \VAR_{p_h(\cdot \mid s^{k}_h,a^{k}_h)}(V^{\underl{\pi}^k}_{h+1})}}{\sqrt{n^{k}_h(s^{k}_h,a^{k}_h)\vee 1}} \mid \underl{\pi}^k \right] + \frac{H^6 S^2 A^2 \logterm^2}{\numagents \epsilon}
        \\
        & \lesssim
        \sqrt{H^7 S A K \epsilon \logterm^2} 
        + \logterm \sqrt{H^7 S A \logterm^3 }K^{1/4} + H^5 S^2 A \logterm^2
        +
        \sqrt{\frac{H^{8} S A K \logterm}{\numagents}} 
        + \sqrt{\frac{H^{9} S A \logterm^{2}}{\numagents \epsilon}} K^{1/4} 
        + \frac{H^6 S^2 A^2 \logterm^2}{\numagents \epsilon}.
    \end{align*}
    Setting $\epsilon=\min\left\{ \frac{HA}{m},\frac{1}{\sqrt{m}}\right\} $, we get:
    \begin{align*}
        \regret
        & \lesssim
        \sqrt{\frac{H^6 S A K \logterm^2}{\numagents \epsilon}} + \frac{H^6 S^2 A^2 \logterm^2}{\numagents \epsilon} + \sqrt{H^7 S A K \epsilon \logterm^2} 
        + \logterm \sqrt{H^7 S A \logterm^3 }K^{1/4} 
        \\
        & \qquad + 
        H^5 S^2 A \logterm^2
        +
        \sqrt{\frac{H^{8} S A K \logterm}{\numagents}} 
        + \sqrt{\frac{H^{9} S A \logterm^{2}}{\numagents \epsilon}} K^{1/4}
        \\
        & \lesssim
        \sqrt{H^5 S K \logterm^2} + \sqrt{\frac{H^7 S A K \logterm^2}{\sqrt{\numagents}}} + \sqrt{\frac{H^{8} S A K \logterm}{\numagents}} + H^5 S^2 A \logterm^2 + \frac{H^6 S^2 A^2 \logterm^2}{\sqrt{\numagents}}
        \\
        & \qquad +
        \sqrt{H^7 S A \logterm^3 }K^{1/4} + \sqrt{H^8 S \logterm^2 }K^{1/4} + \sqrt{\frac{H^{9} S A \logterm^{2}}{\sqrt{\numagents}}} K^{1/4}
        \\
        & \lesssim
        \sqrt{H^5 S K \logterm^2} + \sqrt{\frac{H^7 S A K \logterm^2}{\sqrt{\numagents}}} + \sqrt{\frac{H^{8} S A K \logterm}{\numagents}} + H^5 S^2 A \logterm^2 + \frac{H^6 S^2 A^2 \logterm^2}{\sqrt{\numagents}},
    \end{align*}
    where the last inequality follows because the $K^{1/4}$ terms are dominant only when $K$ is small, and in these cases the constant terms are larger.
\end{proof}

\subsection{Bounds on the cumulative bonuses}

\begin{lemma}
    \label{lem:sum-variance-under-opitmistic-policy}
    Under the good event, if $\frac{\numagents \epsilon}{HA} \leq 1$, 
    \[
        \sum_{k=1}^K \sum_{h=1}^H \bbEk \left[ \frac{\sqrt{ \VAR_{p_h(\cdot \mid s^{k}_h,a^{k}_h)}(V^{\underl{\pi}^k}_{h+1})}}{\sqrt{n^{k}_h(s^{k}_h,a^{k}_h)\vee 1}} \mid \underl{\pi}^k \right] \lesssim \sqrt{\frac{H^4 S A K \logterm}{\numagents \epsilon}} + \frac{H^3 S A^2 \logterm}{\numagents \epsilon}.
    \]
\end{lemma}

\begin{proof}
    By the event $E^{n1}$, we have:
    \begin{align*}
        \sum_{k=1}^{K}\sum_{h=1}^{H} & \bbEk \left[ \frac{\sqrt{\VAR_{p_{h}(\cdot \mid  s_{h}^{k},a_{h}^{k})} (V_{h+1}^{\underl{\pi}^{k}})}} {\sqrt{n_{h}^{k}(s_{h}^{k},a_{h}^{k})\vee1}} \mid \underl{\pi}^{k} \right]
        =
        \sum_{k=1}^{K}\sum_{h=1}^{H}\sum_{s\in\calS}\sum_{a\in\calA}\frac{q_{h}^{\underl{\pi}^{k}}(s,a) \sqrt{\VAR_{p_{h}(\cdot \mid  s,a)} (V_{h+1}^{\underl{\pi}^{k}})}}{\sqrt{n_{h}^{k}(s,a)\vee1}}
        \\
        & \leq \sum_{s\in\calS}\sum_{a\in\calA}\sum_{h=1}^{H}\sum_{k=1}^{K} \frac{q_{h}^{\underl{\pi}^{k}}(s,a) \sqrt{\VAR_{p_{h}(\cdot \mid  s,a)} (V_{h+1}^{\underl{\pi}^{k}})}}{\sqrt{( \frac{\numagents \epsilon}{4HA}\sum_{j=1}^{k-1} q_{h}^{\underl{\pi}^{j}}(s)-\logterm)\vee1}}
        \\
        & \leq \sum_{s\in\calS}\sum_{a\in\calA}\sum_{h=1}^{H}\sum_{k: \frac{\numagents \epsilon}{8HA}\sum_{j=1}^{k-1} q_{h}^{\underl{\pi}^{j}}(s)\leq \logterm}^{K} q_{h}^{\underl{\pi}^{k}}(s,a) 
        \underbrace{\sqrt{\VAR_{p_{h}(\cdot \mid  s,a)} (V_{h+1}^{\underl{\pi}^{k}})}}
        _{\leq H}
        \\
        & \quad +\sum_{s\in\calS}\sum_{a\in\calA}\sum_{h=1}^{H}\sum_{k: \frac{\numagents \epsilon}{8HA}\sum_{j=1}^{k-1} q_{h}^{\underl{\pi}^{j}}(s)>\logterm}^{K} \frac{q_{h}^{\underl{\pi}^{k}}(s,a) \sqrt{\VAR_{p_{h}(\cdot \mid  s,a)} (V_{h+1}^{\underl{\pi}^{k}})}}{\sqrt{( \frac{\numagents \epsilon}{4HA}\sum_{j=1}^{k-1} q_{h}^{\underl{\pi}^{j}}(s)-\logterm)\vee1}}
        \\
        & \lesssim \frac{H^{3}A^{2}S\logterm}{\numagents \epsilon}+\sum_{s\in\calS}\sum_{a\in\calA} \sum_{h=1}^{H} \sum_{k=1}^{K} \frac{q_{h}^{\underl{\pi}^{k}}(s,a) \sqrt{\VAR_{p_{h}(\cdot \mid  s,a)} (V_{h+1}^{\underl{\pi}^{k}})}}{\sqrt{( \frac{\numagents \epsilon}{HA} \sum_{j=1}^{k-1} q_{h}^{\underl{\pi}^{j}}(s))\vee1}}
        \\
        & \leq \frac{H^{3}A^{2}S\logterm}{\numagents \epsilon} + \sqrt{ \sum_{s\in\calS}\sum_{a\in\calA}\sum_{h=1}^{H} \sum_{k=1}^{K} q_{h}^{\underl{\pi}^{k}}(s,a) \VAR_{p_{h}(\cdot \mid  s,a)}(V_{h+1}^{\underl{\pi}^{k}}) }
        \sqrt{ \sum_{s\in\calS}\sum_{a\in\calA}\sum_{h=1}^{H}\sum_{k=1}^{K} \frac{q_{h}^{\underl{\pi}^{k}}(s,a) }{( \frac{\numagents \epsilon}{HA}\sum_{j=1}^{k-1} q_{h}^{\underl{\pi}^{j}}(s))\vee1} }
        \\
        & = \frac{H^{3}A^{2}S\logterm}{\numagents \epsilon} + \sqrt{ \sum_{s\in\calS}\sum_{a\in\calA}\sum_{h=1}^{H} \sum_{k=1}^{K} q_{h}^{\underl{\pi}^{k}}(s,a) \VAR_{p_{h}(\cdot \mid  s,a)}(V_{h+1}^{\underl{\pi}^{k}}) }
        \sqrt{ \sum_{s\in\calS}\sum_{h=1}^{H}\sum_{k=1}^{K} \frac{q_{h}^{\underl{\pi}^{k}}(s) }{( \frac{\numagents \epsilon}{HA}\sum_{j=1}^{k-1} q_{h}^{\underl{\pi}^{j}}(s))\vee1} }, 
    \end{align*}
    where the last inequality is Cauchy–Schwarz inequality. Using the law of total variance \citep[Lemma B.14]{cohen2021minimax},
    \begin{align}
        \nonumber
    	  \sum_{h=1}^{H}\sum_{s\in\calS}\sum_{a\in\calA} & q_{h}^{\underl{\pi}^{k}}(s,a) \VAR_{p_{h}(\cdot\mid s,a)}(V_{h+1}^{\underl{\pi}^{k}})
    	  =\bbE \left[\sum_{h=1}^H\VAR_{p_{h}(\cdot\mid s,a)}(V_{h+1}^{\underl{\pi}^{k}})\mid \underl{\pi}^{k} \right]
    	  \\
    	  & =\bbE\left[\left(V_{1}^{\underl{\pi}^{k}}(s_{1}^{k})-\sum_{h}c_{h}(s_{h}^{k},a_{h}^{k})\right)^{2} \mid \underl{\pi}^{k}\right] 
    	  \leq H^{2}.
    	  \label{eq:law of total variance}
    \end{align}
    Using \citet{rosenberg2020near}[Lemma B.18], since $\frac{\numagents \epsilon}{HA} \leq 1$,
    \begin{align}
    	  \sum_{s\in\calS}\sum_{h=1}^{H}\sum_{k=1}^{K} \frac{q_{h}^{\underl{\pi}^{k}}(s)} {(\frac{\numagents \epsilon}{HA} \sum_{j=1}^{k-1} q_{h}^{\underl{\pi}^{j}}(s))\vee1}
    	  & =\frac{HA}{\numagents \epsilon}\sum_{s\in\calS}\sum_{h=1}^{H}\sum_{k=1}^{K} \frac{\frac{\numagents \epsilon}{HA} q_{h}^{\underl{\pi}^{k}}(s)} {(\frac{\numagents \epsilon}{HA} \sum_{j=1}^{k-1} q_{h}^{\underl{\pi}^{j}}(s))\vee1} 
    	  \lesssim\frac{H^2 S A\logterm}{\numagents \epsilon}.
    	  \label{eq:q/sumq=log}
    \end{align}
    Combining the last three inequalities completes the proof.
\end{proof}

\begin{lemma}
    \label{lem:sum-1-n-optimistic-policy}
    Under the good event, $\sum_{k=1}^K \sum_{h=1}^H \bbEk \left[ \frac{1}{n^{k}_h(s^k_h,a^k_h) \vee 1} \mid \underl{\pi}^k \right] \lesssim \frac{H^2 S A^2 \logterm}{\numagents \epsilon}$.
\end{lemma}

\begin{proof}
    By the event $E^{n1}$, we have:
    \begin{align*}
         \sum_{k=1}^{K} & \sum_{h=1}^{H} \bbEk \left[ \frac{1}{n_{h}^{k}(s_{h}^{k},a_{h}^{k})\vee1}\mid\underl{\pi}^{k} \right] 
        =\sumhsak\frac{q_{h}^{\underl{\pi}^{k}}(s,a)}{n_{h}^{k}(s,a)\vee1}
        \leq\sumhsak \frac{q_{h}^{\underl{\pi}^{k}}(s,a)}{ \left( \frac{\numagents\epsilon}{4HA} \sum_{j=1}^{k - 1}q_{h}^{\underl{\pi}^{j}}(s) - \logterm \right) \vee1}
        \\
        & \leq\sumhsa 
        \sum_{ k:\frac{m\epsilon}{8HA}\sum_{j=1}^{k-1}q_{h}^{\underl{\pi}^{j}}(s)\leq\logterm } 
        q_{h}^{\underl{\pi}^{k}}(s,a)
        +\sumhsa
        \sum_{ k:\frac{m\epsilon}{8HA}\sum_{j=1}^{k-1}q_{h}^{\underl{\pi}^{j}}(s)>\logterm }
        \frac{q_{h}^{\underl{\pi}^{k}}(s,a)}{ \left( \frac{\numagents\epsilon}{4HA} \sum_{j=1}^{k - 1}q_{h}^{\underl{\pi}^{j}}(s) - \logterm \right) \vee1}
        \\
        & \lesssim\frac{H^{2}SA^{2}\logterm}{\numagents\epsilon} + \frac{HA}{\numagents\epsilon}\sumhsak \frac{q_{h}^{\underl{\pi}^{k}}(s,a)}{ \left(  \sum_{j=1}^{k-1}q_{h}^{\underl{\pi}^{j}}(s) \right) \vee1} 
        \lesssim\frac{H^{2}SA^{2}\logterm}{\numagents\epsilon},
    \end{align*}
    where the last inequality follows from \citet{rosenberg2020near}[Lemma B.18].
\end{proof}

\begin{lemma}
    \label{lem:sum-1-n-pi-k-h}
    Let $h' \in [H]$.
    Under the good event, $\sum_{k=1}^K \sum_{h=1}^H \indevent{\pi^{k,v}=\pi^{k,h'}} \bbEk \left[ \frac{1}{n^{k}_h(s^k_h,a^k_h) \vee 1} \mid \pi^{k,h'} \right] \lesssim H S A \logterm$.
\end{lemma}

\begin{proof}
    By the event $E^{n2}$, we have:
    \begin{align*}
          \sum_{k=1}^{K}  \mathbb{I}\{\pi^{k,v} & =\pi^{k,h'}\} \sum_{h=1}^{H} \bbEk \left[ \frac{1}{n_{h}^{k}(s_{h}^{k},a_{h}^{k})\vee1}\mid\pi^{k,h'} \right] 
          =
    	   \sumhsak \mathbb{I}\{\pi^{k,v}=\pi^{k,h'}\}\frac{q_{h}^{\pi^{k,h'}}(s,a)}{n_{h}^{k}(s,a)\vee1}
    	  \\
    	  & \leq \sumhsak \mathbb{I}\{\pi^{k,v}=\pi^{k,h'}\}\frac{q_{h}^{\pi^{k,h'}}(s,a)}{ \left( \frac{1}{2} \sum_{j=1}^{k-1}q_{h}^{\pi^{j,v}}(s,a) - \logterm \right) \vee1}
    	  \\
    	  & \leq \sumhsak\frac{ \mathbb{I}\{\pi^{k,v}=\pi^{k,h'}\}q_{h}^{\pi^{k,h'}}(s,a)}{ \left( \frac{1}{2} \sum_{j=1}^{k-1} \mathbb{I}\{\pi^{j,v}=\pi^{j,h'}\}q_{h}^{\pi^{j,h'}}(s,a) - \logterm \right) \vee1} 
    	  \\
      	  & \leq \sumhsa \sum_{k:\frac{1}{4}\sum_{j=1}^{k-1}\mathbb{I}\{\pi^{j,v}=\pi^{j,h'}\}q_{h}^{\pi^{j,h'}}(s,a)\leq\logterm}  \mathbb{I}\{\pi^{k,v}=\pi^{k,h'}\}q_{h}^{\pi^{k,h'}}(s,a)
    	  \\
    	  & \qquad + \sumhsa
    	  \sum_{k:\frac{1}{4}\sum_{j=1}^{k-1}\mathbb{I}\{\pi^{j,v}=\pi^{j,h'}\}q_{h}^{\pi^{j,h'}}(s,a)>\logterm} 
    	  \frac{ \mathbb{I}\{\pi^{k,v}=\pi^{k,h'}\}q_{h}^{\pi^{k,h'}}(s,a)}{ \left( \frac{1}{2} \sum_{j=1}^{k-1} \mathbb{I}\{\pi^{j,v}=\pi^{j,h'}\}q_{h}^{\pi^{j,h'}}(s,a) - \logterm \right) \vee1} 
    	  \\
    	  & \lesssim HSA\logterm + \sumhsak\frac{ \mathbb{I}\{\pi^{k,v}=\pi^{k,h'}\}q_{h}^{\pi^{k,h'}}(s,a)}{ \left(  \sum_{j=1}^{k-1} \mathbb{I}\{\pi^{j,v}=\pi^{j,h'}\}q_{h}^{\pi^{j,h'}}(s,a) \right) \vee1}
    	  \lesssim HSA\logterm,
    \end{align*}
    where the last inequality follows \cite{rosenberg2020near}[Lemma B.18].
\end{proof}

\begin{lemma}
    \label{lem:sum-1-sqrt-n-pi-k-h}
    Let $h' \in [H]$.
    Under the good event, 
    \[
    \sum_{k=1}^K \sum_{h=1}^H \indevent{\pi^{k,v}=\pi^{k,h'}} \bbEk \left[ \frac{\sqrt{ \VAR_{p_h(\cdot \mid s^{k}_h,a^{k}_h)}(V^{\underl{\pi}^k}_{h+1})}} {\sqrt{n^{k}_h(s^k_h,a^k_h) \vee 1}} \mid \pi^{k,h'} \right] \lesssim \sqrt{H^3 S A K \epsilon \logterm} 
    + \logterm \sqrt{H^3 S A }K^{1/4} + H S A \logterm.
    \]
\end{lemma}

\begin{proof}
    First we bound $\sqrt{ \VAR_{p_h(\cdot \mid s^{k}_h,a^{k}_h)}(V^{\underl{\pi}^k}_{h+1})}$ by $H$. Now, under the good event,
    \begin{align*}
    	  \sum_{k=1}^{K} \mathbb{I}&\{\pi^{k,v}=\pi^{k,h'}\}\sum_{h=1}^{H} \bbEk \left[ \frac{1}{\sqrt{n_{h}^{k}(s_{h}^{k},a_{h}^{k})\vee1}}\mid\pi^{k,h'} \right] 
    	  =\sumhsak \mathbb{I}\{\pi^{k,v}=\pi^{k,h'}\}\frac{q_{h}^{\pi^{k,h'}}(s,a)}{\sqrt{n_{h}^{k}(s,a)\vee1}}
    	  \\
    	  & \leq\sumhsak \mathbb{I}\{\pi^{k,v}=\pi^{k,h'}\}\frac{q_{h}^{\pi^{k,h'}}(s,a)}{\sqrt{\left(\frac{1}{2}\sum_{j=1}^{k-1} q_{h}^{\pi^{j,v}}(s,a) - \logterm \right)\vee1}}
    	  \\
    	  & \leq\sumhsak \frac{ \mathbb{I}\{\pi^{k,v}=\pi^{k,h'}\} q_{h}^{\pi^{k,h'}}(s,a)}{\sqrt{\left(\frac{1}{2}\sum_{j=1}^{k-1} \mathbb{I}\{\pi^{j,v}=\pi^{j,h'}\} q_{h}^{\pi^{j,h'}}(s,a) - \logterm \right)\vee1}} 
    	  \\
    	  & \lesssim HSA \logterm + \sumhsak \frac{ \mathbb{I}\{\pi^{k,v}=\pi^{k,h'}\} q_{h}^{\pi^{k,h'}}(s,a)}{\sqrt{\left(\sum_{j=1}^{k-1} \mathbb{I}\{\pi^{j,v}=\pi^{j,h'}\} q_{h}^{\pi^{j,h'}}(s,a)\right)\vee1}}
    	  \\
    	  & \leq \sqrt{\sumhsak \frac{\mathbb{I}\{ \pi^{k,v}=\pi^{k,h'} \} q_{h}^{\pi^{k,h'}}(s,a)} {\left(\sum_{j=1}^{k-1}\mathbb{I}\{\pi^{j,v}=\pi^{j,h'}\}q_{h}^{\pi^{j,h'}}(s,a)\right)\vee1}}
    	  \sqrt{\sumhsak\mathbb{I} \{\pi^{k,v}=\pi^{k,h'}\} q_{h}^{\pi^{k,h'}}(s,a) }
    	  \\
    	  & \qquad +
    	  HSA\logterm,
    \end{align*}
    where the first inequality is by $E^{n2}$, the second inequality is done by breaking the sum to $k$s such that $\frac{1}{4}\sum_{j=1}^{k-1} q_{h}^{\pi^{j,v}}(s,a)\leq \logterm$ and the rest of the $k$s, as done in the proof of \cref{lem:sum-1-n-pi-k-h} for example, and the last is Cauchy–Schwarz inequality. 
    By \citet{rosenberg2020near}[Lemma B.18],
    \begin{align*}
        \sumhsak \frac{\mathbb{I}\{ \pi^{k,v}=\pi^{k,h'} \} q_{h}^{\pi^{k,h'}}(s,a)} {\left(\sum_{j=1}^{k-1}\mathbb{I}\{\pi^{j,v}=\pi^{j,h'}\}q_{h}^{\pi^{j,h'}}(s,a)\right)\vee1}
        \lesssim 
        HSA\logterm.
    \end{align*}
    By the good event $E^{\epsilon}$,
    \begin{align*}
        \sumhsak\mathbb{I} \{\pi^{k,v}=\pi^{k,h'}\} q_{h}^{\pi^{k,h'}}(s,a) 
        & =
        H \sum_{k=1}^K \mathbb{I} \{\pi^{k,v}=\pi^{k,h'}\}
        \leq K\epsilon + \sqrt{K\logterm}.
        \qedhere
    \end{align*}
\end{proof}

\begin{lemma}
    \label{lem:sum-1-sqrt-n-pi-optim}
    Under the good event, 
    \[
    \sum_{k=1}^K \sum_{h'=1}^H \sum_{h=1}^H \indevent{\pi^{k,v}=\pi^{k,h'}} \bbEk \left[ \frac{\sqrt{ \VAR_{p_h(\cdot \mid s^{k}_h,a^{k}_h)}(V^{\underl{\pi}^k}_{h+1})}} {\sqrt{n^{k}_h(s^k_h,a^k_h) \vee 1}} \mid \underl{\pi}^{k} \right] 
    \lesssim 
    \sqrt{\frac{H^{4} S A K \logterm}{\numagents}} 
    + \sqrt{\frac{H^{5} S A \logterm^{2}}{\numagents \epsilon}} K^{1/4} 
    + \frac{H^{3} A^{2} S \logterm}{\numagents \epsilon}.
    \]
\end{lemma}

\begin{proof}
    By the event $E^{n1}$, we have:
    \begin{align*}
        \sum_{k=1}^{K} \sum_{h'=1}^{H} \sum_{h=1}^{H} & \indevent{\pi^{k,v}=\pi^{k,h'}} \bbEk  \left[ \frac{\sqrt{\VAR_{p_{h}(\cdot \mid  s_{h}^{k},a_{h}^{k})}(V_{h+1}^{\underl{\pi}^{k}})}} {\sqrt{n_{h}^{k}(s_{h}^{k},a_{h}^{k} )\vee1}} \mid \underl{\pi}^{k}\right] =
        \\
        & = \sum_{k=1}^{K} \sum_{h'=1}^{H} \sum_{h=1}^{H} \sum_{s\in\calS} \sum_{a\in\calA} \indevent{\pi^{k,v}=\pi^{k,h'}}\frac{ q_{h}^{\underl{\pi}^{k}}(s,a)\sqrt{\VAR_{p_{h}(\cdot \mid  s,a)}(V_{h+1}^{\underl{\pi}^{k}})}} {\sqrt{n_{h}^{k}(s_{h}^{k},a_{h}^{k} )\vee1}}
        \\
        & \leq \sum_{k=1}^{K} \sum_{h'=1}^{H} \sum_{h=1}^{H} \sum_{s\in\calS} \sum_{a\in\calA} \indevent{\pi^{k,v}=\pi^{k,h'}}\frac{ q_{h}^{\underl{\pi}^{k}}(s,a)\sqrt{\VAR_{p_{h}(\cdot \mid  s,a)}(V_{h+1}^{\underl{\pi}^{k}})}} {\sqrt{(\frac{\numagents\epsilon}{4HA} \sum_{j=1}^{k-1} q_{h}^{\underl{\pi}^{j}}(s) - \logterm )\vee1}}
        \\
        & \lesssim
        \frac{H^{3}A^{2}S\logterm}{\numagents\epsilon}+ \sum_{k=1}^{K} \sum_{h'=1}^{H} \sum_{h=1}^{H} \sum_{s\in\calS} \sum_{a\in\calA} \indevent{\pi^{k,v}=\pi^{k,h'}}\frac{ q_{h}^{\underl{\pi}^{k}}(s,a)\sqrt{\VAR_{p_{h}(\cdot \mid  s,a)}(V_{h+1}^{\underl{\pi}^{k}})}} {\sqrt{(\frac{\numagents\epsilon}{HA} \sum_{j=1}^{k-1} q_{h}^{\underl{\pi}^{j}}(s) )\vee1}}
        \\
        & \leq \sqrt{ \sum_{k,h',h,s,a} \indevent{\pi^{k,v}=\pi^{k,h'}} q_{h}^{\underl{\pi}^{k}}(s,a)\VAR_{p_{h}(\cdot \mid  s,a)}(V_{h+1}^{\underl{\pi}^{k}})}\sqrt{ \sum_{k,h',h,s,a}\frac{ \indevent{\pi^{k,v}=\pi^{k,h'}} q_{h}^{\underl{\pi}^{k}}(s,a)} {(\frac{\numagents\epsilon}{HA} \sum_{j=1}^{k-1} q_{h}^{\underl{\pi}^{j}}(s) )\vee1}}
        \\
        & \qquad + 
        \frac{H^{3}A^{2}S\logterm}{\numagents\epsilon}
        \\
        & =
        \sqrt{ \sum_{k,h'} \indevent{\pi^{k,v}=\pi^{k,h'}} \sum_{h,s,a} q_{h}^{\underl{\pi}^{k}}(s,a)\VAR_{p_{h}(\cdot \mid  s,a)}(V_{h+1}^{\underl{\pi}^{k}})}\sqrt{ \sum_{k,h',h,s}\frac{ \indevent{\pi^{k,v}=\pi^{k,h'}} q_{h}^{\underl{\pi}^{k}}(s)} {(\frac{\numagents\epsilon}{HA} \sum_{j=1}^{k-1} q_{h}^{\underl{\pi}^{j}}(s) )\vee1}}
        \\
        & \qquad +
        \frac{H^{3}A^{2}S\logterm}{\numagents\epsilon},
    \end{align*}
    where the last inequality is by Cauchy-Schwarz. 
    By \cref{eq:law of total variance}
    and the good event $E^\epsilon$, 
    \begin{align*}
    	  \sum_{k=1}^{K} \sum_{h'=1}^{H} \indevent{\pi^{k,v}=\pi^{k,h'}} \sum_{h=1}^{H} \sum_{s\in\calS} \sum_{a\in\calA} q_{h}^{\underl{\pi}^{k}}(s,a)\VAR_{p_{h}(\cdot \mid  s,a)}(V_{h+1}^{\underl{\pi}^{k}}) 
    	  & \leq H^{2} \sum_{k=1}^{K} \sum_{h'=1}^{H} \indevent{\pi^{k,v}=\pi^{k,h'}}
    	  \\
    	  & \leq H^{2}K\epsilon + H^{3}\sqrt{K\log\frac{\numagents H}{\delta'}}.
    \end{align*}
    For last,
    \begin{align*}
    	 \sum_{k=1}^{K} \sum_{h=1}^{H} \sum_{s\in\calS} \sum_{h'=1}^{H}\frac{ \indevent{\pi^{k,v}=\pi^{k,h'}} q_{h}^{\underl{\pi}^{k}}(s)} {(\frac{\numagents\epsilon}{HA} \sum_{j=1}^{k-1} q_{h}^{\underl{\pi}^{j}}(s) )\vee1} 
    	 & \leq 
    	 \sum_{k=1}^{K} \sum_{h=1}^{H} \sum_{s\in\calS}\frac{ q_{h}^{\underl{\pi}^{k}}(s)} {(\frac{\numagents\epsilon}{HA} \sum_{j=1}^{k-1} q_{h}^{\underl{\pi}^{j}}(s) )\vee1}
	     \lesssim 
    	 \frac{H^{2}SA\logterm}{\numagents\epsilon}, 
    \end{align*}
    where the first inequality is since $ \sum_{h'=1}^{H} \indevent{\pi^{k,v}=\pi^{k,h'}}\leq1$
    and the last is as in \cref{eq:q/sumq=log}. Combining the last three inequalities
    completes the proof.
\end{proof}

\subsection{Bounding the regret of $\pi^{k,h'}$ and the optimistic policy}

\begin{lemma}
    \label{lem:reg-pi-k-h}
    Under the good event, for every $k \in [K]$ it holds that:
    \begin{align*}
        V_{1}^{\pi^{k,h'}}(s_{1}^{k}) - V_{1}^{\star}(s_{1}^{k})
        & \lesssim
        H \sum_{h=1}^H \bbEk \left[ \frac{\sqrt{\logterm \VAR_{p_h(\cdot \mid s^{k}_h,a^{k}_h)}(V^{\underl{\pi}^k}_{h+1})}}{\sqrt{n^{k}_h(s^{k}_h,a^{k}_h)\vee 1}} + \frac{H^2 S \logterm}{n^{k}_h(s^{k}_h,a^{k}_h)\vee 1} \mid \pi^{k,h'} \right]
        \\
        & \qquad +
        H^2 \sum_{h=1}^H \bbEk \left[ \frac{\sqrt{\logterm \VAR_{p_h(\cdot \mid s^{k}_h,a^{k}_h)}(V^{\underl{\pi}^k}_{h+1})}}{\sqrt{n^{k}_h(s^{k}_h,a^{k}_h)\vee 1}} + \frac{H^2 S \logterm}{n^{k}_h(s^{k}_h,a^{k}_h)\vee 1} \mid \underl{\pi}^k \right].
    \end{align*}
\end{lemma}

\begin{proof}
    Apply \cref{lem:reg-for-episode-k-pi-k-h} for every $k \in [K]$ and then apply \cref{lemma:rec-optimistic,lemma:rec-non-optimistic,lemma:rec-exp,lemma:rec-optimistic-policy} iteratively.
\end{proof}

\begin{lemma}
    \label{lem:reg-optimistic-pi}
    Under the good event, for every $k \in [K]$ it holds that:
    \begin{align*}
        V_{1}^{\underl{\pi}^{k}}(s_{1}^{k}) - V_{1}^{\star}(s_{1}^{k})
        \lesssim
        H \sum_{h=1}^H \bbEk \left[ \frac{\sqrt{\logterm \VAR_{p_h(\cdot \mid s^{k}_h,a^{k}_h)}(V^{\underl{\pi}^k}_{h+1})}}{\sqrt{n^{k}_h(s^{k}_h,a^{k}_h)\vee 1}} + \frac{H^2 S \logterm}{n^{k}_h(s^{k}_h,a^{k}_h)\vee 1} \mid \underl{\pi}^k \right].
    \end{align*}
\end{lemma}

\begin{proof}
    Similar to the proof of \cref{lem:reg-pi-k-h}.
\end{proof}

\begin{lemma}
    \label{lem:reg-for-episode-k-pi-k-h}
    Let $h' \in [H]$.
    Under the good event, for every $k \in [K]$ it holds that:
    \begin{align*}
        V_{1}^{\pi^{k,h'}}(s_{1}^{k}) - V_{1}^{\star}(s_{1}^{k})
        & \le
        \sum_{h=1}^H \bbEk \left[\overl{Q}_{h}^{k}(s_{h}^{k},a_{h}^{k}) - \underl{Q}_{h}^{k}(s_{h}^{k},a_{h}^{k})\mid\pi^{k,h'}\right] 
        \\
        & \qquad + \bbEk \left[\overl{Q}_{h'}^{k}(s_{h'}^{k},\underl{\pi}^{k}_{h'}(s_{h'}^{k})) - \underl{Q}_{h'}^{k}(s_{h'}^{k},\underl{\pi}^{k}_{h'}(s_{h'}^{k}))\mid\pi^{k,h'}\right].
    \end{align*}
\end{lemma}

\begin{proof}
    It holds that:
    \begin{align*}
    	  V_{1}^{\pi^{k,h'}}(s_{1}^{k}) - V_{1}^{\star}(s_{1}^{k})
    	  & =V_{1}^{\pi^{k,h'}}(s_{1}^{k}) - V_{1}^{\star}(s_{1}^{k})
    	  =\sum_{h} \bbEk \left[\langle Q_{h}^{\star}(s_{h}^{k},\cdot),\pi^{k,h'}(\cdot\mid s_{h}^{k}) - \pi_{h}^{\star}(\cdot\mid s_{h}^{k})\rangle\mid\pi^{k,h'}\right]
    	  \\
    	  & =\sum_{h} \bbEk \left[Q_{h}^{\star}(s_{h}^{k},a_{h}^{k}) - V_{h}^{\star}(s_{h}^{k})\mid\pi^{k,h'}\right]
    	  \leq\sum_{h} \bbEk \left[\overl{Q}_{h}^{k}(s_{h}^{k},a_{h}^{k}) - \underl{V}_{h}^{k}(s_{h}^{k})\mid\pi^{k,h'}\right]
    	  \\
    	  & =\sum_{h} \bbEk \left[\overl{Q}_{h}^{k}(s_{h}^{k},a_{h}^{k}) - \underl{Q}_{h}^{k}(s_{h}^{k},\underl{\pi}^{k}_h(s_{h}^{k}))\mid\pi^{k,h'}\right]
    	  \\
    	  & =\sum_{h\ne h'} \bbEk \left[\overl{Q}_{h}^{k}(s_{h}^{k},a_{h}^{k}) - \underl{Q}_{h}^{k}(s_{h}^{k},\underl{\pi}^{k}_h(s_{h}^{k}))\mid\pi^{k,h'}\right] 
    	  \\
    	  & \qquad + 
    	  \bbEk \left[\overl{Q}_{h'}^{k}(s_{h'}^{k},a_{h'}^{k}) - \underl{Q}_{h'}^{k}(s_{h'}^{k},\underl{\pi}^{k}_{h'}(s_{h'}^{k}))\mid\pi^{k,h'}\right]
    	  \\
    	  & =\sum_{h\ne h'} \bbEk \left[\overl{Q}_{h}^{k}(s_{h}^{k},a_{h}^{k}) - \underl{Q}_{h}^{k}(s_{h}^{k},a_{h}^{k})\mid\pi^{k,h'}\right] 
    	  \\
    	  & \qquad + 
    	  \bbEk \left[\overl{Q}_{h'}^{k}(s_{h'}^{k},a_{h'}^{k}) - \underl{Q}_{h'}^{k}(s_{h'}^{k},\underl{\pi}^{k}_{h'}(s_{h'}^{k}))\mid\pi^{k,h'}\right] 
    	  \\
    	  & \leq \sum_{h} \bbEk \left[\overl{Q}_{h}^{k}(s_{h}^{k},a_{h}^{k}) - \underl{Q}_{h}^{k}(s_{h}^{k},a_{h}^{k})\mid\pi^{k,h'}\right] 
    	  \\
    	  & \qquad +
    	  \bbEk \left[\overl{Q}_{h'}^{k}(s_{h'}^{k},\underl{\pi}^{k}_{h'}(s_{h'}^{k})) - \underl{Q}_{h'}^{k}(s_{h'}^{k},\underl{\pi}^{k}_{h'}(s_{h'}^{k}))\mid\pi^{k,h'}\right].      
    \end{align*}
    The first inequality is by \cref{lemma: optimism ucbvi-non-fresh}, the last equality is since $\pi^{k,h'} = \underl{\pi}^{k}$ for any $h\ne h'$ and the last inequality is by \Cref{lemma:active-arms}.
\end{proof}

\begin{lemma}
    \label{lem:reg-for-episode-k-optimistic-pi}
    Under the good event, for every $k \in [K]$ it holds that:
    \[
        V_{1}^{\underl{\pi}^{k}}(s_{1}^{k}) - V_{1}^{\star}(s_{1}^{k})
        \le
        \sum_{h=1}^H \bbEk \left[\overl{Q}_{h}^{k}(s_{h}^{k},a_{h}^{k}) - \underl{Q}_{h}^{k}(s_{h}^{k},a_{h}^{k})\mid \underl{\pi}^{k}\right].
    \]
\end{lemma}

\begin{proof}
    Similar to the proof of \cref{lem:reg-for-episode-k-pi-k-h}.
\end{proof}

\subsection{Auxiliary lemmas}

\begin{lemma}           
    \label{lemma:active-arms}
    If $a,a'\in \calA^k_h(s)$ then,
    \begin{align*}
        \overl{Q}_{h}^{k}(s,a) - \underl{Q}_{h}^{k}(s_{h}^{k},a')
        \leq
        \overl{Q}_{h}^{k}(s,a) - \underl{Q}_{h}^{k}(s,a) + \overl{Q}_{h}^{k}(s,a') - \underl{Q}_{h}^{k}(s,a').
    \end{align*}    
\end{lemma}

\begin{proof}
    Since $a,a'\in \calA^k_h(s)$, we have that $\underl{Q}_{h}^{k}(s,a) \le \overl{Q}_{h}^{k}(s,a')$.
    Thus:
    \begin{align*}
        \overl{Q}_{h}^{k}(s,a) - \underl{Q}_{h}^{k}(s,a')
        & =
        \overl{Q}_{h}^{k}(s,a) - \underl{Q}_{h}^{k}(s,a) 
        + \overl{Q}_{h}^{k}(s,a') - \underl{Q}_{h}^{k}(s,a')
        +\underbrace{ \underl{Q}_{h}^{k}(s,a) - \overl{Q}_{h}^{k}(s,a')}
        _{\leq 0}. \qedhere
    \end{align*}
\end{proof}

\begin{lemma}[Recursion with Optimistic Next-Action]
    \label{lemma:rec-optimistic}
    Let $h \ne h' - 1$. Under the good event, for every $k \in [K]$ it holds that:
    \begin{align*}
        \bbEk \left[\overl{Q}_{h}^{k}(s_{h}^{k},a_{h}^{k}) - \underl{Q}_{h}(s_{h}^{k},a_{h}^{k})\mid\pi^{k,h'}\right]
        & \le
        \bbEk \left[ \frac{8 \sqrt{\logterm \VAR_{p_h(\cdot \mid s^{k}_h,a^{k}_h)}(V^{\underl{\pi}^k}_{h+1})}}{\sqrt{n^{k}_h(s^{k}_h,a^{k}_h)\vee 1}} + \frac{118 H^2 S \logterm}{n^{k}_h(s^{k}_h,a^{k}_h)\vee 1} \mid \pi^{k,h'} \right]
        \\
        & \qquad +
        \left( 1 + \frac{1}{4H} \right) \bbEk \Big[\overl{Q}_{h+1}^{k}(s_{h+1}^{k},a_{h+1}^{k}) - \underl{Q}_{h+1}^{k}(s_{h+1}^{k},a_{h+1}^{k})\mid\pi^{k,h'}\Big].
    \end{align*}
\end{lemma}

\begin{proof}
    By definition of the optimistic and pessimistic $Q$-functions, we have:
    \begin{align*}
        \bbEk [ & \overl{Q}_{h}^{k}(s_{h}^{k},a_{h}^{k}) - \underl{Q}_{h}(s_{h}^{k},a_{h}^{k})\mid\pi^{k,h'}]
        =
        \\
        & =
        \bbEk [ 2b^k_h(s^{k}_h,a^{k}_h ; c) + 2 b^k_h(s^{k}_h,a^{k}_h ; p) \mid \pi^{k,h'} ]
        \\
        & \qquad
        + \bbEk \left[ \bbE_{\hat{p}^{k}_h(\cdot \mid  s^{k}_h,a^{k}_h)}[  \overl{V}^k_{h+1} - \underl{V}^k_{h+1}] \mid \pi^{k,h'} \right]
        \\
        & \le
        \bbEk [ 2b^k_h(s^{k}_h,a^{k}_h ; c) + 2 b^k_h(s^{k}_h,a^{k}_h ; p) \mid \pi^{k,h'} ]
        \\
        & \qquad +
        \bbEk \left[ \frac{18 H^2 S \logterm}{n^{k}_h(s^{k}_h,a^{k}_h)\vee 1}+ \left( 1 + \frac{1}{16H} \right) \bbE_{p_h(\cdot \mid  s^{k}_h,a^{k}_h)} [ \overl{V}^k_{h+1} - \underl{V}^k_{h+1}] \mid \pi^{k,h'} \right]
        \\
        & =
        \bbEk \left[ 2b^k_h(s^{k}_h,a^{k}_h ; c) + 2 b^k_h(s^{k}_h,a^{k}_h ; p) + \frac{18 H^2 S \logterm}{n^{k}_h(s^{k}_h,a^{k}_h)\vee 1} \mid \pi^{k,h'} \right]
        \\
        & \qquad +
        \left( 1 + \frac{1}{16H} \right) \bbEk \left[ \bbE_{p_h(\cdot \mid  s^{k}_h,a^{k}_h)} [ \overl{V}^k_{h+1}(s^{k}_{h+1}) - \underl{V}^k_{h+1}(s^{k}_{h+1})] \mid \pi^{k,h'} \right]
        \\
        & \le
        \bbEk \left[ \frac{8 \sqrt{\logterm \VAR_{p_h(\cdot \mid s^{k}_h,a^{k}_h)}(V^{\underl{\pi}^k}_{h+1})}}{\sqrt{n^{k}_h(s^{k}_h,a^{k}_h)\vee 1}} + \frac{118 H^2 S \logterm}{n^{k}_h(s^{k}_h,a^{k}_h)\vee 1} \mid \pi^{k,h'} \right]
        \\
        & \qquad +
        \left( 1 + \frac{1}{4H} \right) \bbEk \left[ \bbE_{p_h(\cdot \mid  s^{k}_h,a^{k}_h)} [ \overl{V}^k_{h+1} - \underl{V}^k_{h+1}] \mid \pi^{k,h'} \right]
    \end{align*}
    where the first inequality is by \citet[Lemma B.13]{cohen2021minimax}, and the second one is by \citet[Lemma B.6]{cohen2021minimax}.
    Let $\overl{\pi}_h^k(s) = \arg\min_a \overl{Q}_h^k(s,a)$. For the last term we have:
    \begin{align*}
       \bbEk \Bigl[ \bbE_{p_h(\cdot \mid  s^{k}_h,a^{k}_h)} & [ \overl{V}^k_{h+1} - \underl{V}^k_{h+1}] \mid \pi^{k,h'} \Bigr]
        =
        \\
        & =
        \sum_{h,s,a,s'}q_{h}^{\pi^{k,h'}}(s,a)p_{h}(s'\mid s,a)\left(\overl{V}_{h+1}^{k}(s') - \underl{V}_{h+1}^{k}(s')\right)
        \\
        & =
        \sum_{h,s}q_{h+1}^{\pi^{k,h'}}(s)\left(\overl{V}_{h+1}^{k}(s) - \underl{V}_{h+1}^{k}(s)\right)
        \\
        & =
        \bbEk \left[\overl{V}_{h+1}^{k}(s_{h+1}^{k}) - \underl{V}_{h+1}^{k}(s_{h+1}^{k})\mid\pi^{k,h'}\right]
        \\
        & =
        \bbEk \Big[\overl{Q}_{h+1}^{k}(s_{h+1}^{k},\overl{\pi}_{h+1}^{k}(s_{h+1}^{k})) - \underl{Q}_{h+1}^{k}(s_{h+1}^{k},\underl{\pi}_{h+1}^{k}(s_{h+1}^{k})) \mid \pi^{k,h'}\Big] 
        \\
        & \leq
        \bbEk \Big[\overl{Q}_{h+1}^{k}(s_{h+1}^{k},a_{h+1}^{k}) - \underl{Q}_{h+1}^{k}(s_{h+1}^{k},\underl{\pi}_{h+1}^{k}(s_{h+1}^{k})) \mid \pi^{k,h'}\Big]
        \\
        & =
        \bbEk \Big[\overl{Q}_{h+1}^{k}(s_{h+1}^{k},a_{h+1}^{k}) - \underl{Q}_{h+1}^{k}(s_{h+1}^{k},a_{h+1}^{k})\mid\pi^{k,h'}\Big],
    \end{align*}
    where the last equality follows because $h \ne h' - 1$.
\end{proof}

\begin{lemma}[Recursion with Non-Optimistic Next-Action]
    \label{lemma:rec-non-optimistic}
    Let $h' \in [H]$. Under the good event, for every $k \in [K]$ it holds that:
    \begin{align*}
    	  \bbEk \Bigl[ \overl{Q}_{h'-1}^{k}(s_{h'-1}^{k},a_{h'-1}^{k}) & - \underl{Q}_{h'-1}(s_{h'-1}^{k},a_{h'-1}^{k}) \mid \pi^{k,h'}\Bigr] 
    	  \le 
    	  \\
    	  & \le 
    	  \bbEk \left[ \frac{8 \sqrt{\logterm \VAR_{p_{h'-1}(\cdot \mid s^{k}_{h'-1},a^{k}_{h'-1})}(V^{\underl{\pi}^k}_{h'})}}{\sqrt{n^{k}_{h'-1}(s^{k}_{h'-1},a^{k}_{h'-1})\vee 1}} + \frac{118 H^2 S \logterm}{n^{k}_{h'-1}(s^{k}_{h'-1},a^{k}_{h'-1})\vee 1} \mid \pi^{k,h'} \right]
    	  \\
    	  & \qquad +
    	  \left( 1 + \frac{1}{4H} \right) \bbEk \Big[\overl{Q}_{h'}^{k}(s_{h'}^{k},a_{h'}^{k}) - \underl{Q}_{h'}^{k}(s_{h'}^{k},a_{h'}^{k}) \mid \pi^{k,h'}\Big]
    	  \\
    	  & \qquad +
    	  \left( 1 + \frac{1}{4H} \right) \bbEk \Big[\overl{Q}_{h'}^{k}(s_{h'}^{k},\underl{\pi}_{h'}^{k}(s_{h'}^{k})) - \underl{Q}_{h'}^{k}(s_{h'}^{k},\underl{\pi}_{h'}^{k}(s_{h'}^{k})) \mid \pi^{k,h'}\Big].
    \end{align*}
\end{lemma}

\begin{proof}
    Similarly to \cref{lemma:rec-optimistic}, we have,
    \begin{align*}
        \bbEk & \left[ \overl{Q}_{h'-1}^{k}(s_{h'-1}^{k},a_{h'-1}^{k}) - \underl{Q}_{h'-1}(s_{h'-1}^{k},a_{h'-1}^{k}) \mid \pi^{k,h'}\right] 
        \le
        \\
        & \qquad \le
        \bbEk \left[ \frac{8 \sqrt{\logterm \VAR_{p_{h'-1}(\cdot \mid s^{k}_{h'-1},a^{k}_{h'-1})}(V^{\underl{\pi}^k}_{h'})}}{\sqrt{n^{k}_{h'-1}(s^{k}_{h'-1},a^{k}_{h'-1})\vee 1}} + \frac{118 H^2 S \logterm}{n^{k}_{h'-1}(s^{k}_{h'-1},a^{k}_{h'-1})\vee 1} \mid \pi^{k,h'} \right]
        \\
        & \qquad \qquad +
        \left( 1 + \frac{1}{4H} \right) \bbEk \left[ \overl{Q}^k_{h'}(s^{k}_{h'},a^{k}_{h'}) - \underl{V}^k_{h'}(s^{k}_{h'}) \mid \pi^{k,h'} \right].
    \end{align*}
    Now, by \cref{lemma:active-arms},
    \begin{align*}
        \bbEk \Big[\overl{Q}_{h'}^{k}(s_{h'}^{k},a_{h'}^{k}) - \underl{V}_{h'}^{k}(s_{h'}^{k}) \mid \pi^{k,h'}\Big]
        & =
        \bbEk \Big[\overl{Q}_{h'}^{k}(s_{h'}^{k},a_{h'}^{k}) - \underl{Q}_{h'}^{k}(s_{h'}^{k},\underl{\pi}_{h'}^{k}(s_{h'}^{k})) \mid \pi^{k,h'}\Big]
        \\
        & \leq
        \bbEk \Big[\overl{Q}_{h'}^{k}(s_{h'}^{k},a_{h'}^{k}) - \underl{Q}_{h'}^{k}(s_{h'}^{k},a_{h'}^{k}) \mid \pi^{k,h'}\Big]
        \\
        & \qquad +
        \bbEk \Big[\overl{Q}_{h'}^{k}(s_{h'}^{k},\underl{\pi}_{h'}^{k}(s_{h'}^{k})) - \underl{Q}_{h'}^{k}(s_{h'}^{k},\underl{\pi}_{h'}^{k}(s_{h'}^{k})) \mid \pi^{k,h'} \Big]. \qedhere
    \end{align*}
\end{proof}

\begin{lemma}[Exploration Penalty Term Recursion]
    \label{lemma:rec-exp}
    Let $h' \in [H]$. Under the good event, for every $k \in [K]$ it holds that:
    \begin{align*}
         \bbEk \Bigl[ \overl Q_{h'}^{k}(s_{h'}^{k},\underl{\pi}_{h'}^{k}(s_{h'}^{k})) & - \underl Q_{h'}(s_{h'}^{k},\underl{\pi}_{h'}^{k}(s_{h'}^{k})) \mid \pi^{k,h'} \Bigr] 
        \le
        \\
        & \le
        \bbEk \left[ \frac{8 \sqrt{\logterm \VAR_{p_{h'}(\cdot \mid s^{k}_{h'},a^{k}_{h'})}(V^{\underl{\pi}^k}_{h'+1})}}{\sqrt{n^{k}_{h'}(s^{k}_{h'},a^{k}_{h'})\vee 1}} + \frac{118 H^2 S \logterm}{n^{k}_{h'}(s^{k}_{h'},a^{k}_{h'})\vee 1} \mid \underl{\pi}^{k} \right]  
        \\
        & \qquad + 
        \left( 1 + \frac{1}{4H} \right) \bbEk \left[ \overl Q_{h'+1}^{k}(s_{h'+1}^{k},a_{h'+1}^{k}) - \underl Q_{h'+1}^{k}(s_{h'+1}^{k},a_{h'+1}^{k}) \mid \underl{\pi}^{k} \right].
    \end{align*}
\end{lemma}

\begin{proof}
    Again, similar to \cref{lemma:rec-optimistic},
    \begin{align*}
        \bbEk \Bigl[ \overl Q_{h'}^{k}(s_{h'}^{k},\underl{\pi}_{h'}^{k}(s_{h'}^{k})) & - \underl Q_{h'}(s_{h'}^{k},\underl{\pi}_{h'}^{k}(s_{h'}^{k})) \mid \pi^{k,h'} \Bigr] \le
        \\
        & \le
        \bbEk \left[ \frac{8 \sqrt{\logterm \VAR_{p_{h'}(\cdot \mid s^{k}_{h'},\underl{\pi}_{h'}^{k}(s_{h'}^{k}))}(V^{\underl{\pi}^k}_{h'+1})}}{\sqrt{n^{k}_{h'}(s^{k}_{h'},\underl{\pi}_{h'}^{k}(s_{h'}^{k}))\vee 1}} + \frac{118 H^2 S \logterm}{n^{k}_{h'}(s^{k}_{h'},\underl{\pi}_{h'}^{k}(s_{h'}^{k}))\vee 1} \mid \pi^{k,h'} \right]
        \\
        & \qquad +
        \left( 1 + \frac{1}{4H} \right) \bbEk \left[ \bbE_{p_{h'}(\cdot \mid  s^{k}_{h'},\underl{\pi}_{h'}^{k}(s_{h'}^{k}))} [ \overl{V}^k_{h+1} - \underl{V}^k_{h+1}] \mid \pi^{k,h'} \right]
    \end{align*}
    Note that $q_{h'}^{\pi^{k,h'}}(s) = q_{h'}^{\underl{\pi}^{k}}(s)$ since until step $h'$ the policies are the same, i.e., $\pi^{h',k}_h = \underl{\pi}^{k}_h$ for all $h<h'$.  Hence, denoting the first term above by $\bbEk [ x(s^{k}_{h'},\underl{\pi}_{h'}^{k}(s_{h'}^{k})) \mid \pi^{k,h'} ]$, we can write
    \begin{align*}
        \bbEk [ x(s^{k}_{h'},\underl{\pi}_{h'}^{k}(s_{h'}^{k})) \mid \pi^{k,h'} ]
        & =
        \sum_{s} q_{h'}^{\pi^{k,h'}}(s) x(s,\underl{\pi}_{h'}^{k}(s))
        = 
        \sum_{s} q_{h'}^{\underl{\pi}^{k}}(s) x(s,\underl{\pi}_{h'}^{k}(s))
        \\
        & =
        \sum_{s,a} q_{h'}^{\underl{\pi}^{k}}(s,a) x(s,a)
        = 
        \bbEk [ x(s^k_{h'},a^k_{h'}) \mid \underl{\pi}^{k} ],
    \end{align*}
    where the third equality is since $\underl{\pi}^{k}$ is deterministic. In a similar way, the second term can bounded by,
    \begin{align*}
    	  \bbEk \Bigl[ \bbE_{p_{h'}(\cdot \mid  s^{k}_{h'},\underl{\pi}_{h'}^{k}(s_{h'}^{k}))} & [ \overl{V}^k_{h+1} - \underl{V}^k_{h+1}] \mid \pi^{k,h'} \Bigr] =
    	  \\
    	  & =\sum_{s,s'}q_{h'}^{\pi^{k,h'}}(s)p_{h'}(\cdot \mid  s_{h}^{k},\underl{\pi}_{h'}^{k}(s_{h'}^{k})) \left(\overl V_{h'+1}^{k}(s') - \underl V_{h'+1}^{k}(s')\right)
    	  \\
    	  & =\sum_{s,s'}q_{h'}^{\underl{\pi}^{k}}(s)p_{h'}(\cdot \mid  s_{h}^{k},\underl{\pi}_{h'}^{k}(s_{h'}^{k})) \left(\overl V_{h'+1}^{k}(s') - \underl V_{h'+1}^{k}(s')\right) 
    	  \\
    	  & =\sum_{s}q_{h'+1}^{\underl{\pi}^{k}}(s) \left(\overl V_{h'+1}^{k}(s') - \underl V_{h'+1}^{k}(s')\right)
    	  \\
    	  & = 
    	  \bbEk \left[ \overl V_{h'+1}^{k}(s_{h'+1}^{k}) - \underl V_{h'+1}^{k}(s_{h'+1}^{k}) \mid \underl{\pi}^{k} \right] 
    	  \\
    	  & \leq 
    	  \bbEk \left[ \overl Q_{h'+1}^{k}(s_{h'+1}^{k},a_{h'+1}^{k}) - \underl V_{h'+1}^{k}(s_{h'+1}^{k}) \mid \underl{\pi}^{k} \right] 
    	  \\
    	  & = 
    	  \bbEk \left[ \overl Q_{h'+1}^{k}(s_{h'+1}^{k},a_{h'+1}^{k}) - \underl Q_{h'+1}^{k}(s_{h'+1}^{k},a_{h'+1}^{k}) \mid \underl{\pi}^{k} \right]. \qedhere
    \end{align*}
\end{proof}

\begin{lemma}[Recursion Optimistic Policy]
    \label{lemma:rec-optimistic-policy}
    Let $h \in [H]$. Under the good event, for every $k \in [K]$ it holds that:
    \begin{align*}
       \bbEk \left[\overl{Q}_{h}^{k}(s_{h}^{k},a_{h}^{k}) - \underl{Q}_{h}(s_{h}^{k},a_{h}^{k})\mid\underl{\pi}^k\right]
       & \le
       \bbEk \left[ \frac{8 \sqrt{\logterm \VAR_{p_{h}(\cdot \mid s^{k}_{h},a^{k}_{h})}(V^{\underl{\pi}^k}_{h+1})}}{\sqrt{n^{k}_h(s^{k}_{h},a^{k}_{h})\vee 1}} + \frac{118 H^2 S \logterm}{n^{k}_h(s^{k}_h,a^{k}_h)\vee 1} \mid \underl{\pi}^k \right]
       \\
       & \qquad +
       \left( 1 + \frac{1}{4H} \right) \bbEk \Big[\overl{Q}_{h+1}^{k}(s_{h+1}^{k},a_{h+1}^{k}) - \underl{Q}_{h+1}^{k}(s_{h+1}^{k},a_{h+1}^{k}) \mid \underl{\pi}^k \Big].
    \end{align*}
\end{lemma}

\begin{proof}
    Similar to the proof of \cref{lemma:rec-exp}.
\end{proof}

\newpage

\section{The \texttt{coop-O-REPS} algorithm for adversarial MDPs with fresh randomness and known $p$}
\label{appendix:adversarial-fresh-known}

\begin{algorithm}[t]
    \caption{\textsc{Cooperative O-REPS (coop-O-REPS)}} 
    \label{alg:coop-o-reps}
    \begin{algorithmic}[1]
        \STATE {\bf input:} state space $\calS$, action space $\calA$, horizon $H$, transition function $p$, number of episodes $K$, number of agents $\numagents$, exploration parameter $\gamma$, learning rate $\eta$.

        \STATE {\bf initialize:} $\pi^1_h(a \mid s) = \nicefrac{1}{A} , q^1_h(s,a) = q^{\pi^1}_h(s,a) \ \forall (s,a,h) \in \calS \times \calA \times [H]$. 

        \FOR{$k=1,\dots,K$}
    
            \FOR{$v=1,\dots,\numagents$}
            
                \STATE observe initial state $s^{k,v}_1$.
            
                \FOR{$h=1,\dots,H$}
        
                    \STATE pick action $a^{k,v}_h \sim \pi^k_h(\cdot \mid s^{k,v}_h)$.
                    
                    \STATE suffer and observe cost $c^k_h(s^{k,v}_h,a^{k,v}_h)$.
                    
                    \STATE observe next state $s^{k,v}_{h+1} \sim p_h(\cdot \mid s^{k,v}_h,a^{k,v}_h)$.
        
                \ENDFOR
                \ENDFOR
                
                \STATE compute $W^k_h(s,a) = \Pr [\exists v: \  s^{k,v}_h = s,a^{k,v}_h = a \mid \pi^k] = 1 - \left( 1 - q^k_h(s,a) \right)^\numagents \  \forall (s,a,h) \in \calS \times \calA \times [H]$.
                
                \STATE compute $\hat c^k_h(s,a) = \frac{c^k_h(s,a) \indevent{\exists v: \  s^{k,v}_h = s,a^{k,v}_h = a}}{W^k_h(s,a) + \gamma} \  \forall (s,a,h) \in \calS \times \calA \times [H]$.
                
                \STATE compute $q^{k+1} = \argmin_{q \in \ocset} \eta \langle q , \hat c^k \rangle + \KL{q}{q^k}$.
                
                \STATE compute $\pi^{k+1}_h(a \mid s) = \frac{q^k_h(s,a)}{\sum_{a' \in \calA} q^k_h(s,a')} \  \forall (s,a,h) \in \calS \times \calA \times [H]$.
        \ENDFOR
    \end{algorithmic}
\end{algorithm}

For the setting of adversarial MDPs with fresh randomness and known transitions we propose the Cooperative O-REPS algorithm (\verb|coop-O-REPS|; see \cref{alg:coop-o-reps}).
The idea is simple: all the agents run the same O-REPS algorithm, but the estimated costs are updated based on the trajectories of all of them.
Since the randomness is fresh in this setting, we expect the agents to observe $\numagents$ times more information.
Next, we prove the following optimal regret bound for \verb|coop-O-REPS|.

Similarly to \citet{zimin2013online}, We use the notations $\ocset$ and $\KL{\cdot}{\cdot}$ for the set of occupancy measures of the MDP $\calM$ and the KL-divergence between occupancy measures, respectively.

\begin{theorem}
    \label{thm:reg-coop-o-reps}
    With probability $1 - \delta$, setting $\eta = \gamma = \sqrt{\frac{\log \frac{HSA}{\delta}}{\left( 1 + \frac{S A}{\numagents} \right) K}}$, the individual regret of each agent of \verb|coop-O-REPS| is
    \[
        \regret
        =
        O \left( H \sqrt{K \log \frac{H S A}{\delta}} + \sqrt{\frac{H^2 S A K}{\numagents} \log \frac{H S A}{\delta}} + \frac{H S A}{\numagents} \log \frac{H S A}{\delta} \right).
    \]
\end{theorem}

\subsection{The good event}

Define the following events: 
\begin{align*}
    E^c 
    & = 
    \left\{ \sum_{k=1}^K \langle \bbEk [ \hat c^k ] - \hat c^k , q^k \rangle \le 4 H \sqrt{K \log \frac{3}{\delta}} \right\}
    \\
    E^{\hat c}
    & =
    \left\{ \sum_{k=1}^K \sum_{h=1}^H \sum_{s \in \calS} \sum_{a \in \calA} \left( \frac{1}{\numagents} + q^k_h(s,a) \right) \left( \hat c^k_h(s,a) - 2 c^k_h(s,a) \right) \le \frac{10 H S A \log \frac{3 H S A}{\delta}}{\numagents \gamma} + \frac{10 H \log \frac{3 H S A}{\delta}}{ \gamma} \right\}
    \\
    E^\star
    & =
    \left\{ \sum_{k=1}^K \langle \hat c^k - c^k , q^{\pi^\star} \rangle \le \frac{H  \log \frac{3 H S A}{\delta}}{\gamma} \right\}
\end{align*}

The good event is the intersection of the above events. 
The following lemma establishes that the good event holds with high probability. 

\begin{lemma}[The Good Event]
    \label{lemma:good-event-adversarial-known-p-fresh}
    Let $\bbG =E^c \cap E^{\hat c} \cap E^\star$ be the good event. 
    It holds that $\Pr [ \bbG ] \geq 1-\delta$.
\end{lemma}

\begin{proof}
     We show that each of the events $\neg E^c, \neg E^{\hat c}, \neg E^\star$ occur with probability at most $\delta / 3$. Then, by a union bound we obtain the statement.
    Notice that:
    \begin{align*}
        \langle \hat c^k , q^k \rangle
        & =
        \sum_{h=1}^H \sum_{s \in \calS} \sum_{a \in \calA} q^k_h(s,a) \hat c^k_h(s,a)
        \le
        \sum_{h=1}^H \sum_{s \in \calS} \sum_{a \in \calA} q^k_h(s,a) \frac{\indevent{\exists v : s^{k,v}_h = s , a^{k,v}_h = a}}{W^k_h(s,a) + \gamma}
        \\
        & \le
        \sum_{h=1}^H \sum_{s \in \calS} \sum_{a \in \calA} \frac{q^k_h(s,a)}{1 - \left( 1 - q^k_h(s,a) \right)^\numagents} \indevent{\exists v : s^{k,v}_h = s , a^{k,v}_h = a}
        \\
        & \le
        \sum_{h=1}^H \sum_{s \in \calS} \sum_{a \in \calA} \left( \frac{1}{\numagents} + q^k_h(s,a) \right) \indevent{\exists v : s^{k,v}_h = s , a^{k,v}_h = a}
        \\
        & \le
        \sum_{h=1}^H \sum_{s \in \calS} \sum_{a \in \calA} q^k_h(s,a) + \frac{1}{\numagents} \sum_{h=1}^H \sum_{s \in \calS} \sum_{a \in \calA} \indevent{\exists v : s^{k,v}_h = s , a^{k,v}_h = a}
        \le
        H + \frac{1}{\numagents} \cdot H \numagents = 2H,
    \end{align*}
    where the third inequality is by \cref{lemma:linear-approx}, and the last inequality follows because for each step $h$ the agents visit at most $\numagents$ state-action pairs.
    Thus, event $E^c$ holds by Azuma inequality.
    
    Event $E^{\hat c}$ holds by \citet[Lemma E.2]{cohen2021minimax} since  $\sum_{h,s,a} \left( \frac{1}{\numagents} + q^k_h(s,a) \right) \hat c^k_h(s,a) \le \frac{1}{\gamma} \left( \frac{HSA}{\numagents} + H \right)$ and $\bbEk[\hat c^k_h(s,a)] \le c^k_h(s,a)$.
    Event $E^\star$ holds by \citet[Lemma 14]{jin2020learning}.
\end{proof}

\subsection{Proof of Theorem~\ref{thm:reg-coop-o-reps}}

\begin{proof}[Proof of Theorem~\ref{thm:reg-coop-o-reps}]
    By \cref{lemma:good-event-adversarial-known-p-fresh}, the good event holds with probability $1 - \delta$.
    We now analyze the regret under the assumption that the good event holds.
    We start by decomposing the regret as follows:
    \begin{align*}
        \regret
        & =
        \sum_{k=1}^K V^{k,\pi^k}_1(s_{1}^{k,v}) - V^{k,\pi^\star}_1(s_{1}^{k,v})
        =
        \sum_{k=1}^K \langle c^k , q^k - q^{\pi^\star} \rangle
        \\
        & =
        \underbrace{\sum_{k=1}^K \langle c^k - \hat c^k , q^k \rangle}_{(A)} + \underbrace{\sum_{k=1}^K \langle \hat c^k , q^k - q^{\pi^\star} \rangle}_{(B)} + \underbrace{\sum_{k=1}^K \langle \hat c^k - c^k , q^{\pi^\star} \rangle}_{(C)}.
    \end{align*}
    
    Term $(A)$ can be further decomposed as:
    \begin{align*}
        (A)
        =
        \sum_{k=1}^K \langle c^k - \hat c^k , q^k \rangle
        =
        \sum_{k=1}^K \langle c^k - \bbEk [ \hat c^k ] , q^k \rangle + \sum_{k=1}^K \langle \bbEk [ \hat c^k ] - \hat c^k , q^k \rangle.
    \end{align*}
    The second term is bounded by $4 H \sqrt{K \log \frac{3}{\delta}}$ by the good event $E^c$, and for the first term:
    \begin{align*}
        \sum_{k=1}^K \langle c^k - \bbEk [ \hat c^k ] , q^k \rangle
        & =
        \sum_{k=1}^K \sum_{h=1}^H \sum_{s \in \calS} \sum_{a \in \calA} q^k_h(s,a) c^k_h(s,a) \left( 1 - \frac{\bbEk [\indevent{\exists v: \  s^{k,v}_h = s,a^{k,v}_h = a}]}{W^k_h(s,a) + \gamma} \right)
        \\
        & =
        \sum_{k=1}^K \sum_{h=1}^H \sum_{s \in \calS} \sum_{a \in \calA} q^k_h(s,a) c^k_h(s,a) \left( 1 - \frac{W^k_h(s,a)}{W^k_h(s,a) + \gamma} \right)
        \le
        \gamma \sum_{k=1}^K \sum_{h=1}^H \sum_{s \in \calS} \sum_{a \in \calA} \frac{q^k_h(s,a)}{W^k_h(s,a) + \gamma}
        \\
        & \le
        \gamma \sum_{k=1}^K \sum_{h=1}^H \sum_{s \in \calS} \sum_{a \in \calA} \frac{q^k_h(s,a)}{1 - \left( 1 - q^k_h(s,a) \right)^\numagents}
        \le
        \gamma \sum_{k,h,s,a} \left( \frac{1}{\numagents} + q^k_h(s,a) \right)
        =
        \gamma H K \left( 1 + \frac{S A}{\numagents} \right),
    \end{align*}
    where the last inequality is by \cref{lemma:linear-approx}.
    
    Term $(B)$ is bounded by OMD (see, e.g., \citet{zimin2013online}) as follows:
    \begin{align*}
        (B)
        & =
        \sum_{k=1}^K \langle \hat c^k , q^k - q^{\pi^\star} \rangle
        \le
        \frac{H \log (H S A)}{\eta} + \eta \sum_{k=1}^K \sum_{h=1}^H \sum_{s \in \calS} \sum_{a \in \calA} q^k_h(s,a) \hat c^k_h(s,a)^2
        \\
        & \le
        \frac{H \log (H S A)}{\eta} + \eta \sum_{k=1}^K \sum_{h=1}^H \sum_{s \in \calS} \sum_{a \in \calA} q^k_h(s,a) \frac{\hat c^k_h(s,a)}{W^k_h(s,a) + \gamma}
        \\
        & \le
        \frac{H \log (H S A)}{\eta} + \eta \sum_{k=1}^K \sum_{h=1}^H \sum_{s \in \calS} \sum_{a \in \calA} \hat c^k_h(s,a) \frac{q^k_h(s,a)}{1 - \left( 1 - q^k_h(s,a) \right)^\numagents}
        \\
        & \le
        \frac{H \log (H S A)}{\eta} + \eta \sum_{k=1}^K \sum_{h=1}^H \sum_{s \in \calS} \sum_{a \in \calA} \left( \frac{1}{\numagents} + q^k_h(s,a) \right) \hat c^k_h(s,a)
        \\
        & \le
        \frac{H \log (H S A)}{\eta} + 2 \eta \sum_{k=1}^K \sum_{h=1}^H \sum_{s \in \calS} \sum_{a \in \calA} \left( \frac{1}{\numagents} + q^k_h(s,a) \right) c^k_h(s,a) + \frac{10 \eta H S A \log \frac{3 HSA}{\delta}}{\numagents \gamma} + \frac{10 \eta H \log \frac{3 HSA}{\delta}}{\gamma}
        \\
        & \lesssim
        \frac{H \log \frac{H S A}{\delta}}{\eta} + \frac{\eta H S A K}{\numagents} + \eta H K + \frac{\eta H S A \log \frac{3 HSA}{\delta}}{\numagents \gamma} + \frac{\eta H \log \frac{3 HSA}{\delta}}{\gamma},
    \end{align*}
    where the forth inequality is by \cref{lemma:linear-approx}, and the fifth inequality is by the good event $E^{\hat c}$.
    
    Term $(C)$ is bounded by $\frac{H \log \frac{3 H}{\delta}}{\gamma}$ by the good event $E^\star$.
    Putting the three terms together gives the final regret bound when setting $\eta = \gamma = \sqrt{\frac{\log \frac{HSA}{\delta}}{\left( 1 + \frac{S A}{\numagents} \right) K}}$.
\end{proof}

\subsection{Auxiliary lemmas}

\begin{lemma}   \label{lemma:linear-approx}
    Let $x\in(0,1)$. Then, $\frac{x}{1-(1-x)^{m}}\leq\frac{1}{m}+x$.
\end{lemma}

\begin{proof}
    Using AM-GM inequaility,
    \begin{align*}
        & \left((1-x)^{m}(1+xm)\right)^{\frac{1}{m+1}}  \leq    \frac{m(1-x)+1+xm}{m+1}=1\\
        \Longrightarrow & (1-x)^{m}                     \leq    \frac{1}{1+xm}\\
        \Longrightarrow & 1-(1-x)^{m}                   \geq    \frac{xm}{1+xm}\\
        \Longrightarrow & \frac{x}{1-(1-x)^{m}}         \leq    \frac{1}{m}+x.
    \end{align*}
\end{proof}

\newpage

\section{The \texttt{coop-UOB-REPS} algorithm for adversarial MDPs with fresh randomness and unknown $p$}
\label{appendix:adversarial-fresh-unknown}

\begin{algorithm}[t]
    \caption{\textsc{Cooperative UOB-REPS (coop-UOB-REPS)}} 
    \label{alg:coop-uob-reps}
    \begin{algorithmic}[1]
        \STATE {\bf input:} state space $\calS$, action space $\calA$, horizon $H$, confidence parameter $\delta$, number of episodes $K$, number of agents $\numagents$, exploration parameter $\gamma$, learning rate $\eta$.

        \STATE {\bf initialize:} $n^1_h(s,a)=0,n^1_h(s,a,s')=0,\pi^1_h(a \mid s) = \nicefrac{1}{A} , q^1_h(s,a,s') = \nicefrac{1}{S^2 A} \ \forall (s,a,s',h) \in \calS \times \calA \times \calS \times [H]$. 

        \FOR{$k=1,\dots,K$}
        
            \STATE set $n^{k+1}_h(s,a) \gets n^{k}_h(s,a),n^{k+1}_h(s,a,s') \gets n^{k}_h(s,a,s') \ \forall (s,a,s',h) \in \calS \times \calA \times \calS \times [H]$.
    
            \FOR{$v=1,\dots,\numagents$}
            
                \STATE observe initial state $s^{k,v}_1$.
            
                \FOR{$h=1,\dots,H$} 
        
                    \STATE pick action $a^{k,v}_h \sim \pi^k_h(\cdot \mid s^{k,v}_h)$.
                    
                    \STATE suffer and observe cost $c^k_h(s^{k,v}_h,a^{k,v}_h)$.
                    
                    \STATE observe next state $s^{k,v}_{h+1} \sim p_h(\cdot \mid s^{k,v}_h,a^{k,v}_h)$.
                    
                    \STATE update $n^{k+1}_h(s^{k,v}_h,a^{k,v}_h) \gets n^{k+1}_h(s^{k,v}_h,a^{k,v}_h) + 1 , n^{k+1}_h(s^{k,v}_h,a^{k,v}_h,s^{k,v}_{h+1})\gets n^{k+1}_h(s^{k,v}_h,a^{k,v}_h,s^{k,v}_{h+1}) + 1$.
        
                \ENDFOR
                \ENDFOR
                
                \STATE set $\hat p^{k+1}_h(s' \mid s,a) \gets \frac{n^{k+1}_h(s,a,s')}{n^{k+1}_h(s,a)\vee 1}  \ \forall (s,a,s',h) \in \calS \times \calA \times \calS \times [H]$.
                
                \STATE compute confidence set for $\epsilon^{k+1}_h(s' \mid s,a) = 4\sqrt{ \frac{\hat{p}_{h}^{k+1}(s'| s,a) \ln\frac{HSAK}{4\delta}}{n_{h}^{k+1}(s,a) \vee 1}} + 10 \frac{\ln\frac{HSAK}{4\delta}}{n_{h}^{k+1}(s,a) \vee 1}$:
                \[
                    \mathcal{P}^{k+1}
                    =
                    \left\{ p' \mid \forall (s,a,s',h) : |\hat p^{k+1}_h(s' \mid s,a) - p'_h(s' \mid s,a)| \le \epsilon^{k+1}_h(s' \mid s,a) \right\}.
                \]
                
                \STATE compute $u^k_h(s) = \max_{p' \in \mathcal{P}^{k}} q^{p',\pi^k}_h(s) = \max_{p' \in \mathcal{P}^{k}} \Pr[s_h=s \mid \pi^k,p'] \ \forall s \in \calS$.
                
                \STATE compute $U^k_h(s,a)  = 1 - \left( 1 - u^k_h(s,a) \right)^\numagents \  \forall (s,a,h) \in \calS \times \calA \times [H]$.
                
                \STATE compute $\hat c^k_h(s,a) = \frac{c^k_h(s,a) \indevent{\exists v: \  s^{k,v}_h = s,a^{k,v}_h = a}}{U^k_h(s,a) + \gamma} \  \forall (s,a,h) \in \calS \times \calA \times [H]$.
                
                \STATE compute $q^{k+1} = \argmin_{q \in \Delta(\calM,k+1)} \eta \langle q , \hat c^k \rangle + \KL{q}{q^k}$.
                
                \STATE compute $\pi^{k+1}_h(a \mid s) = \frac{q^{k+1}_h(s,a)}{\sum_{a' \in \calA} q^{k+1}_h(s,a')} \  \forall (s,a,h) \in \calS \times \calA \times [H]$, where $q^{k+1}_h(s,a) = \sum_{s' \in \calS} q^{k+1}_h(s,a,s')$.
        \ENDFOR
    \end{algorithmic}
\end{algorithm}

For the setting of adversarial MDPs with fresh randomness and unknown transitions we propose the Cooperative UOB-REPS algorithm (\verb|coop-UOB-REPS|; see \cref{alg:coop-uob-reps}).
The idea is simple: all the agents run the same UOB-REPS algorithm, but the estimated costs and transitions are updated based on the trajectories of all of them.
Since the randomness is fresh in this setting, we expect the agents to observe $\numagents$ times more information.
Next, we prove the following regret bound for \verb|coop-UOB-REPS|.
Note that this bound is optimal up to a $\sqrt{HS}$ factor.
Removing this extra $\sqrt{HS}$ is an open-problem even for adversarial MDPs with a single agent.

Similarly to \citet{rosenberg2019online}, We use the notation $\ocsetk$ for the set of occupancy measures whose induced transition function is within the confidence set $\mathcal{P}^k$.

\begin{theorem}
    \label{thm:reg-coop-uob-reps}
    With probability $1 - \delta$, setting $\eta = \gamma = \sqrt{\frac{\log \frac{mKHSA}{\delta}}{\left( 1 + \frac{S A}{\numagents} \right) K}}$, the individual regret of each agent of \verb|coop-UOB-REPS| is
    \[
        \regret
        =
        O \left( H \sqrt{K \log \frac{\numagents K H S A}{\delta}} + \sqrt{\frac{H^4 S^2 A K}{\numagents} \log^3  \frac{\numagents K  H S A}{\delta}} + H^3 S^3 A \log^3 \frac{\numagents K  H S A}{\delta} \right).
    \]
\end{theorem}

\subsection{The good event}

Denote $\epsilon^k_h( s' \mid s,a) = \sqrt{\frac{2 \hat p^k_h(s'|s,a) \log \frac{30 K H S A}{\delta}}{n^{k}_h(s,a)\vee 1}} + \frac{2 \log \frac{30 K H S A}{\delta}}{n^{k}_h(s,a) \vee 1}$,  $\epsilon^k_h(s,a) = \sum_{s'\in\calS} \epsilon^k_h( s' \mid  s,a)$,  $\tilde\epsilon^k_h( s' \mid s,a) = 8\sqrt{\frac{ p_h(s'|s,a) \log \frac{30 K H S A}{\delta}}{n^{k}_h(s,a)\vee 1}} + \frac{100 \log \frac{30 K H S A}{\delta}}{n^{k}_h(s,a) \vee 1}$, and $\tilde\epsilon^k_h(s,a) = \sum_{s'}\tilde\epsilon^k_h(s'\mid s,a)$.
Define the following events: 
\begin{align*}
    E^p 
    & = 
    \left\{ \forall (k,s,a,s',h):\ |p_h (s'|s,a) - \hat{p}^{k}_h (s'|s,a)| \le \epsilon^k_h( s' \mid s,a) \right\}
    \\
    E^{on1}
    & =
    \left\{ \forall v \in [\numagents] : \  \sum_{k,h,s,a} \Big( q^{\pi^k}_h(s,a) - \indevent{s^{k,v}_h=s,a^{k,v}_h=a} \Big) \min \{ 2 , \epsilon^k_h(s,a) \} \le 10 \sqrt{ K \log \frac{30 K H S A \numagents}{\delta}} \right\}
    \\
    E^{on2}
    & =
    \left\{ \forall v\in[\numagents]:\ \sum_{k,h,s,a}  q_{h}^{\pi^{k}}(s,a) \tilde\epsilon_{h}^{k}(s,a) \le 2\sum_{k,h,s,a} \indevent{s_{h}^{k,v}=s,a_{h}^{k,v}=a} \tilde\epsilon_{h}^{k}(s,a)+100 H S \log^2 \frac{30 K H S A \numagents}{\delta} \right\} 
    \\
    E^{on3}
    & =
    \left\{ \forall v\in[\numagents]:\ \sum_{k,s,a,h}  \frac{q_{h}^{\pi^{k}}(s,a)}{n_h^k(s,a)}  \le 2\sum_{k,s,a,h} \frac{ \indevent{s_{h}^{k,v}=s,a_{h}^{k,v}=a}}{n_h^k(s,a)} + H \log \frac{m}{\delta} \right\} 
    \\
    E^c 
    & = 
    \left\{ \sum_{k=1}^K \langle \bbEk [ \hat c^k ] - \hat c^k , q^k \rangle \le 4 H \sqrt{K \log \frac{6}{\delta}} \right\}
    \\
    E^{\hat c}
    & =
    \left\{ \sum_{k=1}^K \sum_{h=1}^H \sum_{s \in \calS} \sum_{a \in \calA} \left( \frac{1}{\numagents} + q^k_h(s,a) \right) \left( \hat c^k_h(s,a) - c^k_h(s,a) \right) \le \frac{H S A \log \frac{3 H S A}{\delta}}{\gamma} \right\}
    \\
    E^\star
    & =
    \left\{ \sum_{k=1}^K \langle \hat c^k - c^k , q^{\pi^\star} \rangle \le \frac{H \log \frac{3 H S A}{\delta}}{\gamma} \right\}
\end{align*}

The good event is the intersection of the above events. 
The following lemma establishes that the good event holds with high probability. 

\begin{lemma}[The Good Event]
    \label{lemma:good-event-adversarial-unknown-p-fresh}
    Let $\bbG =E^p \cap E^{on1} \cap E^{on2} \cap E^{on3}  \cap E^c \cap E^{\hat c} \cap E^\star$ be the good event. 
    It holds that $\Pr [ \bbG ] \geq 1-\delta$.
\end{lemma}

\begin{proof}
    $E^{on2}$ and $E^{on3}$ follows \citet[Lemma E.2]{cohen2021minimax}. 
    The rest are similar to the proofs of \cref{lemma:good-event-adversarial-known-p-fresh,lem:ULCVI-good-event} and to proofs in \citet{jin2020learning}.
\end{proof}

\subsection{Proof of Theorem~\ref{thm:reg-coop-uob-reps}}

\begin{proof}[Proof of Theorem~\ref{thm:reg-coop-uob-reps}]
    By \cref{lemma:good-event-adversarial-unknown-p-fresh}, the good event holds with probability $1 - \delta$.
    We now analyze the regret under the assumption that the good event holds.
    We start by decomposing the regret as follows:
    \begin{align*}
        \regret
        & =
        \sum_{k=1}^K V^{k,\pi^k}_1(s_{1}^{k,v}) - V^{k,\pi^\star}_1(s_{1}^{k,v})
        =
        \sum_{k=1}^K \langle c^k , q^{\pi^k} - q^{\pi^\star} \rangle
        \\
        & =
        \underbrace{\sum_{k=1}^K \langle c^k , q^{\pi^k} - q^k \rangle}_{(A)} + \underbrace{\sum_{k=1}^K \langle c^k - \hat c^k , q^k \rangle}_{(B)} + \underbrace{\sum_{k=1}^K \langle \hat c^k , q^k - q^{\pi^\star} \rangle}_{(C)} + \underbrace{\sum_{k=1}^K \langle \hat c^k - c^k , q^{\pi^\star} \rangle}_{(D)}.
    \end{align*}
    
    Let $\logterm = \log \frac{KHSA\numagents}{\delta}$ be a logarithmic term.
    Term $(A)$ can be decomposed using the value difference lemma (see, e.g., \citet{shani2020optimistic}):
    \begin{align*}
        (A)
        & =
        \sum_{k=1}^K \langle c^k , q^{\pi^k} - q^k \rangle
        \le
        2H \sum_{k=1}^K \sumhsa q^{\pi^k}_h(s,a) \lVert p_h(\cdot \mid s,a) - \hat p^k_h(\cdot \mid s,a) \rVert_1
        \\
        & \le
        2H \sum_{k=1}^K \sumhsa q^{\pi^k}_h(s,a) \min \{ 2 , \epsilon^k_h(s,a) \}
        \\
        & \lesssim
        \frac{H}{\numagents} \sum_{k=1}^K \sumhsa \sum_{v=1}^\numagents \indevent{s^{k,v}_h=s,a^{k,v}_h=a} \epsilon^k_h(s,a) + H \sqrt{K \logterm}
        \\
        & \lesssim
        \frac{H \sqrt{S \logterm}}{\numagents} \sum_{k,h,s,a} \frac{\sum_{v=1}^\numagents \indevent{s^{k,v}_h=s,a^{k,v}_h=a} }{\sqrt{n^k_h(s,a) \vee 1}} + \frac{H S \logterm}{\numagents} \sum_{k,h,s,a} \frac{\sum_{v=1}^\numagents \indevent{s^{k,v}_h=s,a^{k,v}_h=a} }{n^k_h(s,a) \vee 1}
        +H\sqrt{K \logterm},
    \end{align*}
    where the second inequality is by event $E^p$, and the third inequality uses event $E^{on1}$ and Cauchy–Schwarz inequality.
    We now bound each of the two sums separately.
    For the second sum recall that $n^k_h(s,a) = \sum_{j=1}^{k-1} \sum_{v=1}^\numagents \indevent{s^{j,v}_h=s,a^{j,v}_h=a}$, thus we have:
    \begin{align}
    \nonumber
        \sum_{k,h,s,a} \frac{\sum_{v=1}^\numagents \indevent{s^{k,v}_h=s,a^{k,v}_h=a} }{n^k_h(s,a) \vee 1}
        & \le
        2 H S A \numagents + \sum_{h,s,a} \sum_{k : n^k_h(s,a) \ge \numagents} \frac{\sum_{v=1}^\numagents \indevent{s^{k,v}_h=s,a^{k,v}_h=a} }{n^k_h(s,a)}
        \\
        \nonumber
        & =
        2 H S A \numagents + \sum_{h,s,a} \sum_{k : n^k_h(s,a) \ge \numagents} \frac{\sum_{v=1}^\numagents \indevent{s^{k,v}_h=s,a^{k,v}_h=a} }{\sum_{j=1}^{k-1} \sum_{v=1}^\numagents \indevent{s^{j,v}_h=s,a^{j,v}_h=a}}
        \\
        & \le
        2 H S A \numagents + H S A \log (K \numagents),
        \label{eq:adv shrink 1/n}
    \end{align}
    where the last inequality is by \citet[Lemma B.18]{rosenberg2020near}.
    For the first term:
    \begin{align}
    \nonumber
        \sum_{k,h,s,a} \frac{\sum_{v=1}^\numagents \indevent{s^{k,v}_h=s,a^{k,v}_h=a} }{\sqrt{n^k_h(s,a) \vee 1}}
        & \le
        2 H S A \numagents + \sum_{h,s,a} \sum_{k : n^k_h(s,a) \ge \numagents} \frac{\sum_{v=1}^\numagents \indevent{s^{k,v}_h=s,a^{k,v}_h=a} }{\sqrt{n^k_h(s,a)}}
        \\
        \nonumber
        & =
        2 H S A \numagents + \sum_{h,s,a} \sum_{k : n^k_h(s,a) \ge \numagents} \frac{\sum_{v=1}^\numagents \indevent{s^{k,v}_h=s,a^{k,v}_h=a} }{\sqrt{n^{k+1}_h(s,a)}} \sqrt{\frac{n^{k+1}_h(s,a)}{n^k_h(s,a)}}
        \\
        \nonumber
        & \le
        2 H S A \numagents + \sum_{h,s,a} \sum_{k : n^k_h(s,a) \ge \numagents} \frac{\sum_{v=1}^\numagents \indevent{s^{k,v}_h=s,a^{k,v}_h=a} }{\sqrt{n^{k+1}_h(s,a)}} \sqrt{\frac{n^k_h(s,a) + \numagents}{n^k_h(s,a)}}
        \\
        \nonumber
        & \le
        2 H S A \numagents + 2 \sum_{h,s,a} \sum_{k : n^k_h(s,a) \ge \numagents} \frac{\sum_{v=1}^\numagents \indevent{s^{k,v}_h=s,a^{k,v}_h=a} }{\sqrt{n^{k+1}_h(s,a)}}
        \\
        \nonumber
        & =
        2 H S A \numagents + 2 \sum_{h,s,a} \sum_{k : n^k_h(s,a) \ge \numagents} \frac{\sum_{v=1}^\numagents \indevent{s^{k,v}_h=s,a^{k,v}_h=a} }{\sqrt{\sum_{j=1}^k \sum_{v=1}^\numagents \indevent{s^{j,v}_h=s,a^{j,v}_h=a}}}
        \\
        \nonumber
        & \le
        2 H S A \numagents + 4 \sumhsa \sqrt{\sum_{k=1}^K \sum_{v=1}^\numagents \indevent{s^{k,v}_h=s,a^{k,v}_h=a} }
        \\
        \nonumber
        & \le
        2 H S A \numagents + 4 \sum_{h=1}^H \sqrt{SA \sum_{k=1}^K \sum_{v=1}^\numagents \sum_{s \in \calS} \sum_{a \in \calA} \indevent{s^{k,v}_h=s,a^{k,v}_h=a}}
        \\
        & =
        2 H S A \numagents + 4 H \sqrt{S A K \numagents},
        \label{eq:adv shrink sqrt}
    \end{align}
    where the forth inequality is by \citet[Lemma 1]{streeter2010less}, and the last inequality is by Jensen's inequality.
    Putting these together we get that: $(A) \lesssim \sqrt{\frac{H^4 S^2 A K \logterm}{\numagents}} + H^2 S^2 A \logterm^2 +H\sqrt{K \logterm}$.
    
    Term $(B)$ can be further decomposed as:
    \begin{align*}
        (B)
        =
        \sum_{k=1}^K \langle c^k - \hat c^k , q^k \rangle
        =
        \sum_{k=1}^K \langle c^k - \bbEk [ \hat c^k ] , q^k \rangle + \sum_{k=1}^K \langle \bbEk [ \hat c^k ] - \hat c^k , q^k \rangle.
    \end{align*}
    The second term is bounded by $4 H \sqrt{K \log \frac{6}{\delta}}$ by the good event $E^c$. For the first term, let $W_h^{\pi^k}(s,a) = {\Pr [\exists v:  s^{k,v}_h = s,a^{k,v}_h = a \mid \pi^k]} = {1-(1-q_h^{\pi^k}(s,a))^\numagents}$:
    \begin{align*}
        \sum_{k=1}^K \langle c^k - \bbEk [ \hat c^k ] , q^k \rangle
        & =
        \sum_{k=1}^K \sum_{h=1}^H \sum_{s \in \calS} \sum_{a \in \calA} q^k_h(s,a) c^k_h(s,a) \left( 1 - \frac{\bbEk [\indevent{\exists v: \  s^{k,v}_h = s,a^{k,v}_h = a}]}{U^k_h(s,a) + \gamma} \right)
        \\
        & =
        \sum_{k,h,s,a} q^k_h(s,a) c^k_h(s,a) \left( 1 - \frac{W^{\pi^k}_h(s,a)}{U^k_h(s,a) + \gamma} \right)
        \le
        \sum_{k,h,s,a} \frac{q^k_h(s,a)}{W^{\pi^k}_h(s,a)} ( U^k_h(s,a) - W^{\pi^k}_h(s,a) + \gamma )
        \\
        & \le
        \sum_{k,h,s,a} \left( \frac{1}{\numagents} + q^k_h(s,a) \right) ( U^k_h(s,a) - W^{\pi^k}_h(s,a) + \gamma )
        \\
        & =
        \sum_{k,h,s,a} \left( \frac{1}{\numagents} + q^k_h(s,a) \right) \left( (1-q^{\pi^k}_h(s,a))^\numagents - (1-u^k_h(s,a))^\numagents \right) + \frac{\gamma H S A K}{\numagents} + \gamma H K
        \\
        & \le
        \sum_{k,h,s,a} \left( \frac{1}{\numagents} + q^k_h(s,a) \right) \numagents (1- q^{\pi^k}_h(s,a))^\numagents (u^k_h(s,a) - q^{\pi^k}_h(s,a)) + \frac{\gamma H S A K}{\numagents} + \gamma H K
        \\
        & \le
        \sum_{k,h,s,a} u^k_h(s,a) - q^{\pi^k}_h(s,a) + \frac{\gamma H S A K}{\numagents} + \gamma H K
        \\
        & \qquad + 
        \sum_{k,h,s,a} \numagents q^{k}_h(s,a) (1- q^{\pi^k}_h(s,a))^\numagents (u^k_h(s,a) - q^{\pi^k}_h(s,a))
        \\
        & \le
        3 \sum_{k,h,s,a} (u^k_h(s,a) - q^{\pi^k}_h(s,a)) \log \numagents + \frac{\gamma H S A K}{\numagents} + \gamma H K,
    \end{align*}
    where the second inequality is by \cref{lemma:linear-approx}, the third inequality is by convexity of the function $f(x) = (1 - x)^m$ for $x \in [0,1]$, and the last inequality follows because $m x (1-x)^m \le \log m$ for every $x \in [0,1]$ since if $1 - x \le 1 - \frac{2 \log m}{m}$ then $(1 - x)^m \le \frac{1}{m^2}$; otherwise $x \le \frac{2 \log m}{m}$.
    Finally, $\sum_{k,h,s,a} (u^k_h(s,a) - q^{\pi^k}_h(s,a))$ is bounded by \cref{lemma:Jin-final}.
    
    Term $(C)$ is bounded by OMD (see, e.g., \citet{rosenberg2019online}) as follows:
    \begin{align*}
        (C)
        & =
        \sum_{k=1}^K \langle \hat c^k , q^k - q^{\pi^\star} \rangle
        \le
        \frac{2 H \log (H S A)}{\eta} + \eta \sum_{k=1}^K \sum_{h=1}^H \sum_{s \in \calS} \sum_{a \in \calA} q^k_h(s,a) \hat c^k_h(s,a)^2
        \\
        & \le
        \frac{2 H \log (H S A)}{\eta} + \eta \sum_{k,h,s,a} q^k_h(s,a) \frac{\hat c^k_h(s,a)}{U^k_h(s,a) + \gamma}
        \le
        \frac{2 H \log (H S A)}{\eta} + \eta \sum_{k,h,s,a} \hat c^k_h(s,a) \frac{q^k_h(s,a)}{1 - \left( 1 - q^k_h(s,a) \right)^\numagents}
        \\
        & \le
        \frac{2 H \log (H S A)}{\eta} + \eta \sum_{k,h,s,a} \left( \frac{1}{\numagents} + q^k_h(s,a) \right) \hat c^k_h(s,a)
        \\
        & \le
        \frac{2 H \log (H S A)}{\eta} + \eta \sum_{k,h,s,a} \left( \frac{1}{\numagents} + q^k_h(s,a) \right) c^k_h(s,a) + \frac{\eta HSA \log \frac{3 HSA}{\delta}}{\gamma}
        \\
        & \le
        \frac{2 H \log \frac{H S A}{\delta}}{\eta} + \frac{\eta H S A K}{\numagents} + \eta H K + \frac{\eta HSA \log \frac{3 HSA}{\delta}}{\gamma},
    \end{align*}
    where the forth inequality is by \cref{lemma:linear-approx}, and the fifth inequality is by the good event $E^{\hat c}$.
    
    Term $(D)$ is bounded by $\frac{H \log \frac{3 H}{\delta}}{\gamma}$ by the good event $E^\star$.
    Putting the three terms together gives the final regret bound when setting $\eta = \gamma = 1 / \sqrt{\left( 1 + \frac{S A}{\numagents} \right) K}$.
\end{proof}

\subsection{Auxiliary Lemmas}

The following Lemma is by \citet[Lemma 8]{jin2020learning}, \citet[Lemma B.13]{cohen2021minimax}.
\begin{lemma}
    \label{lemma:confidence with p}
    Under the good event we have,
    \[
        \forall (k,s,a,s',h):
        \quad
        \ |p_h (s'|s,a) - \hat{p}^{k}_h (s'|s,a)|
        \le \tilde\epsilon^k_h( s' \mid  s,a).
    \]
\end{lemma}

The following Lemma is part of the proof of \citet[Lemma 4]{jin2020learning}. We provide the proof here for completeness.

\begin{lemma}
    \label{lemma:Jin}
    Let $q_{\tilde{h}}^{\pi^k}(\tilde{s},\tilde{a} \mid \tilde{s}' ; h)$ be the probability to visit $(\tilde{s},\tilde{a})$ in time $\tilde h$ given that we visited $\tilde{s}'$ in time $h$.
    Under the good event,
    \begin{align}
        \label{eq:Jin}
    	\sum_{k=1}^K & \sumhsa |u_{h}^{k}(s,a)  - q_{h}^{\pi^k}(s,a)| 
    	\lesssim
    	H\sum_{k=1}^{K}\sum_{h=1}^{H}\sum_{s \in \calS,a \in \calA}  \tilde{\epsilon}_{h}^{k}( s , a ) q_{h}^{\pi^k}(s,a)
    	\\
    	\nonumber
        & + HS \sum_{k=1}^{K} \sum_{1\leq h<\tilde{h}\leq H} \sum_{s \in \calS,a \in \calA,s' \in \calS} \sum_{\tilde{s} \in \calS,\tilde{a} \in \calA}  \tilde{\epsilon}_{h}^{k}(s' \mid  s,a) q_{h}^{\pi^k}(s,a)
    	  \min \left\{ 2 , \sum_{\tilde{s}' \in \calS} \tilde{\epsilon}_{\tilde{h}}^{k}(\tilde{s}' \mid \tilde{s},\tilde{a}) \right\} q_{\tilde{h}}^{\pi^k}(\tilde{s},\tilde{a} \mid s' ; h+1).
    \end{align}
\end{lemma}

\begin{proof}
    Let $q^{k,s,h}$ be the occupancy measure such that $q_{h}^{k,s,h}(s)=u_{h}^{k}(s)$,
    and let $p^{k,s,h}$ be the transition that corresponds to $q^{k,s,h}$.
    Let $\sigma_h(s)$ be the set of all trajectories that end in $s$ in time $h$, i.e., $\sigma_h(s) = \{ s_1,a_1,\dots,s_{h-1},a_{h-1}, s_h \}$ where $s_h = s$.
    We have:
    \begin{align*}
    	u_{h}^{k}(s,a)=q_{h}^{k,s,h}(s,a) 
    	& =
    	\pi_{h}^{k}(a \mid  s) \sum_{\sigma_h(s)} \prod_{h'=1}^{h - 1} \pi_{h'}^{k}(a_{h'} \mid  s_{h'})p_{h'}^{k,s,h}(s_{h'+1} \mid  s_{h'},a_{h'}) 
    	\\
    	q_{h}^{\pi^k}(s,a)                    
    	& =
    	\pi_{h}^{k}(a \mid  s) \sum_{\sigma_h(s)} \prod_{h'=1}^{h - 1} \pi_{h'}^{k}(a_{h'} \mid  s_{h'})p_{h'}(s_{h'+1} \mid  s_{h'},a_{h'}).
    \end{align*}
    Then,
    \[
    	|u_{h}^{k}(s,a) - q_{h}^{\pi^k}(s,a)|=\pi_{h}^{k}(a \mid  s) \sum_{\sigma_h(s)} \prod_{h'=1}^{h - 1} \pi_{h'}^{k}(a_{h'} \mid  s_{h'})\left| \prod_{h'=1}^{h - 1}p_{h'}^{k,s,h}(s_{h'+1} \mid  s_{h'},a_{h'}) -  \prod_{h'=1}^{h - 1}p_{h'}(s_{h'+1} \mid  s_{h'},a_{h'})\right|.
    \]
    We can rewrite the following term as,
    \begin{align*}
    	  \Bigg| \prod_{h'=1}^{h - 1}p^{k,s,h} & (s_{h'+1} \mid  s_{h'},a_{h'}) -  \prod_{h'=1}^{h - 1}p_{h'}(s_{h'+1} \mid  s_{h'},a_{h'}) \Bigg|
    	  \\
    	  & =\Bigg| \sum_{l=2}^{h - 1} \prod_{h'=1}^{l - 1}p_{h'}(s_{h'+1} \mid  s_{h'},a_{h'}) \prod_{h'=l}^{h - 1}p_{h'}^{k,s,h}(s_{h'+1} \mid  s_{h'},a_{h'})+ \prod_{h'=1}^{h - 1}p_{h'}^{k,s,h}(s_{h'+1} \mid  s_{h'},a_{h'})
    	  \\
    	  &  -  \prod_{h'=1}^{h - 1}p_{h'}(s_{h'+1} \mid  s_{h'},a_{h'}) -  \sum_{l=2}^{h - 1} \prod_{h'=1}^{l - 1}p_{h'}(s_{h'+1} \mid  s_{h'},a_{h'}) \prod_{h'=l}^{h - 1}p_{h'}^{k,s,h}(s_{h'+1} \mid  s_{h'},a_{h'})\Bigg|
    	  \\
    	  & =\Bigg| \sum_{l=1}^{h - 1} \prod_{h'=1}^{l - 1}p_{h'}(s_{h'+1} \mid  s_{h'},a_{h'}) \prod_{h'=l}^{h - 1}p_{h'}^{k,s,h}(s_{h'+1} \mid  s_{h'},a_{h'})
    	  \\
    	  &  -  \sum_{l=2}^{h} \prod_{h'=1}^{l - 1}p_{h'}(s_{h'+1} \mid  s_{h'},a_{h'}) \prod_{h'=l}^{h - 1}p_{h'}^{k,s,h}(s_{h'+1} \mid  s_{h'},a_{h'})\Bigg|
    	  \\
    	  & =\Bigg| \sum_{l=1}^{h - 1} \prod_{h'=1}^{l - 1}p_{h'}(s_{h'+1} \mid  s_{h'},a_{h'}) \prod_{h'=l}^{h - 1}p_{h'}^{k,s,h}(s_{h'+1} \mid  s_{h'},a_{h'})
    	  \\
    	  &  -  \sum_{l=1}^{h - 1} \prod_{h'=1}^{l}p_{h'}(s_{h'+1} \mid  s_{h'},a_{h'}) \prod_{h'=l+1}^{h - 1}p_{h'}^{k,s,h}(s_{h'+1} \mid  s_{h'},a_{h'})\Bigg|
    	  \\
    	  & = \sum_{l=1}^{h - 1}\left|p_{l}^{k,s,h}(s_{l+1} \mid  s_{l},a_{l}) - p_{l}(s_{l+1} \mid  s_{l},a_{l})\right| \prod_{h'=1}^{l - 1}p_{h}(s_{h'+1} \mid  s_{h'},a_{h'}) \prod_{h'=l+1}^{h - 1}p_{h'}^{k,s,h}(s_{h'+1} \mid  s_{h'},a_{h'}).
    \end{align*}
    Hence,
    \begin{align*}
    	  |u_{h}^{k}( & s,a) - q_{h}^{\pi^k}(s,a)|
    	  \\
    	  & \leq \pi_{h}^{k}(a \mid  s) \sum_{\sigma_h(s)} \prod_{h'=1}^{h - 1} \pi_{h'}^{k}(a_{h'} \mid  s_{h'}) \sum_{l=1}^{h - 1} \left|p_{l}^{k,s,h}(s_{l+1} \mid  s_{l},a_{l}) - p_{l}(s_{l+1} \mid  s_{l},a_{l})\right|
    	  \\
    	  & \qquad\qquad\qquad\qquad\qquad\qquad\qquad\qquad\qquad\qquad
    	  \cdot \prod_{h'=1}^{l - 1}p_{h'}(s_{h'+1} \mid  s_{h'},a_{h'}) \prod_{h'=l+1}^{h - 1}p_{h'}^{k,s,h}(s_{h'+1} \mid  s_{h'},a_{h'}) 
    	  \\
    	  & \leq  \sum_{l=1}^{h - 1} \sum_{\sigma_h(s)} \left|p_{l}^{k,s,h}(s_{l+1} \mid  s_{l},a_{l}) - p_{l}(s_{l+1} \mid  s_{l},a_{l})\right| \left( \pi_{l}^{k}(a_{l} \mid  s_{l}) \prod_{h'=1}^{l - 1} \pi_{h'}^{k}(a_{h'} \mid  s_{h'})p_{h'}(s_{h'+1} \mid  s_{h'},a_{h'}) \right) 
    	  \\
    	  & \qquad\qquad\qquad\qquad\qquad\qquad\qquad\qquad
    	  \cdot
    	   \left( \pi_{h}^{k}(a \mid  s) \prod_{h'=l+1}^{h - 1} \pi_{h'}^{k}(a_{h'} \mid  s_{h'})p_{h'}^{k,s,h}(s_{h'+1} \mid  s_{h'},a_{h'}) \right) 
    	  \\
    	  & = \sum_{l=1}^{h - 1} \sum_{s_{l} \in \calS,a_{l} \in \calA,s_{l+1} \in \calS} \left|p_{l}^{k,s,h}(s_{l+1} \mid  s_{l},a_{l}) - p_{l}(s_{l+1} \mid  s_{l},a_{l})\right|
    	  \\
    	  & \qquad\qquad\qquad\quad
    	  \cdot
    	  \left(  \sum_{\sigma_l(s_l)} \pi_{l}^{k}(a_{l} \mid  s_{l}) \prod_{h'=1}^{l - 1} \pi_{h'}^{k}(a_{h'} \mid  s_{h'})p_{h'}(s_{h'+1} \mid  s_{h'},a_{h'}) \right) 
    	  \\
    	  & \qquad\qquad\qquad\quad
    	  \cdot \left(  \sum_{a_{l+1} \in \calA} \sum_{\{s_{h''} \in \calS,a_{h''} \in \calA\}_{h''=l+2}^{h - 1}} \pi_{h}^{k}(a \mid  s) \prod_{h'=l+1}^{h - 1} \pi_{h'}^{k}(a_{h'} \mid  s_{h'})p_{h'}^{k,s,h}(s_{h'+1} \mid  s_{h'},a_{h'}) \right) 
    	  \\
    	  & = \sum_{l=1}^{h - 1} \sum_{s_{l} \in \calS,a_{l} \in \calA,s_{l+1} \in \calS} \left|p_{l}^{k,s,h}(s_{l+1} \mid  s_{l},a_{l}) - p_{l}(s_{l+1} \mid  s_{l},a_{l})\right| q_{l}^{\pi^k}(s_{l},a_{l})\cdot q_{h}^{k,s,h}(s,a \mid  s_{l+1}),
    \end{align*}
    where we ease notation and denote $q_{h}^{k,s,h}(s,a \mid  s_{l+1})=q_{h}^{k,s,h}(s,a \mid  s_{l+1} ; l+1)$.
    Similarly, we can show that,
    \begin{align*}
    	& |q_{h}^{k,s,h}(s,a \mid  s_{l+1}) - q_{h}^{\pi^k}(s,a \mid  s_{l+1})| 
    	\\
    	& \quad \lesssim
    	\sum_{h'=l+1}^{h - 1} \sum_{s_{h'} \in \calS,a_{h'} \in \calA,s_{h'+1} \in \calS} \left|p_{h'}^{k,s,h}(s_{h'+1} \mid  s_{h'},a_{h'}) - p_{h'}(s_{h'+1} \mid  s_{h'},a_{h'})\right| q_{h'}^{\pi^k}(s_{h'},a_{h'} \mid  s_{l+1}) q_{h'}^{k,s,h}(s,a \mid  s_{h'+1}) 
    	\\
        & \quad \leq \pi_{h}^{k}(a \mid  s) \sum_{h'=l+1}^{h - 1} \sum_{s_{h'} \in \calS,a_{h'} \in \calA,s_{h'+1} \in \calS} \left|p_{h'}^{k,s,h}(s_{h'+1} \mid  s_{h'},a_{h'}) - p_{h'}(s_{h'+1} \mid  s_{h'},a_{h'})\right| q_{h'}^{\pi^k}(s_{h'},a_{h'} \mid  s_{l+1}),
    \end{align*}
    where the last is since $q_{h'}^{k,s,h}(s,a \mid  s_{h'+1})\leq \pi_{h}^{k}(a \mid  s)$.
    Combining the last two,
    
    \begin{align*}
    	  \sum_{h,s,a,k} & |u_{h}^{k}(s,a) - q_{h}^{\pi^k}(s,a)|
    	  \\
    	  & \lesssim  \sum_{h,s,a,k} \sum_{l=1}^{h - 1} \sum_{s_{l} \in \calS,a_{l} \in \calA,s_{l+1} \in \calS} \left|p_{l}^{k,s,h}(s_{l+1} \mid  s_{l},a_{l}) - p_{l}(s_{l+1} \mid  s_{l},a_{l})\right| q_{l}^{\pi^k}(s_{l},a_{l})\cdot q_{h}^{k,s,h}(s,a \mid  s_{l+1})
    	  \\
    	  & \leq  \sum_{h,s,a,k} \sum_{l=1}^{h - 1} \sum_{s_{l} \in \calS,a_{l} \in \calA,s_{l+1} \in \calS} \tilde{\epsilon}_{l}^{k}(s_{l+1} \mid  s_{l},a_{l}) q_{l}^{\pi^k}(s_{l},a_{l})\cdot q_{h}^{\pi^k}(s,a \mid  s_{l+1})
    	  \\
    	  & \qquad + \sum_{h,s,a,k} \sum_{l=1}^{h - 1} \sum_{s_{l} \in \calS,a_{l} \in \calA,s_{l+1} \in \calS} \tilde{\epsilon}_{l}^{k}(s_{l+1} \mid  s_{l},a_{l}) q_{l}^{\pi^k}(s_{l},a_{l})\pi_{h}^{k}(a \mid  s)\\
    	  & \qquad\qquad \cdot 
    	  \left( \sum_{h'=l+1}^{h - 1} \sum_{s_{h'} \in \calS,a_{h'} \in \calA,s_{h'+1} \in \calS} \left|p_{h'}^{k,s,h}(s_{h'+1} \mid  s_{h'},a_{h'}) - p_{h'}(s_{h'+1} \mid  s_{h'},a_{h'})\right| q_{h'}^{\pi^k}(s_{h'},a_{h'} \mid  s_{l+1}) \right) 
    	  \\
    	  & \le \sum_{k,h} \sum_{l=1}^{h - 1} \sum_{s_{l} \in \calS,a_{l} \in \calA,s_{l+1} \in \calS} \tilde{\epsilon}_{l}^{k}(s_{l+1} \mid  s_{l},a_{l}) q_{l}^{\pi^k}(s_{l},a_{l})\cdot \left(  \sum_{s,a}q_{h}^{\pi^k}(s,a \mid  s_{l+1}) \right) 
    	  \\
    	  & \qquad + \sum_{h,s,k} \sum_{l=1}^{h - 1} \sum_{s_{l} \in \calS,a_{l} \in \calA,s_{l+1} \in \calS} \tilde{\epsilon}_{l}^{k}(s_{l+1} \mid  s_{l},a_{l}) q_{l}^{\pi^k}(s_{l},a_{l})\sum_{a} \pi_{h}^{k}(a \mid  s)\\
    	  & \qquad\qquad\qquad \cdot
    	  \left( \sum_{h'=l+1}^{h - 1} \sum_{s_{h'} \in \calS,a_{h'} \in \calA} \min \left\{ 2 , \sum_{s_{h'+1} \in \calS} \tilde{\epsilon}_{h'}^{k}(s_{h'+1} \mid  s_{h'},a_{h'}) \right\} q_{h'}^{\pi^k}(s_{h'},a_{h'} \mid  s_{l+1}) \right)
    	  \\
    	  & \leq  H \sum_{k=1}^{K} \sum_{1\leq  l\leq  H} \sum_{s_{l} \in \calS,a_{l} \in \calA,s_{l+1} \in \calS} \tilde{\epsilon}_{l}^{k}(s_{l+1} \mid  s_{l},a_{l}) q_{l}^{\pi^k}(s_{l},a_{l})
    	  \\
    	  &\qquad +HS \sum_{k=1}^{K} \sum_{1\leq  l<h'\leq  H} \sum_{s_{l} \in \calS,a_{l} \in \calA,s_{l+1} \in \calS} \sum_{s_{h'} \in \calS,a_{h'} \in \calA} \tilde{\epsilon}_{l}^{k}(s_{l+1} \mid  s_{l},a_{l}) q_{l}^{\pi^k}(s_{l},a_{l})
    	  \\
    	  & \qquad\qquad\qquad\qquad\qquad\qquad\qquad\qquad\qquad\qquad \cdot \min \left\{ 2 , \sum_{s_{h'+1} \in \calS} \tilde{\epsilon}_{h'}^{k}(s_{h'+1} \mid  s_{h'},a_{h'}) \right\} q_{h'}^{\pi^k}(s_{h'},a_{h'} \mid  s_{l+1})
    	  \\
    	  & =H \sum_{k=1}^{K} \sum_{h=1}^{H} \sum_{s \in \calS,a \in \calA,s' \in \calS} \tilde{\epsilon}_{h}^{k}(s' \mid  s,a) q_{h}^{\pi^k}(s,a)
    	  \\
    	  & \qquad +HS \sum_{k=1}^{K} \sum_{1\leq h<\tilde{h}\leq H} \sum_{s \in \calS,a \in \calA,s' \in \calS} \sum_{\tilde{s} \in \calS,\tilde{a} \in \calA}  \tilde{\epsilon}_{h}^{k}(s' \mid  s,a) q_{h}^{\pi^k}(s,a)
    	  \min \left\{ 2 , \sum_{\tilde{s}' \in \calS} \tilde{\epsilon}_{\tilde{h}}^{k}(\tilde{s}' \mid \tilde{s},\tilde{a}) \right\} q_{\tilde{h}}^{\pi^k}(\tilde{s},\tilde{a} \mid s' ; h+1),
    \end{align*}
    where the last inequality is by \cref{lemma:confidence with p} and since $p^{k,s,h}_l(\cdot | s,a) , p_l(\cdot| s,a)$ are probability distributions $\forall (s,a,l)$.
\end{proof}

\begin{lemma}
    \label{lemma:Jin-final}
    Under the good event,
    \[
        \sum_{h,s,a,k}|u_{h}^{k}(s,a) - q_{h}^{\pi^k}(s,a)|
        \lesssim \sqrt{\frac{ H^{4} S^{2} A K \logterm}{m}} + H^3 S^{3} A \logterm^{2}.
    \]
\end{lemma}

\begin{proof}
    We first bound $\sum_{h,s,a,k}|u_{h}^{k}(s,a) - q_{h}^{\pi^k}(s,a)|$ using \Cref{lemma:Jin}. Now, for the first term in \cref{eq:Jin}:
    \begin{align*}
        \sum_{k,h,s,a} \tilde{\epsilon}_{h}^{k}(s,a) q_{h}^{\pi^k}(s,a)
        & =
        \frac{1}{m}\sum_{v=1}^{m} \sum_{k,h,s,a} \tilde{\epsilon}_{h}^{k}(s,a) q_{h}^{\pi^k}(s,a)
        \leq 
        \frac{2}{m} \sum_{v=1}^{m} \sum_{k,h,s,a} \indevent{s_{h}^{k,v}=s,a_{h}^{k,v}=a} \tilde{\epsilon}_{h}^{k}(s,a) + 18HS\logterm^2
        \\
        & \lesssim\frac{\sqrt{S\logterm}}{\numagents} \sum_{k,h,s,a}\frac{\sum_{v=1}^{\numagents} \indevent{s_{h}^{k,v}=s,a_{h}^{k,v}=a}}{\sqrt{n_{h}^{k}(s,a)\vee1}} 
        \\
        & \qquad + \frac{S\logterm}{\numagents}\sum_{k,h,s,a} \frac{\sum_{v=1}^{\numagents} \indevent{s_{h}^{k,v}=s,a_{h}^{k,v}=a}}{n_{h}^{k}(s,a)\vee1} + HS\logterm^2
        \\
        & \lesssim \sqrt{\frac{ H^{2} S^{2} A K \logterm}{m}} + H S^{2} A \logterm^{2},
\end{align*}
where the first inequality is by event $E^{on2}$ and the last inequality is by \cref{eq:adv shrink sqrt,eq:adv shrink 1/n}.
Plugging the definition of $\tilde\epsilon$, we break the second sum in  Eq. \eqref{eq:Jin} as follows:

\begin{align*}
	  \sum_{k=1}^{K}  \sum_{1 \leq  h<\tilde{h} \leq  H} & \sum_{s,a,s'} \sum_{\tilde{s},\tilde{a}} \tilde{\epsilon}_{h}^{k}(s' \mid s,a) q_{h}^{\pi^k}(s,a) \cdot \min \left\{ 2, \sum_{\tilde{s}'}\tilde{\epsilon}_{\tilde{h}}^{k}(\tilde{s}' \mid\tilde{s},\tilde{a}) \right\} q_{\tilde{h}}^{\pi^k}(\tilde{s},\tilde{a} \mid s';h+1)
	  \\
	  & \lesssim \sum_{k=1}^{K} \sum_{1 \leq  h<\tilde{h} \leq  H} \sum_{s,a,s'} \sum_{\tilde{s},\tilde{a},\tilde{s}'} \sqrt{ \frac{p_{h}^{k}(s'|s,a)\tau}{n_{h}^{k}(s,a) \vee 1}} q_{h}^{\pi^k}(s,a) \cdot  \sqrt{ \frac{p_{\tilde h}^{k}(\tilde{s}'|\tilde{s},\tilde{a})\tau}{n_{\tilde h}^{k}(\tilde{s},\tilde{a}) \vee 1}} q_{\tilde{h}}^{\pi^k}(\tilde{s},\tilde{a} \mid s';h+1)
	  \\
	  & \quad+ \sum_{k=1}^{K} \sum_{1 \leq  h<\tilde{h} \leq  H} \sum_{s,a,s'} \sum_{\tilde{s},\tilde{a}} \sqrt{ \frac{p_{h}^{k}(s'|s,a)\tau}{n_{h}^{k}(s,a) \vee 1}} q_{h}^{\pi^k}(s,a) \cdot  \min \left\{ 2, \frac{S \tau}{n^k_{\tilde h}(\tilde s, \tilde a) \vee 1} \right\} q_{\tilde{h}}^{\pi^k}(\tilde{s},\tilde{a} \mid s';h+1)
	  \\
	  & \quad+ \sum_{k=1}^{K} \sum_{1 \leq  h<\tilde{h} \leq  H} \sum_{s,a,s'} \sum_{\tilde{s},\tilde{a}} \frac{\tau}{n_{h}^{k}(s,a) \vee 1} q_{h}^{\pi^k}(s,a) \cdot  \min \left\{ 2, \sum_{\tilde{s}'}\tilde{\epsilon}_{\tilde{h}}^{k}(\tilde{s}' \mid\tilde{s},\tilde{a}) \right\} q_{\tilde{h}}^{\pi^k}(\tilde{s},\tilde{a} \mid s';h+1)
	  \\
	  & \lesssim\underbrace{
	   \sum_{k=1}^{K} \sum_{1 \leq  h<\tilde{h} \leq  H} \sum_{s,a,s'} \sum_{\tilde{s},\tilde{a},\tilde{s}'} \sqrt{ \frac{p_{h}^{k}(s'|s,a)\tau}{n_{h}^{k}(s,a) \vee 1}} q_{h}^{\pi^k}(s,a) \cdot  \sqrt{ \frac{p_{\tilde h}^{k}(\tilde{s}'|\tilde{s},\tilde{a})\tau}{n_{\tilde h}^{k}(\tilde{s},\tilde{a}) \vee 1}} q_{\tilde{h}}^{\pi^k}(\tilde{s},\tilde{a} \mid s';h+1)}
	  _{(i)} 
	  \\
	  & \quad+\underbrace{
	   \sum_{k=1}^{K} \sum_{1 \leq  h<\tilde{h} \leq  H} \sum_{s,a,s'} \sum_{\tilde{s},\tilde{a},\tilde{s}'} q_{h}^{\pi^k}(s,a) p_{h}^{k}(s'|s,a) \cdot  \frac{\tau}{n_{\tilde{h}}^{k}(\tilde{s},\tilde{a}) \vee 1} q_{\tilde{h}}^{\pi^k}(\tilde{s},\tilde{a} \mid s';h+1)}
	  _{(ii)}
	  \\
	  & \quad+\underbrace{
	  \sum_{k=1}^{K} \sum_{1 \leq  h<\tilde{h} \leq  H} \sum_{s,a,s'} \sum_{\tilde{s},\tilde{a}} \frac{\tau}{n_{h}^{k}(s,a) \vee 1} q_{h}^{\pi^k}(s,a) \cdot q_{\tilde{h}}^{\pi^k}(\tilde{s},\tilde{a} \mid s';h+1)}
	  _{(iii)},
\end{align*}
where the last inequality follows because $\sqrt{xy} \le x + y$ for every $x,y \ge 0$.
Term $(i)$ is bounded as follows:

\begin{align*}
	(i) 
	& =
	\tau \sum_{1 \leq  h<\tilde{h} \leq  H} \sum_{k=1}^{K} \sum_{s,a,s'} \sum_{\tilde{s},\tilde{a},\tilde{s}'} \sqrt{ \frac{q_{h}^{\pi^k}(s,a) q_{\tilde{h}}^{\pi^k}(\tilde{s},\tilde{a} \mid s';h+1) p_{h}^{k}(\tilde{s}'|\tilde{s},\tilde{a})}{n_{h}^{k}(s,a) \vee 1}} \cdot  \sqrt{ \frac{q_{h}^{\pi^k}(s,a) q_{\tilde{h}}^{\pi^k}(\tilde{s},\tilde{a} \mid s';h+1) p_{h}^{k}(s'|s,a)}{n_{\tilde{h}}^{k}(\tilde{s},\tilde{a}) \vee 1}}
	\\
	    &  \leq \tau \sum_{1 \leq  h<\tilde{h} \leq  H} \sqrt{ \sum_{k,s,a,s',\tilde{s},\tilde{a},\tilde{s}'} \frac{q_{h}^{\pi^k}(s,a) q_{\tilde{h}}^{\pi^k}(\tilde{s},\tilde{a} \mid s';h+1) p_{h}^{k}(\tilde{s}'|\tilde{s},\tilde{a})}{n_{h}^{k}(s,a) \vee 1}} 
	    \\
	    & \qquad\qquad\qquad\quad
	    \cdot  \sqrt{ \sum_{k,s,a,s',\tilde{s},\tilde{a},\tilde{s}'} \frac{q_{h}^{\pi^k}(s,a) q_{\tilde{h}}^{\pi^k}(\tilde{s},\tilde{a} \mid s';h+1) p_{h}^{k}(s'|s,a)}{n_{\tilde{h}}^{k}(\tilde{s},\tilde{a}) \vee 1}} 
	    \\
	    & =\tau \sum_{1 \leq  h<\tilde{h} \leq  H} \sqrt{S \sum_{k=1}^{K} \sum_{s,a} \frac{q_{h}^{\pi^k}(s,a)}{n_{h}^{k}(s,a) \vee 1}} \cdot  \sqrt{ \sum_{k=1}^{K} \sum_{s',\tilde{s},\tilde{a},\tilde{s}'} \frac{q_{h+1}^{\pi^k}(s') q_{\tilde{h}}^{\pi^k}(\tilde{s},\tilde{a} \mid s';h+1)}{n_{\tilde{h}}^{k}(\tilde{s},\tilde{a}) \vee 1}}
	    \\
	    &  \leq HS \logterm \sqrt{ \sum_{k=1}^{K} \sum_{s,a,h} \frac{q_{h}^{\pi^k}(s,a)}{n_{h}^{k}(s,a) \vee 1}} \cdot  \sqrt{ \sum_{k=1}^{K} \sum_{\tilde{s},\tilde{a},\tilde{h}} \frac{q_{\tilde{h}}^{\pi^k}(\tilde{s},\tilde{a})}{n_{\tilde{h}}^{k}(\tilde{s},\tilde{a}) \vee 1}}
	    \leq HS \logterm \sum_{k=1}^{K} \sum_{s,a,h} \frac{q_{h}^{\pi^k}(s,a)}{n_{h}^{k}(s,a) \vee 1}
	    \lesssim H^{2}S^{2}A\tau^{2},
\end{align*}
where the last inequality is by event $E^{on3}$ and \cref{eq:adv shrink 1/n}.
Term $(ii)$ is bounded as follows:
\begin{align*}
	(ii)
	& =
	S \sum_{k=1}^{K} \sum_{1 \leq  h<\tilde{h} \leq  H}  \sum_{\tilde{s},\tilde{a},s'}  \frac{\tau}{n_{\tilde{h}}^{k}(\tilde{s},\tilde{a}) \vee 1}
	q_{\tilde{h}}^{\pi^k}(\tilde{s},\tilde{a} \mid s';h+1)
	\sum_{s,a} q_{h}^{\pi^k}(s,a) p_{h}^{k}(s'|s,a)
	\\
	& =
	S \sum_{k=1}^{K} \sum_{1 \leq  h<\tilde{h} \leq  H}  \sum_{\tilde{s},\tilde{a}}  \frac{\tau}{n_{\tilde{h}}^{k}(\tilde{s},\tilde{a}) \vee 1}
	\sum_{s'} q_{\tilde{h}}^{\pi^k}(\tilde{s},\tilde{a} \mid s';h+1)
	q^{\pi^k}_{h+1}(s')
	\\
	& =
	HS\tau \sum_{k=1}^{K} \sum_{\tilde{h}}  \sum_{\tilde{s},\tilde{a}}  \frac{q_{\tilde{h}}^{\pi^k}(\tilde{s},\tilde{a})}{n_{\tilde{h}}^{k}(\tilde{s},\tilde{a}) \vee 1}
	\lesssim
	H^2 S^2 A \tau^2.
\end{align*}
Term $(iii)$ is bounded as follows:
\begin{align*}
    (iii)
    & \le
    \sum_{k=1}^{K} \sum_{1 \leq  h<\tilde{h} \leq  H} \sum_{s,a,s'} \frac{\tau}{n_{h}^{k}(s,a) \vee 1} q_{h}^{\pi^k}(s,a) \sum_{\tilde{s},\tilde{a}} q_{\tilde{h}}^{\pi^k}(\tilde{s},\tilde{a} \mid s';h+1)
    \\
    & \le
    H S \tau^2 \sum_{k=1}^{K} \sum_{s,a,h} \frac{q_{h}^{\pi^k}(s,a)}{n_{h}^{k}(s,a) \vee 1}
	\lesssim 
	H^{2}S^{2}A\tau^{2}. \qedhere
\end{align*}
\end{proof}

\newpage

\section{The \texttt{coop-nf-O-REPS} algorithm for adversarial MDPs with non-fresh randomness and known $p$}
\label{appendix:adversarial non-fresh adversarial}

\begin{algorithm}[t]
    \caption{\textsc{Cooperative O-REPS with non-fresh randomness (coop-nf-O-REPS)}} 
    \label{alg:coop-nf-o-reps}
    \begin{algorithmic}[1]
        \STATE {\bf input:} state space $\calS$, action space $\calA$, horizon $H$, transition function $p$, number of episodes $K$, number of agents $\numagents$, exploration parameter $\gamma$, learning rate $\eta$, confidence parameter $\delta$.

        \STATE {\bf initialize:} $\pi^1_h(a \mid s) = \nicefrac{1}{A} , q^1_h(s,a) = q^{\pi^1}_h(s,a) \ \forall (s,a,h) \in \calS \times \calA \times [H]$. 

        \FOR{$k=1,\dots,K$}
    
            \FOR{$v=1,\dots,\numagents$}
            
                \STATE observe initial state $s^{k,v}_1$.
            
                \FOR{$h=1,\dots,H$}
        
                    \STATE pick action $a^{k,v}_h \sim \pi^k_h(\cdot \mid s^{k,v}_h)$, suffer cost $c^k_h(s^{k,v}_h,a^{k,v}_h)$ and  observe next state $s^{k,v}_{h+1}$.
        
                \ENDFOR
                \ENDFOR
                
                \STATE For every $(s,a,h)$ compute $\wt W^k_h(s,a)$ -- the estimate of $W^k_h(s,a) = \Pr [\exists v: \  s^{k,v}_h=s,a^{k,v}=a \mid \pi^k]$ using $N = 10 \gamma^{-2} \log \frac{K H S A \numagents}{\delta}$ samples (\cref{alg:estimate-prob-non-fresh-o-reps}).
                
                \STATE compute $\hat c^k_h(s,a) = \frac{c^k_h(s,a) \indevent{\exists v: \  s^{k,v}_h = s,a^{k,v}_h = a}}{\wt W^k_h(s,a) + \gamma} \  \forall (s,a,h) \in \calS \times \calA \times [H]$.
                
                \STATE compute $q^{k+1} = \argmin_{q \in \ocset} \eta \langle q , \hat c^k \rangle + \KL{q}{q^k}$.
                
                \STATE compute $\pi^{k+1}_h(a \mid s) = \frac{q^k_h(s,a)}{\sum_{a' \in \calA} q^k_h(s,a')} \  \forall (s,a,h) \in \calS \times \calA \times [H]$.
        \ENDFOR
    \end{algorithmic}
\end{algorithm}

For the setting of adversarial MDPs with non-fresh randomness and known transitions we propose the Cooperative O-REPS with non-fresh randomness algorithm (\verb|coop-nf-O-REPS|; see \cref{alg:coop-nf-o-reps}).
The idea is similar to the \verb|coop-O-REPS| algorithm for fresh randomness, but the key difference is that the probability to reach some state-action pair that the algorithm uses (i.e., $W^k_h(s,a)$) must be computed differently in order to suit non-fresh randomness.
In fact, computing $W^k_h(s,a)$ becomes a difficult challenge once the randomness is non-fresh, and a naive computation takes exponential time.
Instead we propose to estimate $W^k_h(s,a)$ from samples.
That is, we simulate $N = 10 \gamma^{-2} \log \frac{K H S A \numagents}{\delta}$ i.i.d episodes in which all agents use policy $\pi^k$ and then estimate $W^k_h(s,a)$ by the fraction of episodes in which $(s,a)$ was reached in step $h$ by at least one of the $\numagents$ agents.
This way our algorithm keeps polynomial running time.

\begin{theorem}
    \label{thm:reg-coop-nf-o-reps}
    With probability $1 - \delta$, setting $\eta = \gamma = 1 / \sqrt{\left( 1 + \frac{S A}{\numagents} \right) K}$, the individual regret of each agent of \verb|coop-nf-O-REPS| is
    \[
        \regret
        =
        O \left( H \sqrt{S K \log \frac{H S A}{\delta}} + \sqrt{\frac{H^2 S A K}{\numagents} \log \frac{H S A}{\delta}} + \frac{HSA}{\numagents} \log \frac{HSA}{\delta} + HS \log \frac{HSA}{\delta} \right).
    \]
\end{theorem}

\begin{algorithm}[t]
    \caption{\textsc{Estimate reachability probability for non-fresh randomness}} 
    \label{alg:estimate-prob-non-fresh-o-reps}
    \begin{algorithmic}[1]
        \STATE {\bf input:} state space $\calS$, action space $\calA$, transition function $p$, number of agents $\numagents$, policy $\pi$, number of samples $N$, state-action-step triplet to estimate $(\bar s,\bar a, \bar h)$.

        \STATE {\bf initialize} indicator for reaching $I(n) \gets 0$ for $n \in [N]$. 

        \FOR{$n=1,\dots,N$}
        
            \STATE initialize realized transitions $p^r_h(s' \mid s,a) = 0 \  \forall (s,a,s',h)$.
        
            \FOR{$h=1,\dots, \bar h$}
            
                \FOR{$(s,a) \in \calS \times \calA$}
                
                    \STATE sample $s' \sim p_h(\cdot \mid s,a)$ and set $p^r_h(s' \mid s,a) = 1$.
                
                \ENDFOR
            
            \ENDFOR
    
            \FOR{$v=1,\dots,\numagents$}
            
                \STATE observe initial state $s^{v}_1 = \sinit$.
            
                \FOR{$h=1,\dots,\bar h$}
        
                    \STATE pick action $a^{v}_h \sim \pi_h(\cdot \mid s^{v}_h)$ and  observe next state $s^{v}_{h+1} \sim p^r_h(\cdot \mid s,a)$.
        
                \ENDFOR
                
                \IF{$s^v_{\bar h} = \bar s , a^v_{\bar h} = \bar a$}
                
                    \STATE set $I(n) \gets 1$.
                    
                    \STATE \textbf{break}
                
                \ENDIF
                
            \ENDFOR
        \ENDFOR
        
        \STATE return $\frac{1}{N} \sum_{n=1}^N I(n)$.
    \end{algorithmic}
\end{algorithm}

\subsection{The good event}

Define the following events: 
\begin{align*}
    E^{app}
    & =
    \left\{ \forall (s,a,h,k) \in \calS \times \calA \times [H] \times [K]: \ |\wt W^k_h(s,a) - W^k_h(s,a) | \le \gamma/2 \right\}
    \\
    E^c 
    & = 
    \left\{ \sum_{k=1}^K \langle \bbEk [ \hat c^k ] - \hat c^k , q^k \rangle \le 4 H \sqrt{K \log \frac{6}{\delta}} \right\}
    \\
    E^{\hat c}
    & =
    \left\{ \sum_{k=1}^K \sum_{h=1}^H \sum_{s \in \calS} \sum_{a \in \calA} \left( \frac{1}{\numagents} + \pi^k_h(a \mid s) \right) \left( \hat c^k_h(s,a) - 2 c^k_h(s,a) \right) \le \frac{10 HSA \log \frac{6 HSA}{\delta}}{\numagents \gamma} + \frac{10 HS \log \frac{6 HSA}{\delta}}{ \gamma} \right\}
    \\
    E^\star
    & =
    \left\{ \sum_{k=1}^K \langle \hat c^k - c^k , q^{\pi^\star} \rangle \le \frac{2 H \log \frac{6 H S A}{\delta}}{\gamma} \right\}
\end{align*}

The good event is the intersection of the above events. 
The following lemma establishes that the good event holds with high probability. 

\begin{lemma}[The Good Event]
    \label{lemma:good-event-adversarial-known-p-non-fresh}
    Let $\bbG =E^{app} \cap E^c \cap E^{\hat c} \cap E^\star$ be the good event. 
    It holds that $\Pr [ \bbG ] \geq 1-\delta$.
\end{lemma}

\begin{proof}
    By Hoeffding inequality we have that $\Pr[\neg E^{app}] \le \delta/6$, and the other events are similar to \cref{lemma:good-event-adversarial-known-p-fresh}.
\end{proof}

\subsection{Proof of Theorem~\ref{thm:reg-coop-nf-o-reps}}

\begin{proof}[Proof of Theorem~\ref{thm:reg-coop-nf-o-reps}]
    By \cref{lemma:good-event-adversarial-known-p-non-fresh}, the good event holds with probability $1 - \delta$.
    We now analyze the regret under the assumption that the good event holds.
    We start by decomposing the regret as follows:
    \begin{align*}
        \regret
        & =
        \sum_{k=1}^K V^{k,\pi^k}_1(s_{1}^{k,v}) - V^{k,\pi^\star}_1(s_{1}^{k,v})
        =
        \sum_{k=1}^K \langle c^k , q^k - q^{\pi^\star} \rangle
        \\
        & =
        \underbrace{\sum_{k=1}^K \langle c^k - \hat c^k , q^k \rangle}_{(A)} + \underbrace{\sum_{k=1}^K \langle \hat c^k , q^k - q^{\pi^\star} \rangle}_{(B)} + \underbrace{\sum_{k=1}^K \langle \hat c^k - c^k , q^{\pi^\star} \rangle}_{(C)}.
    \end{align*}
    
    Term $(A)$ can be further decomposed as:
    \begin{align*}
        (A)
        =
        \sum_{k=1}^K \langle c^k - \hat c^k , q^k \rangle
        =
        \sum_{k=1}^K \langle c^k - \bbEk [ \hat c^k ] , q^k \rangle + \sum_{k=1}^K \langle \bbEk [ \hat c^k ] - \hat c^k , q^k \rangle.
    \end{align*}
    The second term is bounded by $4 H \sqrt{K \log \frac{6}{\delta}}$ by the good event $E^c$, and for the first term:
    \begin{align*}
        \sum_{k=1}^K \langle c^k - \bbEk [ \hat c^k ] , q^k \rangle
        & =
        \sum_{k=1}^K \sum_{h=1}^H \sum_{s \in \calS} \sum_{a \in \calA} q^k_h(s,a) c^k_h(s,a) \left( 1 - \frac{\bbEk [\indevent{\exists v: \  s^{k,v}_h = s,a^{k,v}_h = a}]}{\wt W^k_h(s,a) + \gamma} \right)
        \\
        & \le
        \sum_{k=1}^K \sum_{h=1}^H \sum_{s \in \calS} \sum_{a \in \calA} q^k_h(s,a) c^k_h(s,a) \left( 1 - \frac{\bbEk [\indevent{\exists v: \  s^{k,v}_h = s,a^{k,v}_h = a}]}{W^k_h(s,a) + \gamma/2} \right)
        \\
        & =
        \sum_{k,h,s,a} q^k_h(s,a) c^k_h(s,a) \left( 1 - \frac{W^k_h(s,a)}{W^k_h(s,a) + \gamma/2} \right)
        \le
        \gamma \sum_{k,h,s,a} \frac{q^k_h(s,a)}{W^k_h(s,a) + \gamma/2}
        \\
        & \le
        \gamma \sum_{k,h,s,a} \frac{q^k_h(s) \pi^k_h(a \mid s)}{W^k_h(s,a)}
        \le
        \gamma \sum_{kh,s,a} \sum_{a \in \calA} \left( \frac{1}{\numagents} + \pi^k_h(a \mid s) \right)
        =
        \frac{\gamma H S A K}{\numagents} + \gamma H S K,
    \end{align*}
    where the first inequality is by the event $E^{app}$, and the last inequality is by \cref{lemma:linear-approx-non-fresh}.
    
    Term $(B)$ is bounded by OMD (see, e.g., \citet{zimin2013online}) as follows:
    \begin{align*}
        (B)
        & =
        \sum_{k=1}^K \langle \hat c^k , q^k - q^{\pi^\star} \rangle
        \le
        \frac{H \log (H S A)}{\eta} + \eta \sum_{k=1}^K \sum_{h=1}^H \sum_{s \in \calS} \sum_{a \in \calA} q^k_h(s,a) \hat c^k_h(s,a)^2
        \\
        & \le
        \frac{H \log (H S A)}{\eta} + \eta \sum_{k=1}^K \sum_{h=1}^H \sum_{s \in \calS} \sum_{a \in \calA} q^k_h(s,a) \frac{\hat c^k_h(s,a)}{\wt W^k_h(s,a) + \gamma}
        \\
        & \le
        \frac{H \log (H S A)}{\eta} + \eta \sum_{k=1}^K \sum_{h=1}^H \sum_{s \in \calS} \sum_{a \in \calA} q^k_h(s,a) \frac{\hat c^k_h(s,a)}{W^k_h(s,a) + \gamma/2}
        \\
        & \le
        \frac{H \log (H S A)}{\eta} + \eta \sum_{k=1}^K \sum_{h=1}^H \sum_{s \in \calS} \sum_{a \in \calA} \frac{q^k_h(s) \pi^k_h(a \mid s)}{W^k_h(s,a)} \hat c^k_h(s,a)
        \\
        & \le
        \frac{H \log (H S A)}{\eta} + \eta \sum_{k=1}^K \sum_{h=1}^H \sum_{s \in \calS} \sum_{a \in \calA} \left( \frac{1}{\numagents} + \pi^k_h(a \mid s) \right) \hat c^k_h(s,a)
        \\
        & \le
        \frac{H \log (H S A)}{\eta} + 2 \eta \sum_{k=1}^K \sum_{h=1}^H \sum_{s \in \calS} \sum_{a \in \calA} \left( \frac{1}{\numagents} + \pi^k_h(a \mid s) \right) c^k_h(s,a) + \frac{10\eta HSA \log \frac{6 HSA}{\delta}}{\numagents \gamma} + \frac{10\eta HS \log \frac{6 HSA}{\delta}}{ \gamma}
        \\
        & \lesssim
        \frac{H \log \frac{H S A}{\delta}}{\eta} + \frac{\eta H S A K}{\numagents} + \eta H S K + \frac{\eta HSA \log \frac{6 HSA}{\delta}}{\numagents \gamma} + \frac{\eta HS \log \frac{6 HSA}{\delta}}{ \gamma},
    \end{align*}
    where the forth inequality is by \cref{lemma:linear-approx-non-fresh}, and the fifth inequality is by the good event $E^{\hat c}$.
    
    Term $(C)$ is bounded by $\frac{2 H \log \frac{6 H}{\delta}}{\gamma}$ by the good event $E^\star$.
    Putting the three terms together gives the final regret bound when setting $\eta = \gamma = \sqrt{\frac{\log \frac{HSA}{\delta}}{\left( 1 + \frac{A}{\numagents} \right) SK}}$.
\end{proof}

\subsection{Auxiliary lemmas}

\begin{lemma}   \label{lemma:linear-approx-non-fresh}
    Let $\pi$ be a policy and denote by $q^\pi_h(s)$ the probability to reach state $s$ in time $h$ when playing policy $\pi$.
    Assume that $\numagents$ agents use the same policy $\pi$ in an MDP $\calM$ with non-fresh randomness, and denote by $W_h(s,a)$ the probability that at least one agent to reaches $(s,a)$ in time $h$.
    Then, for every $(s,a,h) \in \calS \times \calA \times [H]$, it holds that: 
    \[
        \frac{q^\pi_h(s) \pi_h(a \mid s)}{W_h(s,a)} \le  \frac{1}{m} + \pi_h(a \mid s).
    \]
\end{lemma}

\begin{proof}
    Let $M_{h}(s)$ be the number of agents that arrive at state $s$ in time $h$. 
    We have that,
    \begin{align}
        \nonumber
        W_{h}(s,a)
        & =
        \Pr [\exists v \in [\numagents]: \  s^v_h=s,a^v_h=a \mid \pi]
        =
        \bbE \left[1-(1-\pi_{h}(a\mid s))^{M_{h}(s)}\mid \pi \right]
        \\
        \label{eq:lem-linear-approx-non-fresh-1}
        & \ge
        \bbE \left[ \frac{\pi_h(a \mid s)}{\frac{1}{M_{h}(s)} + \pi_h(a \mid s)} \mid \pi \right]
        =
        \bbE \left[\frac{M_{h}(s)\pi_{h}(a\mid s)}{1+M_{h}(s)\pi_{h}(a\mid s)}\mid\pi \right],
    \end{align}
    where the inequality is by \Cref{lemma:linear-approx}.
    
    Notice that $\bbE [M_h(s) \mid \pi] = m q^\pi_h(s)$ by linearity of expectation.
    Therefore, \cref{eq:lem-linear-approx-non-fresh-1} is bounded from below by the value of the following optimization problem:
    \begin{align*}
    	   \min_{p_{0},...,p_{m}} & \sum_{i=0}^{m} p_{i}\frac{i\pi_{h}(a\mid s)}{1+i\pi_{h}(a\mid s)},
    	  \\
    	  s.t. \quad 
    	  & \sum_{i=0}^{m} p_{i}i = m q_{h}^\pi(s),
    	  \\
    	  & \sum_{i=0}^{m} p_{i} = 1,
    	  \\
    	  & p_{i}\geq0\quad\forall i\in[m],
    \end{align*}
    where $p_i$ represents $\Pr[M_h(s) = i]$.
    Since the coefficient of $p_i$ in the constrains and the objective are non-negative, we can substitute the equality constrains with $``\geq"$ constrains.
    We get the  following standard form Linear Programming:
    \begin{align*}
         \min_{p\in\mathbb{R}^{m+1}} & b^{T}p,\\
         s.t \quad 
         & A^{T}p \geq c,\\
         & p \geq 0,
    \end{align*}
    where,
    \begin{align*}
    	b=
    	\left(
    	\begin{array}{c}
        	0\\
        	\frac{\pi_{h}(a\mid s)}{1+\pi_{h}(a\mid s)}\\
        	\vdots\\
        	\frac{m\pi_{h}(a\mid s)}{1+m\pi_{h}(a\mid s)}
    	\end{array}
    	\right),
    	\quad A=
    	\left(
    	\begin{array}{cc}
        	1      & 0      \\
        	1      & 1      \\
        	1      & 2      \\
        	\vdots & \vdots \\
        	1      & m      
    	\end{array}
    	\right),
    	\quad c=
    	\left(1,mq_{h}^\pi(s)\right).
    \end{align*}
    
    The dual problem is,
    \begin{align*}
    	\max_{x_{1},x_{2}} 
        & (x_{1}+x_{2}mq_{h}^\pi(s))
        \\
        s.t \quad 
        & x_{1}\leq0,
        \\
        & x_{1}+x_{2}\leq\frac{\pi_{h}(a\mid s)}{1+\pi_{h}(a\mid s)},
        \\
        & x_{1}+2x_{2}\leq\frac{2\pi_{h}(a\mid s)}{1+2\pi_{h}(a\mid s)} ,
        \\
        & \vdots
        \\
        & x_{1}+mx_{2}\leq\frac{m\pi_{h}(a\mid s)}{1+m\pi_{h}(a\mid s)} ,
        \\
        & x_{1},x_{2}\geq0.
    \end{align*}
    From the first and the last constrains we have $x_1 = 0$ and the rest of the constrains are equivalent to $x_{2} \leq \frac{\pi_{h}(a\mid s)} {1+m\pi_{h}(a\mid s)}$.
    Hence the maximum value is $\frac{m\pi_{h}(a \mid  s)} {1+m\pi_{h}(a\mid s)} q_{h}^\pi(s)$, which completes the proof.
\end{proof}

\newpage

\section{The \texttt{coop-nf-UOB-REPS} algorithm for adversarial MDPs with non-fresh randomness and unknown $p$}
\label{appendix:adversarial-non-fresh-unknown}

\begin{algorithm}[t]
    \caption{\textsc{Cooperative UOB-REPS with non-fresh randomness (coop-nf-UOB-REPS)}} 
    \label{alg:coop-nf-uob-reps}
    \begin{algorithmic}[1]
        \STATE {\bf input:} state space $\calS$, action space $\calA$, horizon $H$, number of episodes $K$, number of agents $\numagents = \sqrt{K}$, exploration parameter $\gamma$, learning rate $\eta$, confidence parameter $\delta$.
        
        \STATE {\bf initialize:} $n^1_h(s,a)=0,n^1_h(s,a,s')=0,\pi^1_h(a \mid s) = \nicefrac{1}{A},q^1_h(s,a,s') = \nicefrac{1}{S^2 A} \  \forall (s,a,s',h) \in \calS \times \calA \times \calS \times [H]$.
        
        \STATE {\bf initialize:} define a mapping $\sigma: [H] \times \calA \times [K] \to [\numagents]$ such that $\sigma(h,a,k) \ne \sigma(h',a',k)$ whenever $h \ne h'$ or $a \ne a'$, and such that each agent is assigned by $\sigma$ exactly $H A \sqrt{K}$ times (i.e., $|\sigma^{-1}(v)| = H A \sqrt{K}$ for every $v \in [\numagents]$).

        \FOR{$k=1,\dots,K$}
    
            \STATE set $I^{k}_h(s,a,s')=0,I^{k}_h(s,a)=0 \  \forall (s,a,s',h) \in \calS \times \calA \times \calS \times [H]$.
            \FOR{$v=1,\dots,\numagents$}
            
                \STATE observe initial state $s^{k,v}_1$.
            
                \FOR{$h=1,\dots,H$}
                
                \IF{$\exists \tilde a \in \calA : \ \sigma(h, \tilde a,k) = v$}
                
                    \STATE pick action $a^{k,v}_h = \tilde a$.
                    
                \ELSE
                
                \STATE pick action $a^{k,v}_h \sim \pi^k_h(\cdot \mid s^{k,v}_h)$.
                    
                \ENDIF
        
                    \STATE suffer cost $c^k_h(s^{k,v}_h,a^{k,v}_h)$ and  observe next state $s^{k,v}_{h+1}$.
                    
                    \STATE update $I^{k}_h(s^{k,v}_h,a^{k,v}_h)\gets 1 , I^{k}_h(s^{k,v}_h,a^{k,v}_h,s^{k,v}_{h+1})\gets 1$.
        
                \ENDFOR
                
                \STATE set $n^{k+1}_h(s,a) \gets n^{k}_h(s,a) + I^{k}_h(s,a),n^{k+1}_h(s,a,s') \gets n^{k}_h(s,a,s') + I^{k}_h(s,a,s') \ \forall (s,a,s',h)$.
                
                \STATE set $\hat p^{k+1}_h(s' \mid s,a) \gets \frac{n^{k+1}_h(s,a,s')}{n^{k+1}_h(s,a)\vee 1}  \ \forall (s,a,s',h) \in \calS \times \calA \times \calS \times [H]$.
                
                \STATE compute confidence set for $\epsilon^{k+1}_h(s' \mid s,a) = 4\sqrt{ \frac{\hat{p}_{h}^{k+1}(s'| s,a) \ln\frac{HSAK}{4\delta}}{n_{h}^{k+1}(s,a) \vee 1}} + 10 \frac{\ln\frac{HSAK}{4\delta}}{n_{h}^{k+1}(s,a) \vee 1}$:
                \[
                    \mathcal{P}^{k+1}
                    =
                    \left\{ p' \mid \forall (s,a,s',h) : |\hat p^{k+1}_h(s' \mid s,a) - p'_h(s' \mid s,a)| \le \epsilon^{k+1}_h(s' \mid s,a) \right\}.
                \]
                
                \STATE compute $u^k_h(s) = \max_{p' \in \mathcal{P}^{k}} q^{p',\pi^k}_h(s) = \max_{p' \in \mathcal{P}^{k}} \Pr[s_h=s \mid \pi^k,p'] \ \forall s \in \calS$.
                
                \STATE compute $\hat c^k_h(s,a) = \frac{c^k_h(s,a) \indevent{s^{k,\sigma(h,a,k)}_h = s}}{u^k_h(s) + \gamma} \  \forall (s,a,h) \in \calS \times \calA \times [H]$.
                
                \STATE compute $q^{k+1} = \argmin_{q \in \Delta(\calM,k+1)} \eta \langle q , \hat c^k \rangle + \KL{q}{q^k}$.
                
                \STATE compute $\pi^{k+1}_h(a \mid s) = \frac{q^{k+1}_h(s,a)}{\sum_{a' \in \calA} q^{k+1}_h(s,a')} \  \forall (s,a,h) \in \calS \times \calA \times [H]$, where $q^{k+1}_h(s,a) = \sum_{s' \in \calS} q^{k+1}_h(s,a,s')$.
            \ENDFOR
        \ENDFOR
    \end{algorithmic}
\end{algorithm}

For the setting of adversarial MDPs with non-fresh randomness and unknown transitions we propose the Cooperative UOB-REPS with non-fresh randomness algorithm (\verb|coop-nf-UOB-REPS|; see \cref{alg:coop-nf-uob-reps}).
The idea is to combine the \verb|coop-nf-O-REPS| algorithm for known transitions with ideas from the \verb|coop-ULCAE| algorithm in order to handle unknown transitions under non-fresh randomness.
The main challenge is that, unlike the stochastic case, we cannot eliminate sub-optimal actions.
Thus, our method requires $\sqrt{K}$ agents to attain near-optimal regret as opposed to the stochastic case where only $H^2 A^2$ agents are required.

\begin{theorem}
    \label{thm:reg-coop-nf-uob-reps}
    Assume that \verb|coop-nf-UOB-REPS| is run with $m = \sqrt{K}$ agents.
    With probability $1 - \delta$, setting $\eta = \gamma = \sqrt{\frac{\log \frac{KHSA}{\delta}}{SK}}$, the individual regret of each agent of \verb|coop-nf-O-REPS| is
    \[
        \regret
        =
        O \left( H^2 S \sqrt{K \log \frac{K H S A}{\delta}} + H^3 S^3 \log^2 \frac{K H S A}{\delta} \right).
    \]
\end{theorem}

\subsection{The good event}

Denote  $\epsilon^k_h( s' \mid s,a) = \sqrt{\frac{2 \hat p^k_h(s'|s,a) \log \frac{30 K H S A}{\delta}}{n^{k}_h(s,a)\vee 1}} + \frac{2 \log \frac{30 K H S A}{\delta}}{n^{k}_h(s,a) \vee 1}$ and $\epsilon^k_h(s,a) = \sum_{s'\in\calS} \epsilon^k_h( s' \mid  s,a)$.
Define the following events: 
\begin{align*}
    E^p 
    & = 
    \left\{ \forall (k,s,a,s',h):\ |p_h (s'|s,a) - \hat{p}^{k}_h (s'|s,a)| \le \epsilon^k_h(s' \mid s,a) \right\}
    \\
    E^{on} &= 
     \left\{ \forall (k,h,s,a,v)\in[K]\times[H]\times\mathcal{S}\times\mathcal{A}:
            n_{h}^{k}(s,a) \ge \frac{1}{2}\sum_{j=1}^{k-1}q_{h}^{\pi^j}(s) - \log\frac{6mHSA}{\delta} 
        \right\}
    \\
    E^c 
    & = 
    \left\{ \sum_{k=1}^K \langle \bbEk [ \hat c^k \mid \pi^k ] - \hat c^k , q^k \rangle \le 4 H S \sqrt{K \log \frac{6}{\delta}} \right\}
    \\
    E^{\hat c}
    & =
    \left\{ \sum_{k=1}^K \sum_{h=1}^H \sum_{s \in \calS} \sum_{a \in \calA} \pi^k_h(a \mid s) \left( \hat c^k_h(s,a) - 2 c^k_h(s,a) \right) \le \frac{10 H S \log \frac{3 H S A}{\delta}}{\gamma} \right\}
    \\
    E^\star
    & =
    \left\{ \sum_{k=1}^K \langle \hat c^k - c^k , q^{\pi^\star} \rangle \le \frac{H \log \frac{3 H S A}{\delta}}{\gamma} \right\}
\end{align*}

The good event is the intersection of the above events. 
The following lemma establishes that the good event holds with high probability. 

\begin{lemma}[The Good Event]
    \label{lemma:good-event-adversarial-unknown-p-non-fresh}
    Let $\bbG =E^p \cap E^{on} \cap E^c \cap E^{\hat c} \cap E^\star$ be the good event. 
    It holds that $\Pr [ \bbG ] \geq 1-\delta$.
\end{lemma}

\begin{proof}
    Similar to the proofs of \cref{lemma:good-event-adversarial-known-p-non-fresh,lem:good-event-stochastic-non-fresh} and to proofs in \citet{jin2020learning}.
\end{proof}

\subsection{Proof of Theorem~\ref{thm:reg-coop-nf-uob-reps}}

\begin{proof}[Proof of Theorem~\ref{thm:reg-coop-nf-uob-reps}]
    By \cref{lemma:good-event-adversarial-unknown-p-non-fresh}, the good event holds with probability $1 - \delta$.
    We now analyze the regret under the assumption that the good event holds.
    Note that each agent plays the OMD policy $\pi^k$ in all except for $H A \sqrt{K}$ episodes.
    Thus, the regret is bounded by the regret of the policies $\{ \pi^k \}_{k=1}^K$ plus a $H^2 A \sqrt{K}$ term which is at most $H^2 S \sqrt{K}$.
    Next, we focus on bounding the regret of $\{ \pi^k \}_{k=1}^K$, starting with the following decomposition:
    \begin{align*}
        \sum_{k=1}^K V^{k,\pi^k}_1(s_{1}^{k,v}) - V^{k,\pi^\star}_1(s_{1}^{k,v})
        & =
        \sum_{k=1}^K \langle c^k , q^{\pi^k} - q^{\pi^\star} \rangle
        \\
        & =
        \underbrace{\sum_{k=1}^K \langle c^k , q^{\pi^k} - q^k \rangle}_{(A)} + \underbrace{\sum_{k=1}^K \langle c^k - \hat c^k , q^k \rangle}_{(B)} + \underbrace{\sum_{k=1}^K \langle \hat c^k , q^k - q^{\pi^\star} \rangle}_{(C)} + \underbrace{\sum_{k=1}^K \langle \hat c^k - c^k , q^{\pi^\star} \rangle}_{(D)}.
    \end{align*}
    
    Let $\logterm = \log \frac{KHSA\numagents}{\delta}$ be a logarithmic term.
    Term $(A)$ can be decomposed using the value difference lemma (see, e.g., \citet{shani2020optimistic}):
    \begin{align*}
        (A)
        & =
        \sum_{k=1}^K \langle c^k , q^{\pi^k} - q^k \rangle
        \le
        2H \sum_{k=1}^K \sumhsa q^{\pi^k}_h(s,a) \lVert p_h(\cdot \mid s,a) - \hat p^k_h(\cdot \mid s,a) \rVert_1
        \\
        & \lesssim
        H \sqrt{S \logterm} \sum_{k,h,s,a} \frac{q^{\pi^k}_h(s,a) }{\sqrt{n^k_h(s,a) \vee 1}} + H S \logterm \sum_{k,h,s,a} \frac{q^{\pi^k}_h(s,a) }{n^k_h(s,a) \vee 1},
    \end{align*}
    where the second inequality is by event $E^p$.
    We now bound each of the two sums separately using the event $E^{on}$.
    For the second sum we have:
    \begin{align*}
        \sum_{k,h,s,a} \frac{q^{\pi^k}_h(s,a) }{n^k_h(s,a) \vee 1}
        & \le
        \sum_{k,h,s,a} \frac{q^{\pi^k}_h(s,a) }{(\frac{1}{2}\sum_{j=1}^{k-1}q_{h}^{\pi^j}(s) - \log\frac{6mHSA}{\delta} ) \vee 1}
        =
        \sum_{k,h,s} \frac{q^{\pi^k}_h(s) \sum_a \pi^k_h(a \mid s) }{(\frac{1}{2}\sum_{j=1}^{k-1}q_{h}^{\pi^j}(s) - \log\frac{6mHSA}{\delta} ) \vee 1}
        \\
        & \le
        2 H S \logterm + 2 \sum_{h,s} \sum_{k : \sum_{j=1}^{k-1}q_{h}^{\pi^j}(s) \ge 2 \log\frac{6mHSA}{\delta}} \frac{q^{\pi^k}_h(s) }{\sum_{j=1}^{k-1}q_{h}^{\pi^j}(s)}
        \lesssim
        H S \logterm,
    \end{align*}
    where the last inequality is by \citet[Lemma B.18]{rosenberg2020near}.
    For the first term:
    \begin{align*}
        \sum_{k,h,s,a} \frac{q^{\pi^k}_h(s,a) }{\sqrt{n^k_h(s,a) \vee 1}}
        & \le
        \sum_{k,h,s,a} \frac{q^{\pi^k}_h(s,a) }{\sqrt{(\frac{1}{2}\sum_{j=1}^{k-1}q_{h}^{\pi^j}(s) - \log\frac{6mHSA}{\delta} ) \vee 1}}
        =
        \sum_{k,h,s} \frac{q^{\pi^k}_h(s) \sum_a \pi^k_h(a \mid s) }{\sqrt{(\frac{1}{2}\sum_{j=1}^{k-1}q_{h}^{\pi^j}(s) - \log\frac{6mHSA}{\delta} ) \vee 1}}
        \\
        & \le
        2 H S \logterm + 2 \sum_{h,s} \sum_{k : \sum_{j=1}^{k-1}q_{h}^{\pi^j}(s) \ge 2 \log\frac{6mHSA}{\delta}} \frac{q^{\pi^k}_h(s)  }{\sqrt{\sum_{j=1}^{k-1}q_{h}^{\pi^j}(s)}}
        \\
        & =
        2 H S \logterm + 2 \sum_{h,s} \sum_{k : \sum_{j=1}^{k-1}q_{h}^{\pi^j}(s) \ge 2 \log\frac{6mHSA}{\delta}} \frac{q^{\pi^k}_h(s)  }{\sqrt{\sum_{j=1}^{k}q_{h}^{\pi^j}(s)}} \sqrt{\frac{\sum_{j=1}^{k}q_{h}^{\pi^j}(s)}{\sum_{j=1}^{k-1}q_{h}^{\pi^j}(s)}}
        \\
        & =
        2 H S \logterm + 4 \sum_{h,s} \sum_{k : \sum_{j=1}^{k-1}q_{h}^{\pi^j}(s) \ge 2 \log\frac{6mHSA}{\delta}} \frac{q^{\pi^k}_h(s)  }{\sqrt{\sum_{j=1}^{k}q_{h}^{\pi^j}(s)}}
        \\
        & \le
        2 H S \logterm + 8 \sum_{h,s} \sqrt{\sum_{k=1}^K q_{h}^{\pi^k}(s) }
        \le
        2 H S \logterm + 8 \sqrt{H S \sum_{k=1}^K \sum_{h,s} q_{h}^{\pi^k}(s) }
        =
        2 H S \logterm + 8 H \sqrt{S K},
    \end{align*}
    where the third inequality is by \citet[Lemma 1]{streeter2010less}, and the last inequality is by Jensen's inequality.
    Putting these together we get that: $(A) \lesssim H^2 S \sqrt{K \logterm} + H^2 S^2 \logterm^2$.
    
    Term $(B)$ can be further decomposed as:
    \begin{align*}
        (B)
        =
        \sum_{k=1}^K \langle c^k - \hat c^k , q^k \rangle
        =
        \sum_{k=1}^K \langle c^k - \bbEk [ \hat c^k \mid \pi^k ] , q^k \rangle + \sum_{k=1}^K \langle \bbEk [ \hat c^k \mid \pi^k ] - \hat c^k , q^k \rangle.
    \end{align*}
    The second term is bounded by $4 H S\sqrt{K \log \frac{6}{\delta}}$ by the good event $E^c$, and for the first term:
    \begin{align*}
        \sum_{k=1}^K \langle c^k - \bbEk [ \hat c^k \mid \pi^k ] , q^k \rangle
        & =
        \sum_{k=1}^K \sum_{h=1}^H \sum_{s \in \calS} \sum_{a \in \calA} q^k_h(s,a) c^k_h(s,a) \left( 1 - \frac{\bbEk [\indevent{s^{k,\sigma(h,a,k)}_h = s} \mid \pi^k]}{u^k_h(s) + \gamma} \right)
        \\
        & =
        \sum_{k,h,s,a} q^k_h(s,a) c^k_h(s,a) \left( 1 - \frac{q^{\pi^k}_h(s)}{u^k_h(s) + \gamma} \right)
        \le`
        \sum_{k,h,s} q^k_h(s) \left( 1 - \frac{q^{\pi^k}_h(s)}{u^k_h(s) + \gamma} \right)
        \\
        & \le
        2 \sum_{k,h,s} \frac{q^k_h(s)}{u^k_h(s)} ( u^k_h(s) - q^{\pi^k}_h(s) + \gamma )
        \le
        2 \sum_{k,h,s} (u^k_h(s) - q^{\pi^k}_h(s))  + \gamma H S K,
    \end{align*}
    where the second equality is because agent $\sigma(h,a,k)$ plays policy $\pi^k$ until step $h$.
    Finally, $\sum_{k,h,s} (u^k_h(s) - q^k_h(s))$ is bounded by similarly to \cref{lemma:Jin-final}.
    
    Term $(C)$ is bounded by OMD (see, e.g., \citet{rosenberg2019online}) as follows:
    \begin{align*}
        (B)
        & =
        \sum_{k=1}^K \langle \hat c^k , q^k - q^{\pi^\star} \rangle
        \le
        \frac{2 H \log (H S A)}{\eta} + \eta \sum_{k=1}^K \sum_{h=1}^H \sum_{s \in \calS} \sum_{a \in \calA} q^k_h(s,a) \hat c^k_h(s,a)^2
        \\
        & \le
        \frac{2 H \log (H S A)}{\eta} + \eta \sum_{k=1}^K \sum_{h=1}^H \sum_{s \in \calS} \sum_{a \in \calA} q^k_h(s) \pi^k_h(a \mid s) \frac{\hat c^k_h(s,a)}{u^k_h(s) + \gamma}
        \\
        & \le
        \frac{2 H \log (H S A)}{\eta} + \eta \sum_{k=1}^K \sum_{h=1}^H \sum_{s \in \calS} \sum_{a \in \calA} \pi^k_h(a \mid s) \hat c^k_h(s,a)
        \\
        & \le
        \frac{2 H \log (H S A)}{\eta} + \eta \sum_{k=1}^K \sum_{h=1}^H \sum_{s \in \calS} \sum_{a \in \calA} \pi^k_h(a \mid s) c^k_h(s,a) + \frac{\eta HS \log \frac{3 HSA}{\delta}}{\gamma}
        \\
        & \le
        \frac{2 H \log (H S A)}{\eta} + \eta \sum_{k=1}^K \sum_{h=1}^H \sum_{s \in \calS} \sum_{a \in \calA} \pi^k_h(a \mid s) + \frac{\eta HS \log \frac{3 HSA}{\delta}}{\gamma}
        \\
        & =
        \frac{2 H \log (H S A)}{\eta} + \eta H S K + \frac{\eta HS \log \frac{3 HSA}{\delta}}{\gamma},
    \end{align*}
    where the forth inequality is by the good event $E^{\hat c}$.
    
    Term $(D)$ is bounded by $\frac{H \log \frac{3 H}{\delta}}{\gamma}$ by the good event $E^\star$.
    Putting the three terms together gives the final regret bound when setting $\eta = \gamma = \sqrt{\frac{\log \frac{KHSA}{\delta}}{SK}}$.
\end{proof}

\end{document}